\documentclass[aac,crcready]{iosart2x}

\usepackage{xspace}
\usepackage{tcolorbox}
\usepackage{subfigure}
\usepackage{multirow}
\usepackage{multicol}
\usepackage{makecell}
\usepackage{pifont}
\usepackage{array}
\usepackage{tcolorbox}
\newcolumntype{L}[1]{>{\raggedright\let\newline\\\arraybackslash\hspace{0pt}}m{#1}}
\newcolumntype{C}[1]{>{\centering\let\newline\\\arraybackslash\hspace{0pt}}m{#1}}
\newcolumntype{R}[1]{>{\raggedleft\let\newline\\\arraybackslash\hspace{0pt}}m{#1}}
\newcolumntype{B}[1]{>{\raggedright\let\newline\\\arraybackslash\hspace{0pt}}p{#1}}
\newcolumntype{N}[1]{>{\centering\let\newline\\\arraybackslash\hspace{0pt}}p{#1}}
\newcolumntype{M}[1]{>{\raggedleft\let\newline\\\arraybackslash\hspace{0pt}}p{#1}}

\newcommand{\cmark}{\ding{52}}%
\newcommand{\xmark}{\ding{56}}%
\newcommand{\xlr}{AXPLR\xspace}
\newcommand{\xlrs}{\xlr}

\newcommand{\R}{\mathbb{R}}

\newcommand{\Args}{\ensuremath{\mathcal{A}}}
\newcommand{\Rels}{\ensuremath{\mathcal{R}}}
\newcommand{\Atts}{\ensuremath{\Rels^-}}
\newcommand{\Supps}{\ensuremath{\Rels^+}}
\newcommand{\Attsp}{\ensuremath{\Rels^{-\prime}}}
\newcommand{\Suppsp}{\ensuremath{\Rels^{+\prime}}}

\newcommand{\grasp}{GrASP\xspace}
\newcommand{\sigmo}{\mbox{sigmoid}}
\newcommand{\as}[1]{[\texttt{#1}]} 

\newcommand{\QBAF}{QBAFc\xspace}
\newcommand{\QBAFs}{QBAFcs\xspace}
\newcommand{\QBAFP}{QBAFc\ensuremath{'}\xspace}
\newcommand{\QBAFPs}{QBAFc\ensuremath{'}s\xspace}
\newcommand{\BQBAF}{BQBAFc\xspace}
\newcommand{\BQBAFs}{BQBAFcs\xspace}
\newcommand{\BQBAFP}{BQBAFc\ensuremath{'}\xspace}
\newcommand{\BQBAFPs}{BQBAFc\ensuremath{'}s\xspace}
\newcommand{\TQBAF}{TQBAFc\xspace}
\newcommand{\TQBAFs}{TQBAFcs\xspace}
\newcommand{\TQBAFP}{TQBAFc\ensuremath{'}\xspace}
\newcommand{\TQBAFPs}{TQBAFc\ensuremath{'}s\xspace}

\newcommand{\PLR}{PLR\xspace}

\newtheorem{preprop}{Proposition}
\newtheorem{prelem}{Lemma}
\newtheorem{predef}{Definition}
\newtheorem{pretheo}{Theorem}
\newtheorem{precor}{Corollary}

\usepackage{tikz}
\usetikzlibrary{shapes.geometric, arrows}
\tikzstyle{arg} = [circle, minimum width=1cm, minimum height=1cm, text centered, draw=black]
\tikzstyle{argsm} = [circle, minimum width=0.8cm, minimum height=0.8cm, text centered, draw=black]
\tikzstyle{rel} = [thick,->,>=stealth]
\newcommand{\tatt}[2]{\draw [rel] (#1) -- node[anchor=south] {\large \textbf{--}} (#2);}
\newcommand{\tsupp}[2]{\draw [rel] (#1) -- node[anchor=south] {\large \textbf{+}} (#2);}

\pubyear{0000}
\volume{0}
\firstpage{1}
\lastpage{1}

\begin{document}

\begin{frontmatter}

\title{Argumentative Explanations for Pattern-Based Text Classifiers}
\runtitle{Argumentative Explanations for Pattern-Based Text Classifiers}

\begin{aug}
\author[A]{\inits{P.}\fnms{Piyawat} \snm{Lertvittayakumjorn}\ead[label=e1]{pl1515@imperial.ac.uk}%
\thanks{Corresponding author. \printead{e1}.}}
\author[A]{\inits{F.}\fnms{Francesca} \snm{Toni}\ead[label=e2]{ft@imperial.ac.uk}}
\address[A]{Department of Computing, \orgname{Imperial College London},
\cny{United Kingdom}\printead[presep={\\}]{e1,e2}}
\end{aug}

\begin{abstract}
Recent works in Explainable AI mostly address the transparency issue of black-box models or create explanations for any kind of models (i.e., they are model-agnostic), while leaving explanations of interpretable models largely underexplored. In this paper, we fill this gap by focusing on explanations for a specific interpretable model, namely pattern-based logistic regression (PLR) for binary text classification. We do so because, albeit interpretable, PLR is challenging when it comes to explanations. In particular, we found that a standard way to extract explanations from this model does not consider relations among the features, making the explanations hardly plausible to humans. Hence, we propose \textbf{\xlr}, a novel explanation method using (forms of) computational argumentation to generate explanations (for outputs computed by PLR) which unearth model agreements and disagreements among the features. Specifically, we use computational argumentation as follows: we see features (patterns) in PLR as arguments in a form of quantified bipolar argumentation frameworks (QBAFs) and extract attacks and supports between arguments based on specificity of the arguments; we understand logistic regression as a gradual semantics for these QBAFs, used to determine the arguments' dialectic strength; and we study standard properties of gradual semantics for QBAFs in the context of our argumentative re-interpretation of PLR, sanctioning its suitability for explanatory purposes. We then show how to extract intuitive explanations (for outputs computed by PLR) from the constructed QBAFs. Finally, we conduct an empirical evaluation and two experiments in the context of human-AI collaboration to demonstrate the advantages of our resulting \xlr method.
\end{abstract}


\end{frontmatter}


\section{Introduction} \label{sec:intro}

Humans have been using explanations in AI
for many purposes such as to support human decision making \cite{lai2019human,lertvittayakumjorn-etal-2021-supporting}, to increase human trust in the AI \cite{symeonidis2009moviexplain,jacovi2021formalizing}, to verify and improve the AI \cite{caruana2015intelligible,lertvittayakumjorn-etal-2020-find}, and to learn new knowledge from the AI \cite{mac2018teaching,lai2020chicago}.
Explanations may also be required for an AI-assisted system to comply with recent regulations including the General Data Protection Regulation (GDPR) \cite{goodman2017european}. 
These various needs for explanation have drawn a great amount of attention to the field of \emph{explainable AI (XAI)} in recent years \cite{adadi2018peeking}. 
When AI-assisted systems are used for prediction (referred to as \emph{prediction models} or simply \emph{models} in the literature and in this paper), explanations of interest are often categorized broadly into two types: \emph{local explanations} and \emph{global explanations} \cite{adadi2018peeking}, where the former focus on explaining the predictions for specific inputs while the latter aim to explain the behavior of the model in general, irrespective of any inputs that it may take.
If the model is \emph{inherently interpretable} \cite{rudin2019stop} (e.g., a decision tree), the model itself can be viewed as the global explanation whereas local explanations can be obtained during the prediction process (e.g., the corresponding path in the decision tree for the input, leading to the output/prediction).
In this paper, we refer to explanations straightforwardly extracted from inherently interpretable models (e.g., the applicable path in a decision tree)  as \textit{model-inherent explanations}.
However, if the model is opaque (e.g., it is a deep learning model), we may need to apply an additional step, by using a so-called  \emph{post-hoc explanation method} (e.g., LIME \cite{ribeiro2016should} and SHAP \cite{lundberg2017unified}), for extracting the explanations.

A number of properties of explanations have been identified as desirable in the literature, e.g., as in  \cite{sokol2020explainability}. Amongst them, generally, we call an explanation \textit{faithful} to the model if it accurately reflects the true reasoning process of the model (e.g., as discussed in \cite{jacovi-goldberg-2020-towards}). 
Meanwhile, an explanation is deemed \textit{plausible} if it agrees with human judgement (e.g., see \cite{wiegreffe-pinter-2019-attention}).
These two properties of explanations, i.e., faithfulness and plausibility, may be important in different situations.
For instance, we want faithful explanations in order to verify the model correctness while we want plausible explanations to satisfy end users.
Note that model-inherent explanations 
can be deemed faithful due to their straightforward and sensible explanation extraction process. However, this does not guarantee other desirable properties of the explanations.
For instance, using a decision tree path with depth of 15 as an explanation is not comprehensible to humans and, therefore, not very plausible 
either. 
Post-hoc explanations could be more effective for impressing end-users in this case though they are not inherently extracted from (or perfectly faithful to) the underlying interpretable model (i.e., the decision tree).

In this paper, we develop a novel post-hoc local explanation method that aims to generate plausible explanations for a specific class of interpretable prediction models performing \emph{binary text classification} with natural language data. 
Binary text classification aims to classify a given text into one of two possible categories.
Examples of binary text classification (both studied in this paper) are \emph{sentiment analysis} (where a piece of text is classified as having positive or negative sentiment) and
\emph{spam detection} (where a message is classified as spam or not). %
Our interpretable prediction models are built using logistic regression (LR) \cite[chapter~5]{jurafsky2020speech} with textual patterns \cite{shnarch-etal-2017-grasp} as features.
LR is a traditional machine learning method,  leading to interpretable models with linear combinations of features,   
that can be used, in particular, for text classification \cite[chapter~5]{jurafsky2020speech}.
Because text documents are unstructured data, we need to perform \emph{feature extraction} so as to obtain numerical representations of the documents before training the LR classifier. 
One standard way to do feature extraction is using frequent n-grams
(i.e., frequent $n$ consecutive words in the dataset)
as features and applying TF-IDF vectorization to find associated values of the features \cite[chapter~2]{weiss2010text}.
Though using n-grams as features is simple and often effective, it makes the model less generalizable to words or n-grams that have never appeared during training.
Also, the features are usually too fine-grained for humans to synthesize the overview of what the model has learned even though LR is inherently interpretable. 
In this paper, as elsewhere (e.g., \cite{hussain2021text,carstens2017using,ghanem2002automatic}), we use \emph{patterns} as interpretable features for prediction, in alternative to n-gram features. 
We call our models of interest \emph{pattern-based logistic regression (\PLR) models}. 

\PLR models are inherently interpretable because LR is interpretable and because their features are interpretable and, as we will show, convenient for humans to learn or to extract knowledge from.
However, their model-inherent explanations for \PLR  may not be plausible.
 This issue is especially critical when interactions among input features underpin the model whereas the model-inherent explanations treat features independently of each other.
 These feature interactions may result from agreement or disagreement between correlated pattern features.
In order to address this problem, our proposed explanation method leverages \emph{computational argumentation (CA)} to take care of the feature interactions and generate  more plausible local explanations than the model-inherent explanations.
We call our novel explanation method (combining LR,  textual patterns, and CA) \emph{Argumentative eXplanations for Pattern-based Logistic Regression (\xlr}, pronounced ``ax-plore'').\footnote{Note that we use \xlr to indicate both our method for generating explanations and the generated explanations themselves; also, when used to refer to explanations, \xlr has the same form in the singular and the plural.}

Generally, local explanations often have an argumentative spirit by nature since they need to argue for or against possible predictions of the model \cite{vcyras2021argumentative}. 
When there are several arguments involved, these arguments may also have dialectical relationships between each other. 
Hence, there are several existing works which use computational argumentation to underpin XAI methods and produce argumentative explanations. 
For example, DEAr \cite{cocarascu2020data} considers related training examples as arguments, which argue to classify a test example, and uses a dispute tree \cite{vcyras2016explanation} as dialectical explanation.
DAX \cite{dejl2021argflow} extracts local argumentative explanations from a deep neural network by using arguments and their relations to represent the nodes and their connections in the neural network. 
(For more approaches, see recent survey papers of argumentative XAI \cite{vcyras2021argumentative,vassiliades2021argumentation}.)
In our work, \xlr uses 
pattern-based input features of the \PLR model as arguments and draws dialectical relations from specificity of the pattern features. 
All are modeled using  modified versions of \emph{Quantitative Bipolar Argumentation Frameworks (QBAFs)} \cite{baroni2019fine} before being processed and translated into argumentative explanations for human consumption.
Specifically, we use two variants of QBAFs. Both include a mapping associating arguments with the classification they advocate, in addition to arguments, attack and support relations and base scores as in standard QBAFs. The two variants differ in the way they use specificity of patterns to define the direction of attacks and supports. 

We summarize the contribution of our work as follows.
\begin{itemize}
    \item We show that model-inherent local explanations for pattern-based logistic regression can lead to implausible explanations.
    \item We propose \xlr, a novel argumentative explanation method, to tackle the above problem by modeling relationships among pattern features using quantitative bipolar argumentation.
    \item We prove that the argumentation framework underpinning \xlr always predicts the same output as the original \PLR model and satisfies 
    several dialectical properties of human debates.
    \item Using three binary text classification datasets, we conduct an empirical evaluation of the extracted argumentation frameworks. 
    Moreover, for the same datasets, we conduct two human experiments to evaluate how plausible and helpful \xlr is for human consumption compared to other explanation methods.
\end{itemize}

In the remainder of the paper, we explain the pattern-based logistic regression (\PLR) model in Section~\ref{sec:background}. 
Then we discuss the weakness of the model-inherent explanations for \PLR in Section~\ref{sec:need}.
After that, Section~\ref{sec:xlr} describes the two variants of QBAFs underpinning \xlr, while Section~\ref{subsec:properties} shows that these QBAFs satisfy many dialectical properties of human debates, leading to leaner derived \xlr (of which the presentations are described in Section~\ref{subsec:explanations}).
Next, the experimental setup for \xlr is explained in Section~\ref{sec:setup}, followed by one empirical experiment in Section~\ref{sec:empirical} and two human experiments (to assess the amenability of the argumentation underpinning \xlr specifically) in Sections~\ref{sec:plausibility} and \ref{sec:tutorial}.
Lastly, we discuss other possible uses of \xlr in Section~\ref{sec:discussions}, position our work with respect to other related work in Section~\ref{sec:relatedwork}, and summarize the paper in Section~\ref{sec:conclusion}.

\section{Background} \label{sec:background}

In this section, we provide necessary background on text classification (see Section~\ref{sec:tc}) and \PLR, including 
logistic regression (LR) which is the core machine learning method of \PLR (see Section~\ref{subsec:LR}) and pattern features as well as the pattern extraction algorithm \grasp \cite{shnarch-etal-2017-grasp} used for constructing pattern features from training data (see Section~\ref{subsec:grasp}). We conclude with an illustration of the overall process of \PLR combining LR with \grasp for text classification (see Section~\ref{subsec:plr}).
To illustrate ideas, we will use \emph{sentiment analysis} as a running example of text classification throughout this section.

\subsection{Binary Text Classification}
\label{sec:tc}

We focus on the binary text classification task with two classes, using, as conventional, $\mathcal{C} = \{0, 1\}$ as the set of classes.
For example, in the case of sentiment analysis, 0 stands for negative sentiment and 1 stands for positive sentiment. 
A training dataset 
$\mathcal{D}$ contains $N$ different pairs of the form $(x, y)$ where 
$x$ is an input text 
and $y \in \mathcal{C}$ is the true class label of $x$.
This dataset is used to train a classifier, which determines the probability of classes for any given input.
In the context of binary text classification, $\mathcal{D}$ can be split into disjoint sets $\mathcal{D}^+$ and $\mathcal{D}^-$, containing \emph{positive} examples ($y = 1$) and \emph{negative} examples ($y = 0$)  in $\mathcal{D}$, respectively.
A classifier trained on $\mathcal{D}$ determines, for input $x$, a class $\hat{y}\in \mathcal{C}$.

\subsection{Logistic Regression}
\label{subsec:LR}

For each input text $x$, let us assume that $x$ can be represented as a feature vector $\mathbf{f} = [f_{1}, f_{2}, \ldots, f_{d}]$  
where $f_{i}$ is a feature and $d \geq 1$ is the number of features used to represent $x$.
Then, 
an LR model targeting binary classification gives 
\begin{align}
    \begin{split}
    P(y=1|x) &= \sigmo(\textbf{w}^T\textbf{f} + b) 
    = \sigmo(\sum\limits_{i=1}^dw_if_{i} + b) \\ 
    &= \sigmo(w_1f_{1}+w_2f_{2}+...+w_df_{d}+b) \label{eq5:bilr}
    \end{split}
\end{align}
where $\textbf{w}\in\R^d$ and $b\in\R$ are weights and bias of the LR model, respectively. The \textit{sigmoid function} (so called a \textit{logistic function}) is used to convert any real number into a value between 0 and 1:
\begin{equation}
    \sigmo(z) = \frac{1}{1+e^{-z}} \label{eq2:sigmoid}
\end{equation}
where $z = 0$ yields $\sigmo(z) = 0.5$. Note also that $1-\sigmo(z)=\sigmo(-z)$.
Normally, if $P(y=1|x) \geq 0.5$ (i.e., $\sum\limits_{i=1}^dw_if_{i} + b \geq 0$), we predict class 1 for input $x$ (i.e., $\hat{y} = 1$). Otherwise, we predict class 0 (i.e., $\hat{y} = 0$).

The LR model is obtained after the training process is completed; it is fully characterized by  $\textbf{w}$ and $b$ which minimize the objective function (typically the binary cross-entropy loss) to be used for predicting unseen examples (in some test datasets, for example).

The next questions are ``How do the $d$ features look like for text?'' and ``How can we obtain them?''. We use \emph{pattern features} whereby patterns indicate high-level characteristics of words in input texts,
in addition to specifying exact words or lemmas. 
These high-level characteristics include both syntactic attributes (such as part-of-speech tags) and semantic attributes (such as synonyms and hypernyms). Thereby, we choose \textbf{\grasp} for this purpose.

\subsection{Pattern features and \grasp: GReedy Augmented Sequential Patterns} \label{subsec:grasp}

\grasp is a supervised algorithm which learns expressive patterns 
able to distinguish two classes of text \cite{shnarch-etal-2017-grasp}. 
An example of \grasp pattern 
for distinguishing between 
positive and negative texts in sentiment analysis  is \begin{quote}
   [\as{TEXT:nothing}, \as{SENTIMENT:pos}] 
with two gaps allowed in-between.
\end{quote}
This pattern matches, for example, a sequence of two words where the first word is ``nothing'' and the second word is a positive word according to a specific lexicon (such as the one released by \cite{hu2004mining}).
Moreover, the pattern allows at most two additional words in-between 
to increase flexibility of the pattern.
Examples of texts matched by this pattern include: 

\begin{quote}
``There is \textbf{nothing} \textbf{delicious} in this dinner .'' and 
\\
``... worse products . \textbf{Nothing} is \textbf{delicate} ! ...'' 
\end{quote}
where the \textbf{bold-face words} are words matching components in the pattern.

\grasp\ is applied directly to the training data. In order to use it, we need to prepare two lists of texts containing positive and negative examples that we want to distinguish. 
Also, we need to specify some hyperparameters such as the
desired number of patterns, the number of gaps allowed,  the set of linguistic attributes which can appear in the patterns, and the maximum number of attributes per pattern.
In the experiments in this paper, we employ the publicly released implementation of \grasp \cite{lertvittayakumjorn2021grasp} which provides several built-in attributes that are suitable for classification tasks in general, e.g., the token text itself, its lemma, its hypernyms (according to wordnet \cite{miller1995wordnet}), part-of-speech tags, and sentiment tags.
The resulting \grasp patterns are used as features in the LR model.
Note, however, that we do not utilize the associated class \grasp assigns to each pattern to classify the input directly because \grasp does not tell us how to properly deal with the input that matches multiple (and potentially contradicting or relating) patterns.
Instead, we use the patterns from \grasp only as features for training a classifier, letting the classifier decide how multiple patterns should play their roles and contribute to the final classification.

\subsection{\PLR using \grasp} \label{subsec:plr}

In this paper, we focus on \PLR, i.e., LR with \grasp\ patterns as features.
To train \PLR, we first split the training data $\mathcal{D}$ into $\mathcal{D}^+$ and $\mathcal{D}^-$ and feed them to the \grasp algorithm along with some hyperparameters mentioned above.
After obtaining the $d$ patterns, we extract the binary feature vector $\textbf{f}$ for each training example $x$ and use it to train the LR model together with the class label $y$.
Specifically, for each input text $x \in \mathcal{D}$, we extract the feature vector $\mathbf{f} = [f_{1}, f_{2}, \ldots, f_{d}] \in \{0, 1\}^d$ where $f_{i}$ is a binary feature and $d$ is the number of textual patterns used to represent $x$. $f_{i}$ equals 1 if the input $x$ contains the pattern $p_i$; otherwise, $f_{i}$ equals 0. 
Then, to learn from the training data $\mathcal{D}$, we train a binary logistic regression model using the binary cross-entropy loss (with a regularization term) as objective function.

During testing, given an unseen document $x$, we get the prediction by extracting the feature vector $\textbf{f}$ using the $d$ \grasp patterns and running the LR model on $\textbf{f}$.
Figure~\ref{fig5:ex} shows an illustrative example of how to make a prediction using a trained PLR model. Given the sentence ``There is nothing better than hot sausages of this restaurant.'' as an input text $x$, we want to use a trained PLR model to predict the sentiment of this sentence. 
Assume that among the $d$ patterns of the model, there are only four patterns -- $p_1$, $p_2$, $p_3$, and $p_4$ as shown in Figure~\ref{fig5:ex} -- that match this sentence. 
In other words, for $i\in\{1,2,3,4\}$, $f_i = 1$; otherwise, $f_i=0$.
According to Equation~\ref{eq5:bilr}, the probability of this text being a positive sentiment text, i.e., $P(y=1|x)$, equals $\sigmo(w_1f_{1}+w_2f_{2}+w_3f_{3}+w_4f_{4}+b)$.
For the trained weights and bias in Figure~\ref{fig5:ex}, the predicted class $\hat{y}$ of $x$ is positive (1) since the predicted probability is 0.5744, which is greater than 0.5.

\section{Explaining \PLR Classifiers: The Need for Argumentation}
\label{sec:need}

Since logistic regression is inherently interpretable and the \grasp patterns used are also interpretable, we can generate local explanations for inputs to a trained PLR model by reporting parts of the inputs 
that match the top-$k$ patterns
, in the spirit of much work in the XAI literature
(e.g., \cite[chapter~5]{molnar2020interpretable} and \cite{caruana2015intelligible}).
Formally, given an input $x$, let $s_{i}$ be the \emph{contribution} of the pattern $p_i$ for the prediction $\hat{y}$, defined as follows, with reference to Equation~\ref{eq5:bilr}: when $\hat{y} = 1$, $s_{i} = w_{i}f_{i}$; when $\hat{y} = 0$, $s_{i} = -w_{i}f_{i}$ (so, we can combine both cases to be $s_{i} = (-1)^{\hat{y}+1}w_{i}f_{i}$). 
Then we can return, as the local explanation for $\hat{y}$, a list of triplets of the form $(p_{i'}, \pi(p_{i'}, x), s_{i'})$  where $s_{i'}$, the contribution of pattern $p_{i'}$, is one of the $k$ highest 
contributions and  $s_{i'} \neq 0$, and $\pi(p_{i'}, x)$ is a part of $x$ that matches the pattern $p_{i'}$.
We call this resulting model-inherent explanation the \textbf{flat logistic regression explanation (FLX)} (for $x$ and $\hat{y}$).

\begin{figure}[t]
    \begin{tcolorbox}
    \textbf{Input} $x$: \textit{There is nothing better than hot sausages of this restaurant.}\\
    \textbf{Matched patterns:}
    \begin{align*}
        p_1 &= [\as{TEXT:nothing}, \as{SENTIMENT:pos}] & w_1f_{1} &= -0.9\\
        p_2 &= [\as{TEXT:nothing}] & w_2f_{2} &= -0.4\\
        p_3 &= [\as{SENTIMENT:pos}] & w_3f_{3} &= 1.2\\
        p_4 &= [\as{TEXT:hot}, \as{TEXT:sausages}] & w_4f_{4} &= 0.5\\
        &&b&=-0.1
    \end{align*}
    \textbf{Predicted class $\hat{y}$}: 1 (Positive)\\
    \textbf{Predicted Probability for $\hat{y} = 1$}: sigmoid(-0.9-0.4+1.2+0.5-0.1) = sigmoid(0.3) = 0.5744.\\
    \textbf{FLX contributions for $\hat{y}$}: 
    \colorbox{orange!100}{\strut $p_3$ (1.2)} $>$ \colorbox{orange!42}{\strut $p_4$ (0.5)} $>$
    \colorbox{cyan!33}{\strut $p_2$ (-0.4)} $>$
    \colorbox{cyan!75}{\strut $p_1$ (-0.9)}
    \end{tcolorbox}
    \caption{An illustrative example of using pattern-based logistic regression for sentiment analysis. (Here, FLX stands for \emph{flat logistic regression explanation}, see Section~\ref{sec:need}.)
    }
    \label{fig5:ex}
\end{figure}

For the example in Figure~\ref{fig5:ex}, 
the input text $x$ matches four patterns, and the model predicts class positive (i.e., $\hat{y}=1$). 
If we use FLX
s, we can see that $p_3$ (which means the input text containing a positive word) has the highest contribution of 1.2. So, we will obtain $(p_3, \pi(p_3, x), s_3) = (\mbox{[\as{SENTIMENT:pos}]}, \mbox{``better''}, 1.2)$ as the top triplet in the FLX (i.e., the most important reason) for predicting $\hat{y}=1$.
Nevertheless, one problem with FLXs is that they do not take into account relationships among the patterns.
For the example in Figure~\ref{fig5:ex}, the model has actually weakened the effect of $p_3$ by $p_1$ because the positive word in this case (``better'') follows the word ``nothing'' and the model no longer considers it strongly positive in the context.
What really makes the model answer positively is rather $p_4$, which is considered less important by the FLX. 
Although the contribution of $p_4$ (0.5) is lower than that of $p_3$, it is not overridden by other patterns.
We could see that these four patterns 
are arguing to make the prediction, in that each pattern is an argument for or against the prediction.
Some patterns have dialectical relations with one another (such as 
the disagreement between
$p_1$ and $p_3$). 
Hence, to improve plausibility of the explanations and make them in line with the underpinning dialectical relations, we apply computational argumentation, as shown next, to generate local explanations for this \PLR model. 
Specifically, we aim to use a form of) quantitative bipolar argumentation frameworks (QBAFs) \cite{baroni2019fine} to simulate how the \PLR model works on an input text. Intuitively, QBAFs are tuples comprising arguments, an attacks relation between arguments, a support relation between arguments and a base score for each argument. This conceptualisation  serves our purposes well, as, intuitively, arguments in QBAFs can be used to represent applicable patterns, supports and attacks can reflect agreement and disagreement between these patterns, and base scores can represent the (learned) absolute weights of these patterns in the \PLR model.

\section{\xlr: Argumentative Explanations for Pattern-Based Logistic Regression} \label{sec:xlr}

In this section, we introduce our \xlr\ method, whose 
overall generation process is shown in Figure~\ref{fig5:workflow}, alongside the illustrative example from Figure~\ref{fig5:ex}.
In Figure~\ref{fig5:workflow}, the part above the purple line is the standard prediction process already captured in Figure~\ref{fig5:ex}, starting from extracting the feature vector from the input text and then computing the predicted probability using the model weights ($\mathbf{w}$ and $b$).
Below the purple line, Figure~\ref{fig5:workflow} shows the four main steps for generating \xlr.
Using the feature vector and the model weights, the first step constructs a special type of quantitative bipolar argumentation framework (i.e., \QBAF) to represent relationships between the pattern features found in the input text (as well as the bias term $b$ of the \PLR model, corresponding to the top-most $\delta$ argument).
This type of QBAF is special because each argument also supports a class in $\mathcal{C}$, as represented by its background color. 
Particularly, the arguments with green background support the positive class, whereas the ones with red background support the negative class.
The supported class, as well as the base score of the argument, is determined by the weight of the corresponding pattern (or the bias term $b$) in the \PLR model.
 
The second step computes the dialectical strength of each argument, considering its attacker(s) and supporter(s).
Here, we propose a new strength function, returning a real number score in $(-\infty, \infty)$ that could reflect a class probability predicted by the \PLR model (when the function is applied to the argument $\delta$).
With this new strength function, the dialectical strengths of some arguments might be negative, and thus possibly difficult to interpret (e.g., in Figure~\ref{fig5:workflow},what does it mean for argument $\alpha_1$ (which supports the negative class) to support argument $\delta$ when the latter, having a negative strength, no longer supports the negative class?), so we do post-processing in the third step, making all the strength values to be positive and adjusting relations accordingly in a way that preserves the original meaning (e.g., in Figure~\ref{fig5:workflow}, after $\delta$'s strength is flipped to be positive, its supported class then becomes the positive class; so, $\alpha_1$ needs to attack $\delta$ and $\alpha_4$ needs to support $\delta$ to preserve the original interactions between arguments).
Finally, using the post-processed \QBAF, the fourth step generates the explanation which could be shallow (using only top-level arguments in the \QBAF) or deep (also using arguments at other levels in the \QBAF).
The background color of text fragments matching the patterns reflects the post-processed strengths of the corresponding arguments (where $\alpha_1$ corresponds to ``nothing better'', $\alpha_2$ to ``nothing'', $\alpha_3$ to ``better'' and $\alpha_4$ to ``hot sausages''). 
Finally, we can present the explanations to humans.

\begin{figure}[t!]
    \centering
    \includegraphics[width=\textwidth]{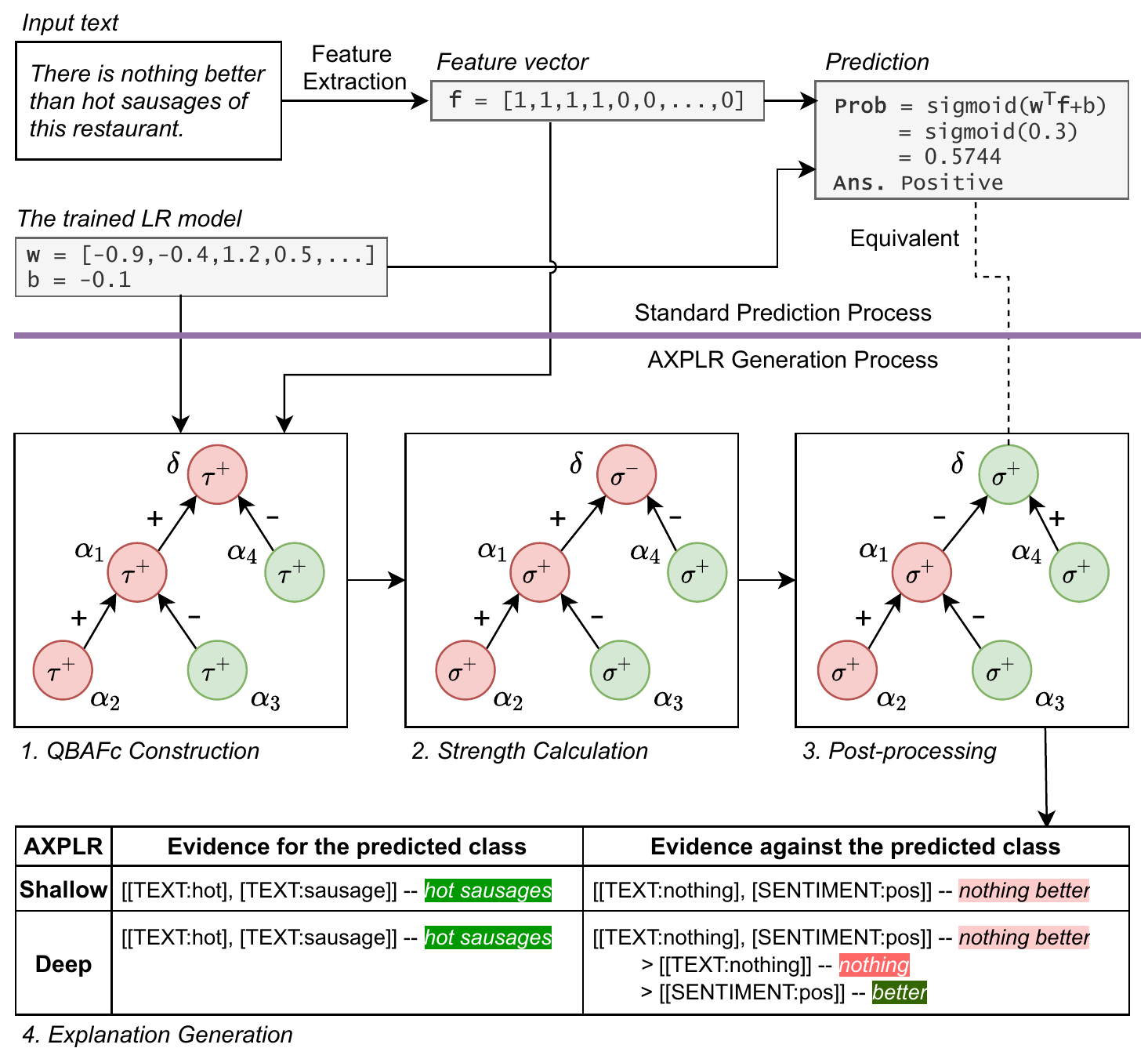}
    \caption{Overview of the \xlr generation process. Above the purple line, it shows the standard prediction process of pattern-based logistic regression. Below the purple line, it shows the four main steps to generate \xlr for the illustrative example from Figure~\ref{fig5:ex}
    (using a bottom-up \QBAF (\BQBAF), where edges labelled + indicate support and edges labelled - indicate attack).
    In step 1, $\tau^+$ indicates that all the base scores of the arguments are positive.
    In steps 2-3, $\sigma^+$ and $\sigma^-$ indicate that the dialectical strengths of the arguments are positive and negative, respectively. 
    In steps 1-3, the background color indicates the supported class (green indicates the positive class and red indicates the negative class).
    Step 4 automatically generates explanations (possibly of different kinds - e.g., shallow or deep) from the \QBAF\ resulting from step 3. 
    (See Section~\ref{sec:xlr} for further details.)
    }
    \label{fig5:workflow}
\end{figure}

In the remainder of this section, we provide details for the first three steps of the \xlr generation process (in Sections~\ref{subsec:af}, \ref{subsec:semantics}, and \ref{subsec:post}, respectively).
Then, in Section~\ref{subsec:explanations}, we give details of step 4.
Before that, in Section~\ref{subsec:properties}, we prove formal properties of (original and post-processed) \QBAFs, providing formal guarantees about their suitability to give rise to explanations.

\subsection{\QBAF construction} \label{subsec:af}

To begin with, we define how two patterns can be related. 

\begin{predef} \label{def:specific}
A pattern $p_1$ is more specific than or equivalent to another pattern $p_2$ (written as $p_1 \succeq p_2$) if and only if for every text $t$ matched by $p_1$, $t$ is also matched by $p_2$. In addition,  $p_1$ is more specific than $p_2$ (written as $p_1 \succ p_2$) if and only if $p_1 \succeq p_2$ but $p_2 \nsucceq p_1$. 
\end{predef}

For instance, we can say from Figure~\ref{fig5:ex} that $p_1 \succeq p_3$ because every text matched by $p_1$ is guaranteed to have a positive sentiment word which makes it matched by $p_3$. However, $p_3 \nsucceq p_1$ because a text matched by $p_3$ is guaranteed to have a positive word but it may not have the word ``nothing'' followed by a positive word. These two facts also imply $p_1 \succ p_3$. Similarly, $p_1 \succ p_2$. 

\begin{prelem} \label{lem:succ}
The relation $\succ$ is not reflexive and not symmetric, but it is transitive. 
\end{prelem}
\begin{proof}
See Appendix~\ref{proof:lem:succ}.
\end{proof}

Next, we extract argumentation frameworks from a trained \PLR model and a target input text $x$.
These argumentation frameworks, like QBAFs \cite{baroni2019fine}, envisage that arguments can attack or support arguments, and that they are equipped with a base score. However, these frameworks differ from QBAFs in that the arguments therein support\footnote{In this paper, we abuse terminology and use the term `support' with two meanings: an argument may support a class (by means of the function $c$ in Definition~\ref{def:axplrqbaf}) or an argument may support another argument in a dialectical sense (relations $\Atts_T$ and $\Atts_B$ in Definition~\ref{def:axplrqbaf}).}  classes 
(as indicated by the signs of the corresponding parameters in the PLR model).
Moreover, these frameworks instantiate the notions of attack and support in generic QBAFs to match the computation of the \PLR model. 
We name these frameworks
\textbf{\QBAFs} (i.e., \textbf{QBAFs with supported \textbf{c}lasses}).
We consider two ways to define dialectical relations in \QBAFs.
Intuitively, we know that arguments for two patterns that are related by $\prec$ should be in a dialectical relation somehow; 
however, we are uncertain whether the more specific one should be the attacker/supporter or should be attacked/supported. So, we propose two variations of the extracted \QBAFs: \emph{top-down \QBAFs} and \emph{bottom-up \QBAFs}.  

\begin{predef} \label{def:axplrqbaf}
Given a trained binary logistic regression model based on  feature patterns $p_1, \ldots, p_d$ with  weights $\langle w_1, \ldots, w_d, b \rangle$ and an input text $x$ with  binary feature vector $\textbf{f} = [f_1, \ldots, f_d]$, the extracted \emph{top-down \QBAF (\TQBAF)} and the extracted \emph{bottom-up \QBAF (\BQBAF)} are 5-tuples $\langle \Args, \Atts_T, \Supps_T, \tau, c \rangle$ and $\langle \Args, \Atts_B, \Supps_B, \tau, c \rangle$, respectively, such that:
\begin{itemize}
    \item $\Args = \{\alpha_i | f_i = 1\} \cup \{\delta\}$ is the set of arguments, where $\alpha_i$ is the argument drawn from pattern $p_i$ if this appears in $x$, whereas $\delta$ is the \emph{default argument}, corresponding to the bias term in the trained model.
    \item $\tau:\Args\rightarrow [0, \infty)$ is the \emph{base score} function where $\tau(\alpha_i)=|w_i|$ and $\tau(\delta)=|b|$.
    \item $c:\Args\rightarrow\{0, 1\}$ is the function mapping an argument to its \emph{supported class}. Here, $c(\alpha_i)=1$ if $w_i \geq 0$; otherwise, $c(\alpha_i)=0$. Similarly, $c(\delta)=1$ if $b \geq 0$; otherwise, $c(\delta)=0$.
    \item $\Atts_T \subseteq \Args \times \Args$ is the attack relation for the \TQBAF where
    \begin{align*}
        \Atts_T =& \{(\alpha_i, \delta)| c(\alpha_i) \neq c(\delta) \wedge \nexists j[\alpha_j \in \Args \wedge p_i \succ p_j]\} \cup \\
        & \{(\alpha_i, \alpha_j)| c(\alpha_i) \neq c(\alpha_j) \wedge p_i \succ p_j \wedge \nexists k[\alpha_k \in \Args \wedge p_i \succ p_k \succ p_j]\}.
    \end{align*}
    \item $\Supps_T \subseteq \Args \times \Args$ is the 
    support relation for the \TQBAF where
    \begin{align*}
        \Supps_T =& \{(\alpha_i, \delta)| c(\alpha_i) = c(\delta) \wedge \nexists j[\alpha_j \in \Args \wedge p_i \succ p_j]\} \cup \\
        & \{(\alpha_i, \alpha_j)| c(\alpha_i) = c(\alpha_j) \wedge p_i \succ p_j \wedge \nexists k[\alpha_k \in \Args \wedge p_i \succ p_k \succ p_j]\}.
    \end{align*}
    \item $\Atts_B \subseteq \Args \times \Args$ is the 
    attack relation for the \BQBAF where
    \begin{align*}
        \Atts_B =& \{(\alpha_i, \delta)| c(\alpha_i) \neq c(\delta) \wedge \nexists j[\alpha_j \in \Args \wedge p_j \succ p_i]\} \cup \\
        & \{(\alpha_j, \alpha_i)| c(\alpha_i) \neq c(\alpha_j) \wedge p_i \succ p_j \wedge \nexists k[\alpha_k \in \Args \wedge p_i \succ p_k \succ p_j]\}.
    \end{align*}
    \item $\Supps_B \subseteq \Args \times \Args$ is the 
    support relation for the \BQBAF where
    \begin{align*}
        \Supps_B =& \{(\alpha_i, \delta)| c(\alpha_i) = c(\delta) \wedge \nexists j[\alpha_j \in \Args \wedge p_j \succ p_i]\} \cup \\
        & \{(\alpha_j, \alpha_i)| c(\alpha_i) = c(\alpha_j) \wedge p_i \succ p_j \wedge \nexists k[\alpha_k \in \Args \wedge p_i \succ p_k \succ p_j]\}.
    \end{align*}
\end{itemize}
\end{predef}

\begin{figure}[!t]
\centering
\small
\begin{tikzpicture}[node distance=2cm]
\node[draw,align=left,fill=red!20] (d) at (0,0) { $\delta$\\$\tau(\delta)=0.1$\\ $\sigma(\delta) = -0.3$};
\node[draw,align=left,fill=red!20] (a2) at (-4,-3) {$\alpha_2$ \\ $\mbox{[\as{TEXT:nothing}]}$\\$\tau(\alpha_2)=0.4$\\ $\sigma(\alpha_2) = 0.85$};
\node[draw,align=left,fill=green!20] (a3) at (-0.5,-3) {$\alpha_3$ \\ $\mbox{[\as{SENTIMENT:pos}]}$ \\$\tau(\alpha_3)=1.2$\\ $\sigma(\alpha_3) = 0.75$};
\node[draw,align=left,fill=green!20] (a4) at (4.15,-3) {$\alpha_4$ \\ $\mbox{[\as{TEXT:hot}, \as{TEXT:sausages}]}$\\$\tau(\alpha_4)=0.5$ \\ $\sigma(\alpha_4) = 0.5$};
\node[draw,align=left,fill=red!20] (a1) at (-1,-6) {$\alpha_1$ \\ $\mbox{[\as{TEXT:nothing}, \as{SENTIMENT:pos}]}$\\$\tau(\alpha_1)=0.9$\\ $\sigma(\alpha_1) = 0.9$};
\tsupp{a2}{d};
\draw [rel] (a3) -- node[anchor=south west] {\large \textbf{--}} (d);
\tatt{a4}{d};
\tsupp{a1}{a2};
\draw [rel] (a1) -- node[anchor=south west] {\large \textbf{--}} (a3);
\end{tikzpicture}
\caption{The extracted top-down \QBAF for the example in Figure~\ref{fig5:ex}. Here and everywhere in this paper we show \QBAFs as graphs, with nodes representing the arguments and labelled edges representing attack (-) or support (+). The color of the nodes represents the supported class (i.e., green for positive (1) and red for negative (0)).
(The meaning of the equalities of the form $\tau(x)=v$ and $\sigma(x)=v$ will be explained later.) }\label{fig5:extqbaf}
\end{figure}
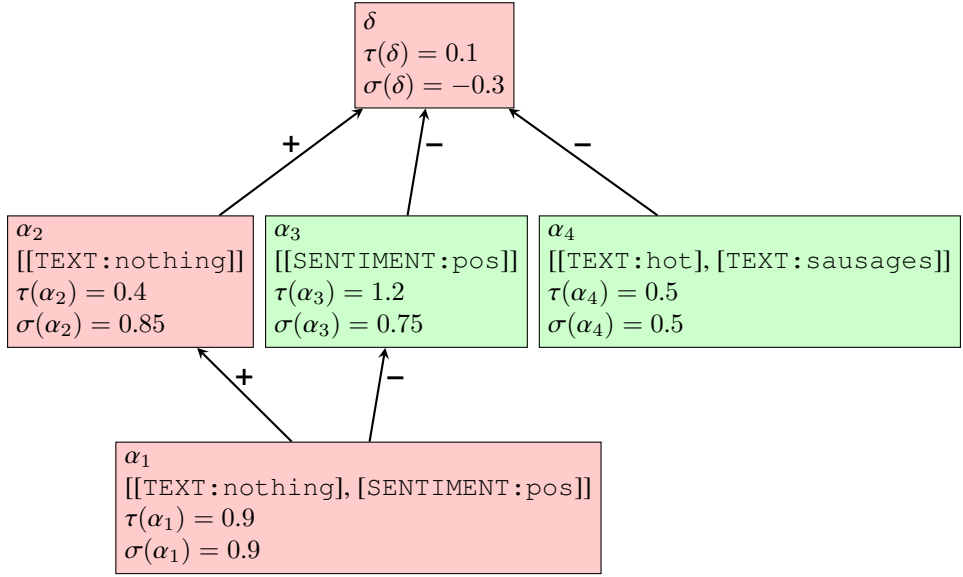

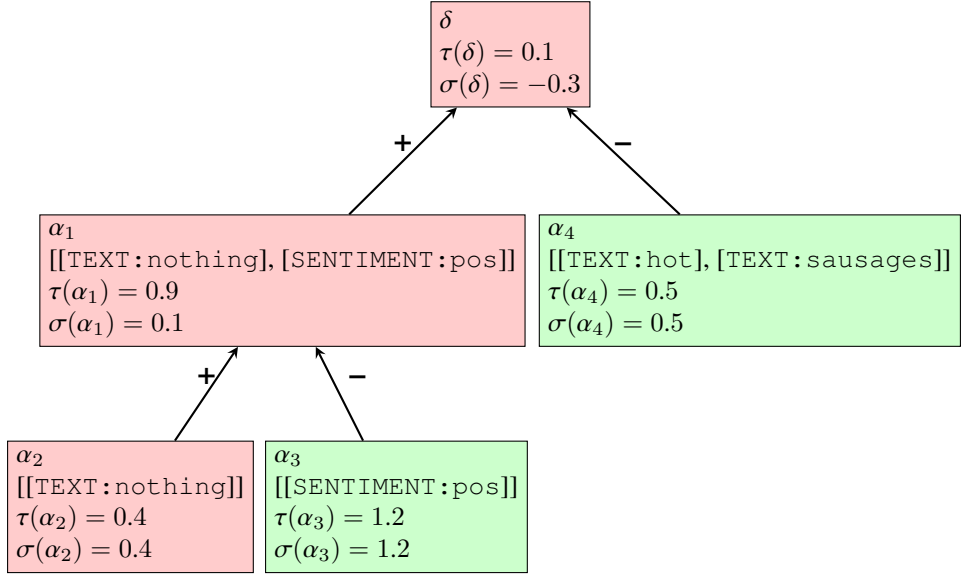
\begin{figure}[!t]
\centering
\small
\begin{tikzpicture}[node distance=2cm]
\node[draw,align=left,fill=red!20] (d) at (1,0) { $\delta$\\$\tau(\delta)=0.1$\\ $\sigma(\delta) = -0.3$};
\node[draw,align=left,fill=red!20] (a2) at (-4,-6) {$\alpha_2$ \\ $\mbox{[\as{TEXT:nothing}]}$\\$\tau(\alpha_2)=0.4$\\ $\sigma(\alpha_2) = 0.4$};
\node[draw,align=left,fill=green!20] (a3) at (-0.5,-6) {$\alpha_3$ \\ $\mbox{[\as{SENTIMENT:pos}]}$ \\$\tau(\alpha_3)=1.2$\\ $\sigma(\alpha_3) = 1.2$};
\node[draw,align=left,fill=green!20] (a4) at (4.15,-3) {$\alpha_4$ \\ $\mbox{[\as{TEXT:hot}, \as{TEXT:sausages}]}$\\$\tau(\alpha_4)=0.5$ \\ $\sigma(\alpha_4) = 0.5$};
\node[draw,align=left,fill=red!20] (a1) at (-2,-3) {$\alpha_1$ \\ $\mbox{[\as{TEXT:nothing}, \as{SENTIMENT:pos}]}$\\$\tau(\alpha_1)=0.9$\\ $\sigma(\alpha_1) = 0.1$};
\tsupp{a1}{d};
\tatt{a4}{d};
\draw [rel] (a3) -- node[anchor=south west] {\large \textbf{--}} (a1);
\tsupp{a2}{a1};
\end{tikzpicture}
\caption{The extracted bottom-up \QBAF for the example in Figure~\ref{fig5:ex}. The color represents the supported class (i.e., green for positive (1) and red for negative (0)). 
(The meaning of the equalities of the form $\tau(x)=v$ and $\sigma(x)=v$ will be explained later.)
}\label{fig5:exbqbaf}
\end{figure}

To explain, both the \TQBAF and the \BQBAF use the same $\Args$, $\tau$, and $c$. If the input text $x$ matches $n$ patterns, $\Args$ will have $n+1$ arguments. Amongst them, $n$ arguments (those of the form $\alpha_i$) are for the $n$ matched patterns, while the other one is for the default argument ($\delta$) corresponding to the bias term $b$ in the LR model. Therefore, the \QBAFs always have at least one argument, which is the default. 
The supported class ($c$) of each argument depends on whether the corresponding weight in the LR model is positive or negative. If $w_i$ is positive, it means that the existence of the pattern $p_i$ contributes to the positive class. So, the supported class of $\alpha_i$ should be positive (1). For the default argument, we consider the sign of the bias term $b$ instead.
Because the supported class encapsulates the sign, the base score ($\tau$) of the argument will be only the absolute value of the corresponding weight.

The extracted \TQBAF and \BQBAF for the example in Figure~\ref{fig5:ex} are shown in Figures~\ref{fig5:extqbaf} and \ref{fig5:exbqbaf}, respectively.
Following Definition~\ref{def:axplrqbaf}, the differences between the \TQBAF and \BQBAF are the $\Atts$ and $\Supps$ components. For the \TQBAF, (arguments for) more specific patterns attack or support (arguments for) more general patterns. The most general patterns, in turn, attack or support the default argument.
Hence, the more general patterns will stay ``closer'' to the default argument (which is usually placed at the top of \QBAFs when using graphs to visualise them in figures, as shown in Figure~\ref{fig5:extqbaf}. 
That is why we call \TQBAFs \textit{top-down}.
Conversely, for the \BQBAF, (arguments for) more general patterns attack or support (arguments for) more specific patterns. The most specific patterns, in turn, attack or support the default argument.
Therefore, the more specific patterns will stay ``closer'' to the default argument, as shown in Figure~\ref{fig5:exbqbaf}, so we call \BQBAFs \textit{bottom-up}.
To decide whether two arguments are related by attack or support, we check the classes supported by the arguments: if they support the same class, then they are related by support; otherwise, by attack.

We believe that both top-down and bottom-up arrangements may be legitimate, but in different situations.
Later, in Section~\ref{subsec:explanations}, we will show that, in \TQBAFs, we explain to users with general patterns first and provide more specific patterns as details when requested.
In \BQBAFs, by contrast, we explain to users with specific patterns first (as they contain more information) and mention general patterns as supporting or opposing reasons.

After this point, when we mention a \QBAF in this paper, we mean that it could be either a \TQBAF or a \BQBAF, unless otherwise stated. We assume that any generic \QBAF is of the form 
$\langle \Args, \Atts, \Supps, \tau, c \rangle$.
Furthermore, following notations in related work \cite{baroni2019fine}, we use $\Atts(a)$ and $\Supps(a)$ to represent sets of arguments attacking and supporting the argument $a$, respectively. 
Formally, $\Atts(a) = \{b \in \Args | (b, a) \in \Atts\}$ and $\Supps(a) = \{b \in \Args | (b, a) \in \Supps\}$.

\begin{prelem} \label{lem:deltaroot}
Given a \QBAF $\langle \Args, \Atts, \Supps, \tau, c \rangle$, then $\delta \notin \Atts(a)$ and $\delta \notin \Supps(a)$ for all $a \in \Args$. So, the out-degree of $\delta$ is 0.
\end{prelem}
Indeed, we can see from $\Atts_T$, $\Supps_T$, $\Atts_B$, and $\Supps_B$ in Definition~\ref{def:axplrqbaf} that $\delta$ never attacks or supports any other argument. So, its out-degree equals 0, and therefore we usually put it at the top of figures (as shown in Figures~\ref{fig5:extqbaf} and \ref{fig5:exbqbaf}).

Additionally, thanks to Definition~\ref{def:axplrqbaf} and Lemma~\ref{lem:deltaroot}, the graph structures underlying any \TQBAF and \BQBAF are directed acyclic graphs (DAGs).

\begin{pretheo}\label{theo:dag}
The graph structure of any \QBAF is a directed acyclic graph (DAG).

\end{pretheo}
\begin{proof}
See Appendix~\ref{proof:theo:dag}.
\end{proof}

\subsection{Strength calculation} \label{subsec:semantics}
After we obtain the \QBAFs, the next step is to calculate the dialectical strength of each argument therein. To make this strength faithful to the underlying \PLR model, we propose the \textit{logistic regression semantics}, given by the strength function $\sigma$, defined next.

\begin{predef} \label{def:logregstr}
The \emph{logistic regression semantics} is defined as the strength function $\sigma:\Args\rightarrow\R$, where, for any $a \in \Args$: 
\begin{equation}\label{eq5:logregstr}
    \sigma(a) = \tau(a) + \sum_{b \in \Supps(a)}\frac{\sigma(b)}{\nu(b)} - \sum_{b \in \Atts(a)}\frac{\sigma(b)}{\nu(b)}
\end{equation}
where $\tau(a)$ is the base score of $a$ and $\nu(b)$ is the out-degree of $b$.
\end{predef}

This semantics can be applied to both \TQBAFs and \BQBAFs.
According to Equation~\ref{eq5:logregstr}, the strength of an argument starts from its base score, and it is increased and decreased by the strengths of its supporters and its attackers, respectively. 
However, the strength of each supporter/attacker must be divided by its out-degree (i.e., $\nu(b)$) before being combined with the base score.
Note that $\nu(b)$ in Equation~\ref{eq5:logregstr} is always greater than or equal to 1 because $b \in \Atts(a)$ or $b \in \Supps(a)$, meaning that $b$ attacks or supports at least one argument (which is $a$). So, no division by 0 may occur in this equation.
Additionally, any argument $a$ with no attackers or supporters (i.e., $\Atts(a) = \Supps(a) = \emptyset$) will have the strength equal to its base score, by Definition~\ref{def:logregstr}.

Because \QBAFs are DAGs (see Theorem~\ref{theo:dag}), we can use topological sorting to define the order to compute the arguments' strengths.
Considering the \TQBAF in Figure~\ref{fig5:extqbaf}, for example, $\alpha_1$ and $\alpha_4$ do not have any attacker or supporter, so their strengths equal their base scores. Next, we can calculate the strengths of $\alpha_2$ and $\alpha_3$, and then $\delta$:
\begin{align*}
    \sigma(\alpha_2) &= \tau(\alpha_2) + \sum_{b \in \Supps(\alpha_2)}\frac{\sigma(b)}{\nu(b)} - \sum_{b \in \Atts(\alpha_2)}\frac{\sigma(b)}{\nu(b)}\\
    &= 0.4 + \sum_{b \in \{\alpha_1\}}\frac{\sigma(b)}{\nu(b)} - \sum_{b \in \emptyset}\frac{\sigma(b)}{\nu(b)} = 0.4 + \frac{0.9}{2} = 0.85 \\
    \sigma(\alpha_3) &= \tau(\alpha_3) + \sum_{b \in \Supps(\alpha_3)}\frac{\sigma(b)}{\nu(b)} - \sum_{b \in \Atts(\alpha_3)}\frac{\sigma(b)}{\nu(b)}\\
    &= 1.2 + \sum_{b \in \emptyset}\frac{\sigma(b)}{\nu(b)} - \sum_{b \in \{\alpha_1\}}\frac{\sigma(b)}{\nu(b)} = 1.2 - \frac{0.9}{2} = 0.75 \\
    \sigma(\delta) &= \tau(\delta) + \sum_{b \in \Supps(\delta)}\frac{\sigma(b)}{\nu(b)} - \sum_{b \in \Atts(\delta)}\frac{\sigma(b)}{\nu(b)}\\
    &= 0.1 + \sum_{b \in \{\alpha_2\}}\frac{\sigma(b)}{\nu(b)} - \sum_{b \in \{\alpha_3, \alpha_4\}}\frac{\sigma(b)}{\nu(b)}
    = 0.1 + \frac{0.85}{1} - \frac{0.75}{1} - \frac{0.5}{1} = -0.3
\end{align*}

All the results are displayed in Figure~\ref{fig5:extqbaf}. Similarly, the strengths are computed for the \BQBAF and shown in Figure~\ref{fig5:exbqbaf}. We can see that the strength of the default arguments $\delta$ of both \TQBAF and \BQBAF is equal to the absolute of the logit $\sum\limits_{i=1}^dw_if_i + b$ of the \PLR model.

\begin{pretheo}\label{theo:predict}
For a given \QBAF, the prediction of the underlying \PLR model can be inferred from the strength of the default argument:
\begin{enumerate}
    \item The predicted probability for the class $c(\delta)$ equals $sigmoid(\sigma(\delta))$.
    \item Hence, if $\sigma(\delta) > 0$, the \PLR model predicts class $c(\delta)$. Otherwise, it predicts the opposite class (i.e., $1-c(\delta)$).
\end{enumerate}
\end{pretheo}
\begin{proof}
See Appendix~\ref{proof:theo:predict}.
\end{proof}

In other words, we can read the prediction from the default argument $\delta$. The negative strength of $\delta$ implies that the argument can no longer support its originally supported class; therefore, the prediction must be the opposite class.
Since $\sigma(\delta)$ is computed from $\tau(\delta)$ and the strengths of the attackers and the supporters of $\delta$, we can use these attackers and supporters as explanation for the prediction.
Furthermore, we may generalize the results of Theorem~\ref{theo:predict} to other arguments $\alpha_i \in \Args$.
For instance, in Figure~\ref{fig5:exbqbaf}, we could say that the pattern [\as{TEXT:nothing}, \as{SENTIMENT:pos}] of $\alpha_1$ (weakly) supports the negative class with $\alpha_1$'s strength of 0.1, but it is not sufficient to make the final prediction become negative 
(indeed, even after the support by $\alpha_1$, the strength of $\delta$ remains negative, meaning that 
the final prediction is no longer the class $\delta$ originally supports, i.e., no longer the negative class).

\subsection{Post-Processing} \label{subsec:post}
We have shown how to extract \QBAFs, equipped with a suitable notion of dialectical strength to match the workings of \PLR so as to serve as a basis for explanation thereof.
Nevertheless, when arguments in these \QBAFs have a negative dialectical strength, the human interpretation of any resulting explanations may be difficult. 
Using Figure~\ref{fig5:extqbaf} as an example, we can see that argument $\alpha_2$ ([\as{TEXT:nothing}]), 
supporting the negative class, \textit{supports} argument $\delta$, which represents the final prediction. 
Due to $\delta$ supporting the negative class and $\sigma(\delta)$ being negative, we can read from the figure that the final prediction is the positive class. However, it is counterintuitive to say that a pattern for the negative class supports the prediction of the positive class.
Hence, we propose a post-processing step for \QBAFs to pave the way towards explanations better aligned with human interpretation.

\begin{predef} \label{def:postprocess}
Given a \QBAF 
$\langle \Args, \Atts, \Supps, \tau, c \rangle$ with 
$\sigma(a)$ the dialectical strength of any $a \in \Args$, the corresponding \emph{post-processed \QBAF}, denoted \emph{\QBAFP},  is defined as $\langle \Args', \Attsp, \Suppsp, \tau', c' \rangle$ where
\begin{itemize}
    \item $\Args' = \Args$.
    \item $\tau':\Args\rightarrow\R$ and $c':\Args\rightarrow\{0, 1\}$ are defined such that, for each $a \in \Args$,
        \begin{itemize}
            \item If $\sigma(a) \geq 0$, then $\tau'(a) = \tau(a)$ and $c'(a) = c(a)$.
            \item If $\sigma(a) < 0$, then $\tau'(a) = -\tau(a)$ and $c'(a) = 1-c(a)$.
        \end{itemize}
    \item $\Attsp = \{(a, b) \in \Atts \cup \Supps | c'(a) \neq c'(b) \wedge \sigma(a) \neq 0\}$.
    \item $\Suppsp = \{(a, b) \in \Atts \cup \Supps | c'(a) = c'(b) \wedge \sigma(a) \neq 0\}$.
\end{itemize}
\end{predef}

According to Definition~\ref{def:postprocess}, to post-process a \QBAF from the previous step, we change the supported class of arguments with negative strengths ($\sigma(a) < 0$) to the other class (i.e., $c'(a) = 1-c(a)$) and flip their base scores to be negative values (i.e., $\tau'(a) = -\tau(a)$). 
Then we re-label attacks and supports between arguments 
according to the new supported classes $c'$ while keeping the direction of the edges intact. 
Figures~\ref{fig5:extqbafpp} and \ref{fig5:exbqbafpp} show the post-processed \QBAFs of Figures~\ref{fig5:extqbaf} and \ref{fig5:exbqbaf}, respectively.
Note that, in this step, 
we also remove any edges where the strengths of the attackers or the supporters equal 0.
As a result of this post-processing, the dialectical strength for all arguments becomes positive (as shown also in Figures~\ref{fig5:extqbafpp} and \ref{fig5:exbqbafpp}). Generally:

\begin{pretheo} \label{theo:sigma'}
Given a \QBAF $\langle \Args, \Atts, \Supps, \tau, c \rangle$ and the corresponding \QBAFP $\langle \Args', \Attsp, \Suppsp, \tau', c' \rangle$, using the logistic regression semantics $\sigma$, let $\sigma(a)$ and $\sigma(a)'$ represent the dialectical strength of $a \in \Args = \Args'$ in \QBAF and \QBAFP, respectively. Then, the following statements are true:
\begin{itemize}
    \item If $\sigma(a) \geq 0$, then $\sigma(a)'= \sigma(a)$.
    \item If $\sigma(a) < 0$, then $\sigma(a)'= -\sigma(a)$.
\end{itemize}
\end{pretheo}
\begin{proof}
See Appendix~\ref{proof:theo:sigma'}.
\end{proof}

\begin{precor} \label{cor:sigma'}
Given a \QBAF and the corresponding \QBAFP, $\sigma(a)'= |\sigma(a)|$ for all $a \in \Args = \Args'$.
\end{precor}

Furthermore, Theorem~\ref{theo:predict} also applies to \QBAFPs, as follows:

\begin{precor} \label{cor:predict'}
For a given \QBAFP, the prediction of the underlying \PLR model is the class $c'(\delta)$ with the predicted probability of $sigmoid(\sigma(\delta)')$.
\end{precor}

Thus, intuitively, the effect of the post-processing step is to flip all the negative strengths to be positive, so we adjust the \QBAF accordingly, while preserving the interpretations of the arguments.
For instance, if the original argument $a$ has $\tau(a) = 0.3$, $c(a) = 1$ and $\sigma(a) = -0.5$, the meaning is that the argument initially supports the positive class with the base score of 0.3, but after taking into account dialectical relations, it supports the negative class instead with strength 0.5. After post-processing, we will obtain $\tau'(a) = -0.3$, $c'(a) = 0$ and $\sigma(a)' = 0.5$, with the (equivalent) meaning that the argument supports the negative class with strength 0.5.

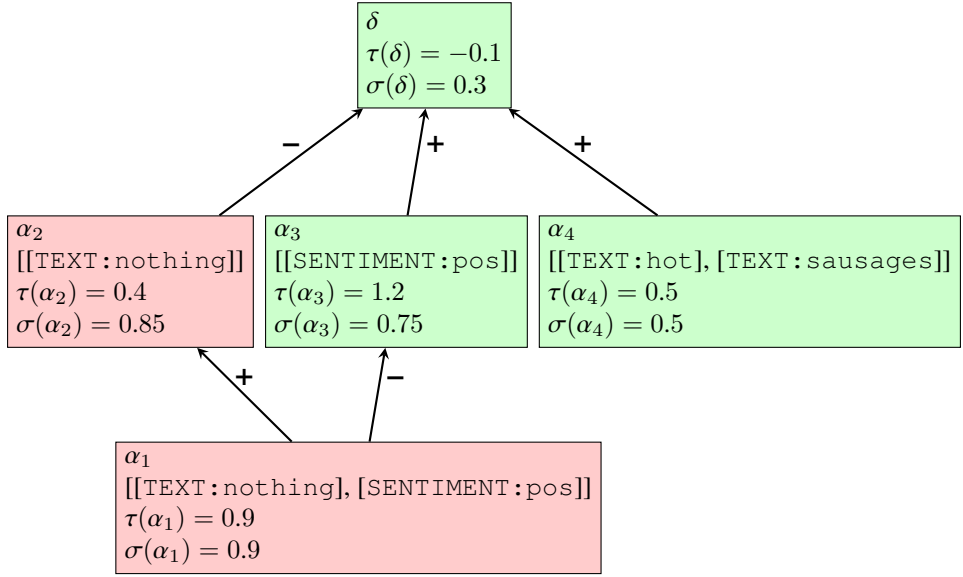
\begin{figure}[!t]
\centering
\small
\begin{tikzpicture}[node distance=2cm]
\node[draw,align=left,fill=green!20] (d) at (0,0) { $\delta$\\$\tau(\delta)=-0.1$\\ $\sigma(\delta) = 0.3$};
\node[draw,align=left,fill=red!20] (a2) at (-4,-3) {$\alpha_2$ \\ $\mbox{[\as{TEXT:nothing}]}$\\$\tau(\alpha_2)=0.4$\\ $\sigma(\alpha_2) = 0.85$};
\node[draw,align=left,fill=green!20] (a3) at (-0.5,-3) {$\alpha_3$ \\ $\mbox{[\as{SENTIMENT:pos}]}$ \\$\tau(\alpha_3)=1.2$\\ $\sigma(\alpha_3) = 0.75$};
\node[draw,align=left,fill=green!20] (a4) at (4.15,-3) {$\alpha_4$ \\ $\mbox{[\as{TEXT:hot}, \as{TEXT:sausages}]}$\\$\tau(\alpha_4)=0.5$ \\ $\sigma(\alpha_4) = 0.5$};
\node[draw,align=left,fill=red!20] (a1) at (-1,-6) {$\alpha_1$ \\ $\mbox{[\as{TEXT:nothing}, \as{SENTIMENT:pos}]}$\\$\tau(\alpha_1)=0.9$\\ $\sigma(\alpha_1) = 0.9$};
\tatt{a2}{d};
\draw [rel] (a3) -- node[anchor=south west] {\large \textbf{+}} (d);
\tsupp{a4}{d};
\tsupp{a1}{a2};
\draw [rel] (a1) -- node[anchor=south west] {\large \textbf{--}} (a3);
\end{tikzpicture}
\caption{The extracted top-down \QBAF in Figure~\ref{fig5:extqbaf} after being post-processed.}\label{fig5:extqbafpp}
\end{figure}

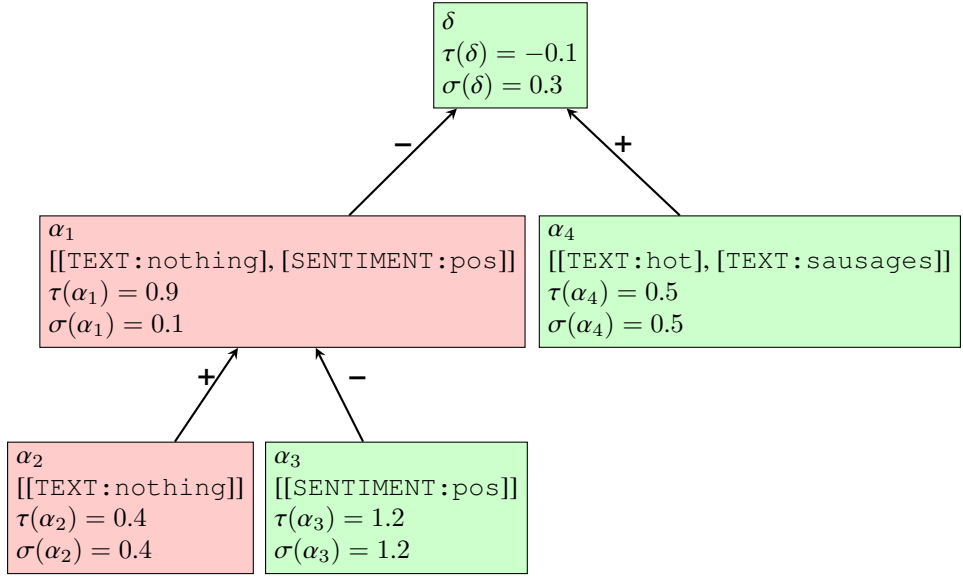
\begin{figure}[!t]
\centering
\small
\begin{tikzpicture}[node distance=2cm]
\node[draw,align=left,fill=green!20] (d) at (1,0) { $\delta$\\$\tau(\delta)=-0.1$\\ $\sigma(\delta) = 0.3$};
\node[draw,align=left,fill=red!20] (a2) at (-4,-6) {$\alpha_2$ \\ $\mbox{[\as{TEXT:nothing}]}$\\$\tau(\alpha_2)=0.4$\\ $\sigma(\alpha_2) = 0.4$};
\node[draw,align=left,fill=green!20] (a3) at (-0.5,-6) {$\alpha_3$ \\ $\mbox{[\as{SENTIMENT:pos}]}$ \\$\tau(\alpha_3)=1.2$\\ $\sigma(\alpha_3) = 1.2$};
\node[draw,align=left,fill=green!20] (a4) at (4.15,-3) {$\alpha_4$ \\ $\mbox{[\as{TEXT:hot}, \as{TEXT:sausages}]}$\\$\tau(\alpha_4)=0.5$ \\ $\sigma(\alpha_4) = 0.5$};
\node[draw,align=left,fill=red!20] (a1) at (-2,-3) {$\alpha_1$ \\ $\mbox{[\as{TEXT:nothing}, \as{SENTIMENT:pos}]}$\\$\tau(\alpha_1)=0.9$\\ $\sigma(\alpha_1) = 0.1$};
\tatt{a1}{d};
\tsupp{a4}{d};
\draw [rel] (a3) -- node[anchor=south west] {\large \textbf{--}} (a1);
\tsupp{a2}{a1};
\end{tikzpicture}
\caption{The extracted bottom-up \QBAF in Figure~\ref{fig5:exbqbaf} after being post-processed.}\label{fig5:exbqbafpp}
\end{figure}

\section{Analyzing Properties of \QBAF and \QBAFP} \label{subsec:properties}

In this section, we analyze the logistic regression semantics $\sigma$, when applied to \QBAFs and \QBAFPs, according to 11 group properties of gradual semantics proposed in \cite{baroni2019fine}.
These properties have been used to evaluate many argumentation frameworks and semantics in the literature \cite{albini2020dax,potyka2021interpreting,sukpanichnant2021lrp} in order to ensure that the frameworks and semantics will lead to explanations that are consistent with general human reasoning and debate.
Table~\ref{table:new} gives the formal definition of these properties for QBAFs $\langle \Args_*, \Atts_*, \Supps_*, \tau_* \rangle$
under semantics $\sigma_*$. Note that these properties apply naturally to \QBAFs of the form $\langle \Args, \Atts, \Supps, \tau, c \rangle$ and \QBAFPs of the form $\langle \Args', \Attsp, \Suppsp, \tau', c' \rangle$, 
under the LR semantics $\sigma$, for 
$\Args_*=\Args$ or $\Args_*=\Args'$,
$\Atts_*=\Atts$ or $\Atts_*=\Attsp$, and so on.
The definition of $<$ between two sets used in GP10 and GP11 is defined as follows. Given $P\subseteq \Args$ and $Q\subseteq \Args$, $P \leq Q$ if and only if there exists an injective mapping $f$ from $P$ to $Q$ such that $\forall \alpha \in P, \sigma(\alpha) \leq \sigma(f(\alpha))$. Furthermore, $P < Q$ if and only if $P \leq Q$ but $Q \not \leq P$. 

\begin{table}
\caption{Dialectical properties for 
QBAFs $\langle \Args_*, \Atts_*, \Supps_*, \tau_* \rangle$
under semantics $\sigma_*$ (adapted from \cite{baroni2019fine}).
\label{table:new}}
\begin{tabular}{|c|l|}
    \hline
    \textbf{GP1.} & If $\Atts_*(\alpha)=\emptyset$ and $\Supps_*(\alpha)=\emptyset$, then $\sigma_*(\alpha) = \tau_*(\alpha)$. \\ \hline
    \textbf{GP2.} & If $\Atts_*(\alpha)\neq\emptyset$ and $\Supps_*(\alpha)=\emptyset$, then $\sigma_*(\alpha) < \tau_*(\alpha)$. \\ \hline
    \textbf{GP3.} & If $\Atts_*(\alpha)=\emptyset$ and $\Supps_*(\alpha)\neq\emptyset$, then $\sigma_*(\alpha) > \tau_*(\alpha)$. \\ \hline
    \textbf{GP4.} & If $\sigma_*(\alpha) < \tau_*(\alpha)$, then $\Atts_*(\alpha)\neq\emptyset$. \\ \hline
    \textbf{GP5.} & If $\sigma_*(\alpha) > \tau_*(\alpha)$, then $\Supps_*(\alpha)\neq\emptyset$. \\ \hline
    \textbf{GP6.} & If $\Atts_*(\alpha)=\Atts_*(\beta)$, $\Supps_*(\alpha)=\Supps_*(\beta)$, and $\tau_*(\alpha)=\tau_*(\beta)$, then $\sigma_*(\alpha)=\sigma_*(\beta)$. \\ \hline
    \textbf{GP7.} & If $\Atts_*(\alpha)\subset\Atts_*(\beta)$, $\Supps_*(\alpha)=\Supps_*(\beta)$, and $\tau_*(\alpha)=\tau_*(\beta)$, then $\sigma_*(\alpha)>\sigma_*(\beta)$. \\ \hline
    \textbf{GP8.} & If $\Atts_*(\alpha)=\Atts_*(\beta)$, $\Supps_*(\alpha)\subset\Supps_*(\beta)$, and $\tau_*(\alpha)=\tau_*(\beta)$, then $\sigma_*(\alpha)<\sigma_*(\beta)$. \\ \hline
    \textbf{GP9.} & If $\Atts_*(\alpha)=\Atts_*(\beta)$, $\Supps_*(\alpha)=\Supps_*(\beta)$, and $\tau_*(\alpha)<\tau_*(\beta)$, then $\sigma_*(\alpha)<\sigma_*(\beta)$. \\ \hline
    \textbf{GP10.} & If $\Atts_*(\alpha)<\Atts_*(\beta)$, $\Supps_*(\alpha)=\Supps_*(\beta)$, and $\tau_*(\alpha)=\tau_*(\beta)$, then $\sigma_*(\alpha)>\sigma_*(\beta)$. \\ \hline
    \textbf{GP11.} & If $\Atts_*(\alpha)=\Atts_*(\beta)$, $\Supps_*(\alpha)<\Supps_*(\beta)$, and $\tau_*(\alpha)=\tau_*(\beta)$, then $\sigma_*(\alpha)<\sigma_*(\beta)$. \\ \hline
\end{tabular}
\end{table}

\begin{table*}[t]
\caption{Summary of the group properties for gradual semantics \cite{baroni2019fine} satisfied or unsatisfied by the logistic regression semantics $\sigma$ when applied on \QBAFs and \QBAFPs.}
    \label{tab5:properties}
    \centering
    \begin{tabular}{|l|c c c c c c c c c c c|}
    
    \hline
        & GP1 & GP2 & GP3 & GP4 & GP5 & GP6 & GP7 & GP8 & GP9 & GP10 & GP11 \\ \hline 
        $\langle \QBAF, \sigma \rangle$ & \cmark & \xmark & \xmark & \xmark & \xmark & \cmark & \xmark & \xmark & \cmark & \xmark & \xmark \\
        $\langle \QBAFP, \sigma \rangle$  & \cmark & \cmark & \cmark & \cmark & \cmark & \cmark & \cmark & \cmark & \cmark & \xmark & \xmark \\ \hline
    \end{tabular}
\end{table*}

Table~\ref{tab5:properties} summarizes our results (the proofs are in Appendix~\ref{proof:properties}).
To briefly explain here, 
GP1 is satisfied by both $\langle \QBAF, \sigma \rangle$ and $\langle \QBAFP, \sigma \rangle$ because when there is neither attacker nor supporter, the right side of Equation~\ref{eq5:logregstr} has only $\tau(\alpha)$ left, making the argument's strength equal its base score.
Meanwhile, GP2-GP5 are not satisfied by $\langle \QBAF, \sigma \rangle$ since the strengths of attackers and supporters of \QBAF (not yet post-processed) could be negative. 
As a result, when an argument has only attackers, it may not be the case that its strength becomes lower than its base score (i.e., GP2 may not be satisfied).
Similarly, when an argument has only supporters, it may not be the case that its strength becomes higher than its base score (i.e., GP3 may not be satisfied).
Furthermore, the strength of an argument in $\langle \QBAF, \sigma \rangle$ could be less than its base score due to not only attackers with positive strengths but also supporters with negative strengths, making GP4 unsatisfied.
Similarly, when the strength of an argument in $\langle \QBAF, \sigma \rangle$ is higher than its base score, it could also be due to attackers with negative strengths (not only supporters with positive strengths), making GP5 unsatisfied.
In contrast, after we post-process \QBAF to be \QBAFP, no argument strengths can be negative (see Corollary~\ref{cor:sigma'}); therefore, GP2-GP5 are satisfied by $\langle \QBAFP, \sigma \rangle$.
Next, GP6 and GP9 are satisfied by both $\langle \QBAF, \sigma \rangle$ and $\langle \QBAFP, \sigma \rangle$ as we can easily see from  Equation~\ref{eq5:logregstr}.
As for GP2-GP5, GP7 and GP8 may not be satisfied by $\langle \QBAF, \sigma \rangle$ due to the fact that, in \QBAF, the strengths of attackers and supporters could be negative, whereas 
$\langle \QBAFP, \sigma \rangle$ satisfies GP7 and GP8 since \QBAFP does not suffer from negative strengths.
Lastly, both $\langle \QBAF, \sigma \rangle$ and $\langle \QBAFP, \sigma \rangle$ do not satisfy GP10 and GP11 because the $<$ relation imposes a condition only on argument strengths while our semantics $\sigma$ considers not only the strengths of the attackers and the supporters but also their out-degrees. 
For illustrative counterexamples of these GPs, please see  Appendix~\ref{proof:properties}.

In conclusion, $\langle \QBAFP, \sigma \rangle$ satisfies nine out of the eleven group properties, while $\langle \QBAF, \sigma \rangle$ satisfies only three. This means that our post-processing step is important to make the argumentation framework align better with human interpretation and become more suitable for generating local explanations.

\section{Presenting \xlr to humans} \label{subsec:explanations}

Presenting the whole \QBAFP as a local explanation to lay users 
may not be a good idea since the graph could be very complicated (in terms of the number of arguments, relations, and depth). Also, the notions of attack and support may not be familiar to the users. 
So, the last step of \xlr is extracting the explanation from the \QBAFP.
We know from Theorem~\ref{theo:predict} and Corollary~\ref{cor:predict'} that the prediction of the LR model is associated to the strength of the default argument $\delta$. 
Hence, we can explain the prediction based on how $\sigma(\delta)'$ was calculated.
The value of $\sigma(\delta)'$ depends on $\tau'(\delta)$ (corresponding to the bias term in LR) and the strength $\sigma$ of all the attackers and supporters of $\delta$.
Therefore, we return, as the local explanation for $c'(\delta)$, a list of triplets $(p_{j}, \pi(p_{j}, x), \sigma(\alpha_j)')$  where $x$ is the input text, $\alpha_j$ (representing the pattern $p_j$) is one of the $k$ strongest supporters of $\delta$, and $\pi(p_{j}, x)$ is a part of $x$ that matches the pattern $p_{j}$.
If we want both evidence for and counter-evidence against the prediction, we can show triplets $(p_{j}, \pi(p_{j}, x), \sigma(\alpha_j)')$ for the top $k$ supporters and attackers with the highest $\sigma(\alpha_j)'$.
We call explanations of this form  \textbf{shallow \xlrs}.
Figure~\ref{fig5:shallowAXPLR} shows an example of shallow \xlr (extracted from a \BQBAFP) for the deceptive review detection task\footnote{This task aims to classify whether a review is genuine or fake.} where the color intensity represents the strengths of the arguments.
Shallow \xlrs\ are similar to the flat logistic regression explanations (FLXs) introduced in Section~\ref{subsec:plr}.
The only differences are that \textit{(i)} FLXs select the top $k$ patterns based on the size of $w_{j}f_{j}$ (which is equivalent to $\tau(\alpha_j)$) while our shallow \xlrs select top $k$ arguments based on the dialectical strength $\sigma(\alpha_j)$, and \textit{(ii)} any patterns matched in $x$ can be in FLXs whereas only attackers and supporters of $\delta$ can be in shallow \xlrs. 

\begin{figure}[t!]
    \centering
    \includegraphics[width=\textwidth]{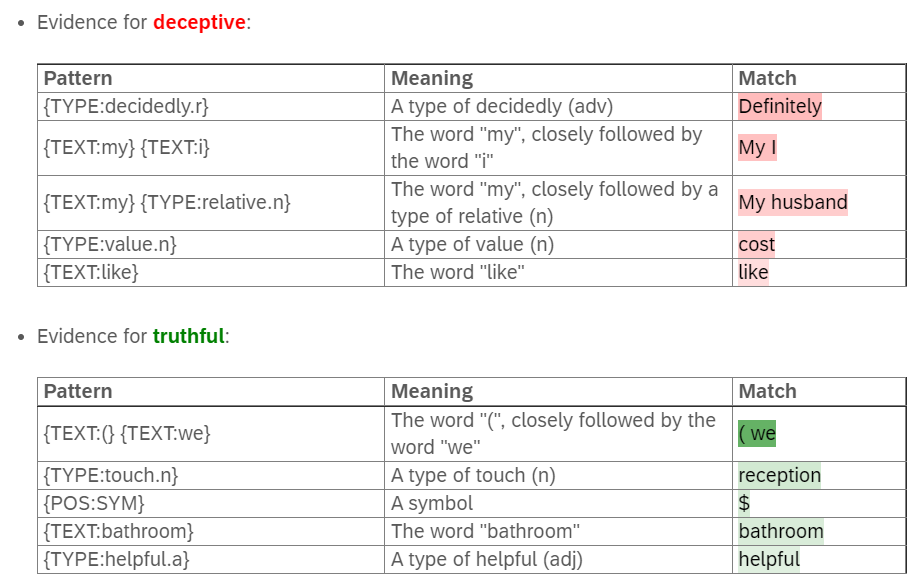}
    \caption{Example of shallow AXPLR for deceptive review detection.
    The patterns shown correspond to strongest supporters and attackers of $\delta$. The meaning of each pattern/argument is also provided. The color and its intensity represent the supported class and the strengths of the arguments, respectively.}
    \label{fig5:shallowAXPLR}
\end{figure}

Shallow \xlrs leverage only the attackers and supporters of $\delta$, but ignores the additional information available in the \QBAFP. Therefore, we propose another variation of \xlrs, called \textbf{deep \xlrs}, which also use other arguments and relations in the \QBAFP. Basically, deep \xlrs expand shallow \xlrs by additionally allowing  users to see attackers and supporters (if any) of arguments in shallow \xlrs as well as ``deeper'' arguments in the \QBAFP until there is no attacker or supporter for those arguments. 
Figure~\ref{fig5:deepAXPLR} shows a deep \xlr, explaining the same example and using the same \BQBAFP as the shallow \xlr in Figure~\ref{fig5:shallowAXPLR} does.
Note that a deep \xlr from a \BQBAFP (as in Figure~\ref{fig5:deepAXPLR}) shows specific patterns to the users first (as they directly support or attack $\delta$) and hides more general patterns as supporting or opposing reasons (to be expanded) in deeper levels.
In contrast, due to the opposite way of drawing attacks and supports, a deep \xlr from a \TQBAFP explains to users with general patterns first and provides more specific patterns as expandable details in deeper levels.

\begin{figure}[t!]
    \centering
    \includegraphics[width=\textwidth]{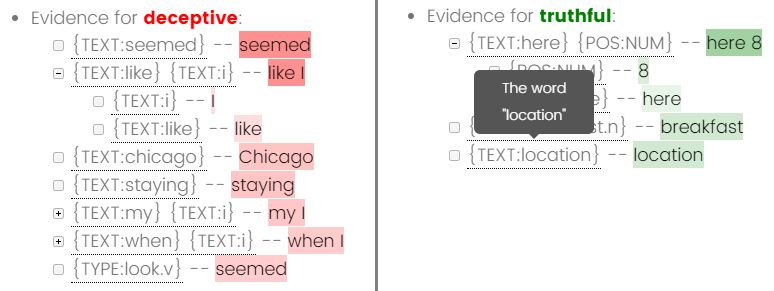}
    \caption{Example of deep AXPLR for deceptive review detection. A user can expand some patterns/arguments to see their sub-patterns (i.e., their attackers and/or supporters) such as, on the left, [\as{TEXT:i}] and [\as{TEXT:like}] supporting [\as{TEXT:like}, \as{TEXT:i}]
    . The meaning of each pattern is provided as a tooltip. The color and its intensity represent the supported class and the strengths of the arguments, respectively.}
    \label{fig5:deepAXPLR}
\end{figure}

Shallow and deep \xlr are just two examples of explanations that can be drawn from \QBAFPs. 
We leave the study of other forms of explanations, such as conversational explanations \cite{cocarascu2019extracting} and counterfactual explanations \cite{albini2021interpreting}, for future work.

\section{Experimental Setup} \label{sec:setup}

To evaluate \xlr, we conducted both empirical and human evaluations. For the empirical evaluation, we calculated some statistics for the \QBAFs extracted for target examples and performed analyses concerning \textit{sufficiency} of the generated explanations (see Section~\ref{sec:empirical}). 
For the human evaluation, we \textit{(i)} assessed plausibility of the explanations, i.e., how well the explanations from \xlr align with human explanations compared to a standard method for explaining LR results (see Section~\ref{sec:plausibility}), and \textit{(ii)} assessed how well \xlr can teach and support humans to perform a new task (see Section~\ref{sec:tutorial}). 

In the experiments, we targeted binary text classification using three English datasets as shown in Table~\ref{tab5:dataset}. The table also shows the classes we consider as positive and negative when running \grasp and \xlr
and the number of examples for each data split (used for training, developing, and testing the models; please see Appendix~\ref{app:terminology} for more details about these splits). 
Specifically, the datasets are:

\begin{itemize}
    \item \textbf{SMS Spam Collection} \cite{almeida2011contributions} focusing on detecting spams in a collection of SMS (short message service) messages. The dataset is imbalanced, containing 13.40\% spam messages and 86.60\% ham messages (i.e., non-spams). 
    \item \textbf{Amazon Clothes} \cite{he2016ups},  focusing on classifying whether a review (of clothing, shoes, and jewelry products) has positive or negative sentiment. The overall dataset is balanced.
    \item \textbf{Deceptive Hotel Reviews} \cite{ott-etal-2011-finding,ott-etal-2013-negative} focusing on identifying whether a given hotel review is truthful (genuine) or deceptive (fake). There are 1600 reviews in total for 20 hotels. For each hotel, there are 20 truthful positive, 20 truthful negative, 20 deceptive positive, and 20 deceptive negative reviews (positive and negative here refer to the review sentiment).
\end{itemize}

\begin{table}[t!]
\caption{Datasets used in the experiments.} \label{tab5:dataset}
\centering
\begin{tabular}{l c c c} 
 \hline
 Dataset & Positive class & Negative class & Training / Development / Testing\\
 \hline
 SMS Spam Collection  & Spam & Not spam & 3567 / 892 / 1115\\
 Amazon Clothes & Positive & Negative & 3000 / 300 / 10000 \\
 Deceptive Hotel Reviews & Deceptive & Truthful & 1024 / 256 / 320\\ 
 \hline
 \end{tabular}
\end{table}

For the LR classifiers of the first two datasets, the \grasp patterns were constructed with lemma, part-of-speech tags (POS), wordnet hypernyms, and sentiment attributes. We used alphabet size of 200, allowed two gaps in the patterns, and generated 100 patterns in total.
For the last dataset (Deceptive Hotel Reviews), the settings were the same except that we used text attributes (capturing the whole word) instead of the lemma attributes and we generated 200 patterns in total. 
The performance of the LR classifiers of the three datasets are reported in Table~\ref{tab5:accuracy}.
\textit{Accuracy} is the percentage of correct predictions on the test set, while \textit{F1} is a harmonic mean of \textit{Precision} and \textit{Recall} of the model. \textit{Positive F1} and \textit{Negative F1} are F1s when we consider the positive and the negative classes as the main class. These two F1s are then averaged to be \textit{Macro F1}. For more details about the evaluation metrics, please see Appendix~\ref{app:terminology}. 

\begin{table}[t!]
\caption{Performance of the pattern-based LR models on the test sets.} \label{tab5:accuracy}
\centering
\begin{tabular}{L{0.33\linewidth} C{0.12\linewidth} C{0.12\linewidth} C{0.11\linewidth} C{0.13\linewidth}} 
 \hline
 Dataset & Positive F1 & Negative F1 & Macro F1 & Accuracy\\
 \hline
 SMS Spam Collection  & 0.891 & 0.986 & 0.939 & 0.975 \\
 Amazon Clothes & 0.836 & 0.836 & 0.836 & 0.836 \\
 Deceptive Hotel Reviews & 0.847 & 0.859 & 0.853 & 0.853 \\ 
 \hline
 \end{tabular}
\end{table}

\section{Experiment 1: Empirical Evaluation} \label{sec:empirical}

We divide the empirical evaluation into two parts. The first part discusses the statistics for \QBAFs and \QBAFPs we generated from the test sets. This helps us understand what the argumentation graphs look like on average.
The second part focuses on \textit{sufficiency}, aiming to answer ``How many supporting arguments are needed on average so as to sufficiently make the model predict what it predicts?''.
This helps us decide how many arguments we should show in \xlrs\ generally.

\subsection{Statistics for \QBAFs and \QBAFPs}
Tables~\ref{tab5:statspam}-\ref{tab5:statdeceptive} show the statistics of the \QBAFs and \QBAFPs for the SMS Spam Collection, Amazon Clothes, and Deceptive Hotel Reviews, respectively.
We measure the number of arguments supporting positive and negative classes ($\Args^{+,\delta}$ and
$\Args^{-,\delta}$ respectively, defined below), the number of pairs in the (attacks and support) relations ($\mathcal{R}$), and specifically those not involving the default argument
($\mathcal{R_{\backslash\delta}}$):
\begin{equation*}
\begin{aligned}[c]
\Args^{+,\delta} &= \{a\in\Args|c(a)=1\} \\
\Args^{-,\delta} &= \{a\in\Args|c(a)=0\}
\end{aligned}
\qquad\qquad
\begin{aligned}[c]
\mathcal{R} &= \Atts \cup \Supps \\
\mathcal{R_{\backslash\delta}} &= \{(a, b) \in \mathcal{R} | b \neq \delta\}
\end{aligned}
\end{equation*}

Note that, if $\mathcal{R_{\backslash\delta}}=\emptyset$, then the generated \xlr amounts to a FLX where we do not consider relationships between features.
We consider both the statistics for the whole test sets and the statistics for each of the four possible situations -- true positives (TP), true negatives (TN), false positives (FP), and false negatives (FN). TP and TN are the cases where the model correctly predicts that the true class is 1 and 0, respectively. 
FP refers to cases when the predicted label is 1 but the true label is 0.
On the contrary, FN refers to cases when the predicted label is 0 but the true label is 1.
For more details, please see Appendix~\ref{app:terminology}.

\begin{table}[hbt!]
\caption{Statistics (Average $\pm$ SD) of \QBAF and \QBAFP for the SMS Spam Collection dataset. TP, TN, FP, and FN stand for true positives, true negatives, false positives, and false negatives, respectively. 
$\Args^{+,\delta}, \Args^{-,\delta}$ are sets of arguments supporting positive and negative classes, respectively. 
$\mathcal{R}$ is a set of (attack and support) relations, whereas $\mathcal{R}_{\backslash\delta}$ is a subset of $\mathcal{R}$ containing specifically those not involving the default argument $\delta$.
}
\label{tab5:statspam}
\centering
\begin{tabular}{l r r r r} 
 \hline
 Measurement & \TQBAF & \TQBAFP & \BQBAF & \BQBAFP\\
 \hline
 \# Examples & \multicolumn{4}{c}{1115 (TP: 115, TN: 972, FP: 5, FN: 23)} \\ \hline
$|\Args|$& 10.08$\pm$11.55& 10.08$\pm$11.55& 10.08$\pm$11.55& 10.08$\pm$11.55\\
- TP& 35.98$\pm$12.26& 35.98$\pm$12.26& 35.98$\pm$12.26& 35.98$\pm$12.26\\
- TN& 6.66$\pm$6.01& 6.66$\pm$6.01& 6.66$\pm$6.01& 6.66$\pm$6.01\\
- FP& 36.00$\pm$11.98& 36.00$\pm$11.98& 36.00$\pm$11.98& 36.00$\pm$11.98\\
- FN& 19.30$\pm$9.66& 19.30$\pm$9.66& 19.30$\pm$9.66& 19.30$\pm$9.66\\ \hline
$|\Args^{+,\delta}|$& 5.86$\pm$7.12& 6.25$\pm$7.77& 5.86$\pm$7.12& 6.47$\pm$8.07\\
- TP& 22.36$\pm$7.17& 24.68$\pm$7.70& 22.36$\pm$7.17& 25.42$\pm$7.52\\
- TN& 3.70$\pm$3.57& 3.85$\pm$3.69& 3.70$\pm$3.57& 4.00$\pm$4.01\\
- FP& 20.40$\pm$6.31& 22.20$\pm$6.87& 20.40$\pm$6.31& 24.20$\pm$6.02\\
- FN& 11.26$\pm$4.73& 11.74$\pm$4.83& 11.26$\pm$4.73& 12.48$\pm$5.69\\ \hline
$|\Args^{-,\delta}|$& 4.22$\pm$4.63& 3.83$\pm$4.14& 4.22$\pm$4.63& 3.61$\pm$3.81\\
- TP& 13.63$\pm$5.60& 11.30$\pm$5.37& 13.63$\pm$5.60& 10.57$\pm$5.41\\
- TN& 2.96$\pm$2.66& 2.81$\pm$2.65& 2.96$\pm$2.66& 2.66$\pm$2.33\\
- FP& 15.60$\pm$5.77& 13.80$\pm$5.50& 15.60$\pm$5.77& 11.80$\pm$6.06\\
- FN& 8.04$\pm$5.17& 7.57$\pm$5.29& 8.04$\pm$5.17& 6.83$\pm$4.25\\ \hline
$|\mathcal{R}|$& 11.64$\pm$16.77& 11.64$\pm$16.77& 13.82$\pm$21.40& 13.82$\pm$21.40\\
- TP& 49.31$\pm$19.63& 49.31$\pm$19.63& 61.43$\pm$26.11& 61.43$\pm$26.11\\
- TN& 6.69$\pm$8.17& 6.69$\pm$8.17& 7.57$\pm$10.37& 7.57$\pm$10.37\\
- FP& 50.60$\pm$21.14& 50.60$\pm$21.14& 64.20$\pm$28.35& 64.20$\pm$28.35\\
- FN& 23.96$\pm$14.27& 23.96$\pm$14.27& 29.17$\pm$19.09& 29.17$\pm$19.09\\ \hline
$|\mathcal{R}_{\backslash\delta}|$& 8.32$\pm$14.76& 8.32$\pm$14.76& 8.32$\pm$14.76& 8.32$\pm$14.76\\
- TP& 40.81$\pm$18.81& 40.81$\pm$18.81& 40.81$\pm$18.81& 40.81$\pm$18.81\\
- TN& 4.06$\pm$7.09& 4.06$\pm$7.09& 4.06$\pm$7.09& 4.06$\pm$7.09\\
- FP& 43.40$\pm$21.41& 43.40$\pm$21.41& 43.40$\pm$21.41& 43.40$\pm$21.41\\
- FN& 18.43$\pm$13.93& 18.43$\pm$13.93& 18.43$\pm$13.93& 18.43$\pm$13.93\\ \hline
$|\Atts|$& 6.39$\pm$8.48& 5.31$\pm$6.33& 7.51$\pm$10.90& 5.73$\pm$7.93\\
- TP& 24.92$\pm$9.91& 16.50$\pm$8.10& 31.50$\pm$13.48& 20.22$\pm$11.75\\
- TN& 3.95$\pm$4.43& 3.74$\pm$4.15& 4.36$\pm$5.43& 3.75$\pm$4.71\\
- FP& 26.60$\pm$10.83& 20.60$\pm$10.36& 32.40$\pm$14.84& 22.40$\pm$12.95\\
- FN& 12.78$\pm$7.11& 12.13$\pm$6.59& 15.30$\pm$9.18& 13.35$\pm$7.79\\ \hline
$|\Supps|$& 5.24$\pm$8.49& 6.33$\pm$10.88& 6.31$\pm$10.72& 8.10$\pm$13.88\\
- TP& 24.39$\pm$10.38& 32.81$\pm$12.21& 29.93$\pm$13.40& 41.22$\pm$15.19\\
- TN& 2.74$\pm$3.99& 2.95$\pm$4.28& 3.21$\pm$5.20& 3.82$\pm$5.85\\
- FP& 24.00$\pm$10.89& 30.00$\pm$11.07& 31.80$\pm$14.04& 41.80$\pm$15.69\\
- FN& 11.17$\pm$7.51& 11.83$\pm$7.95& 13.87$\pm$10.33& 15.83$\pm$11.62\\ \hline
 \end{tabular}
\end{table}

\begin{table}[hbt!]
\caption{Statistics (Average $\pm$ SD) of \QBAF and \QBAFP for the Amazon Clothes dataset. TP, TN, FP, and FN stand for true positives, true negatives, false positives, and false negatives, respectively. 
$\Args^{+,\delta}, \Args^{-,\delta}$ are sets of arguments supporting positive and negative classes, respectively. 
$\mathcal{R}$ is a set of (attack and support) relations, whereas $\mathcal{R}_{\backslash\delta}$ is a subset of $\mathcal{R}$ containing specifically those not involving the default argument $\delta$.
}
\label{tab5:statsentiment}
\centering
\begin{tabular}{l r r r r} 
 \hline
 Measurement & \TQBAF & \TQBAFP & \BQBAF & \BQBAFP\\
 \hline
 \# Examples & \multicolumn{4}{c}{10000 (TP: 4176, TN: 4186, FP: 848, FN: 790)} \\ \hline
$|\Args|$& 16.09$\pm$8.13& 16.09$\pm$8.13& 16.09$\pm$8.13& 16.09$\pm$8.13\\
- TP& 14.85$\pm$7.57& 14.85$\pm$7.57& 14.85$\pm$7.57& 14.85$\pm$7.57\\
- TN& 17.45$\pm$8.21& 17.45$\pm$8.21& 17.45$\pm$8.21& 17.45$\pm$8.21\\
- FP& 14.99$\pm$8.14& 14.99$\pm$8.14& 14.99$\pm$8.14& 14.99$\pm$8.14\\
- FN& 16.57$\pm$9.31& 16.57$\pm$9.31& 16.57$\pm$9.31& 16.57$\pm$9.31\\ \hline
$|\Args^{+,\delta}|$& 7.39$\pm$4.36& 7.45$\pm$4.50& 7.39$\pm$4.36& 7.43$\pm$4.88\\
- TP& 8.91$\pm$4.18& 9.74$\pm$4.07& 8.91$\pm$4.18& 10.22$\pm$4.51\\
- TN& 5.93$\pm$3.98& 5.20$\pm$3.76& 5.93$\pm$3.98& 4.75$\pm$3.68\\
- FP& 7.80$\pm$4.27& 8.47$\pm$4.12& 7.80$\pm$4.27& 8.49$\pm$4.34\\
- FN& 6.65$\pm$4.55& 6.10$\pm$4.29& 6.65$\pm$4.55& 5.84$\pm$4.37\\ \hline
$|\Args^{-,\delta}|$& 8.70$\pm$5.23& 8.64$\pm$5.89& 8.70$\pm$5.23& 8.65$\pm$6.24\\
- TP& 5.94$\pm$3.99& 5.10$\pm$4.18& 5.94$\pm$3.99& 4.63$\pm$4.27\\
- TN& 11.52$\pm$4.96& 12.25$\pm$5.36& 11.52$\pm$4.96& 12.71$\pm$5.53\\
- FP& 7.19$\pm$4.18& 6.52$\pm$4.38& 7.19$\pm$4.18& 6.50$\pm$4.44\\
- FN& 9.92$\pm$5.14& 10.47$\pm$5.49& 9.92$\pm$5.14& 10.73$\pm$5.57\\ \hline
$|\mathcal{R}|$& 17.64$\pm$10.25& 17.64$\pm$10.25& 21.68$\pm$13.29& 21.68$\pm$13.29\\
- TP& 16.33$\pm$9.78& 16.33$\pm$9.78& 20.55$\pm$12.49& 20.55$\pm$12.49\\
- TN& 19.10$\pm$10.20& 19.10$\pm$10.20& 23.07$\pm$13.47& 23.07$\pm$13.47\\
- FP& 16.38$\pm$10.47& 16.38$\pm$10.47& 20.06$\pm$13.57& 20.06$\pm$13.57\\
- FN& 18.18$\pm$11.57& 18.18$\pm$11.57& 22.03$\pm$15.32& 22.03$\pm$15.32\\ \hline
$|\mathcal{R}_{\backslash\delta}|$& 12.71$\pm$8.73& 12.71$\pm$8.73& 12.71$\pm$8.73& 12.71$\pm$8.73\\
- TP& 12.34$\pm$8.42& 12.34$\pm$8.42& 12.34$\pm$8.42& 12.34$\pm$8.42\\
- TN& 13.30$\pm$8.74& 13.30$\pm$8.74& 13.30$\pm$8.74& 13.30$\pm$8.74\\
- FP& 11.68$\pm$9.01& 11.68$\pm$9.01& 11.68$\pm$9.01& 11.68$\pm$9.01\\
- FN& 12.65$\pm$9.77& 12.65$\pm$9.77& 12.65$\pm$9.77& 12.65$\pm$9.77\\ \hline
$|\Atts|$& 5.27$\pm$3.79& 5.11$\pm$3.96& 8.10$\pm$5.54& 4.08$\pm$3.47\\
- TP& 4.74$\pm$3.55& 5.50$\pm$4.38& 8.20$\pm$5.18& 3.72$\pm$3.43\\
- TN& 5.73$\pm$3.86& 4.56$\pm$3.27& 8.00$\pm$5.70& 4.12$\pm$3.31\\
- FP& 5.20$\pm$3.82& 6.30$\pm$4.66& 8.09$\pm$5.73& 5.15$\pm$3.75\\
- FN& 5.61$\pm$4.25& 4.68$\pm$3.67& 8.14$\pm$6.31& 4.72$\pm$3.84\\ \hline
$|\Supps|$& 12.38$\pm$7.04& 12.53$\pm$7.20& 13.58$\pm$8.44& 17.59$\pm$10.52\\
- TP& 11.59$\pm$6.76& 10.83$\pm$5.97& 12.35$\pm$7.72& 16.83$\pm$9.79\\
- TN& 13.37$\pm$6.98& 14.54$\pm$7.64& 15.06$\pm$8.65& 18.95$\pm$10.82\\
- FP& 11.18$\pm$7.16& 10.08$\pm$6.30& 11.97$\pm$8.35& 14.91$\pm$10.26\\
- FN& 12.56$\pm$7.92& 13.49$\pm$8.52& 13.89$\pm$9.65& 17.32$\pm$11.89\\ \hline
 \end{tabular}
\end{table}

\begin{table}[hbt!]
\caption{Statistics (Average $\pm$ SD) of \QBAF and \QBAFP for the Deceptive Hotel Review dataset. TP, TN, FP, and FN stand for true positives, true negatives, false positives, and false negatives, respectively. 
$\Args^{+,\delta}, \Args^{-,\delta}$ are sets of arguments supporting positive and negative classes, respectively. 
$\mathcal{R}$ is a set of (attack and support) relations, whereas $\mathcal{R}_{\backslash\delta}$ is a subset of $\mathcal{R}$ containing specifically those not involving the default argument $\delta$.
} \label{tab5:statdeceptive}
\centering
\begin{tabular}{l r r r r} 
 \hline
 Measurement & \TQBAF & \TQBAFP & \BQBAF & \BQBAFP\\
 \hline
 \# Examples & \multicolumn{4}{c}{320 (TP: 130, TN: 143, FP: 26, FN: 21)} \\ \hline
$|\Args|$& 19.44$\pm$6.82& 19.44$\pm$6.82& 19.44$\pm$6.82& 19.44$\pm$6.82\\
- TP& 19.54$\pm$6.54& 19.54$\pm$6.54& 19.54$\pm$6.54& 19.54$\pm$6.54\\
- TN& 20.24$\pm$6.86& 20.24$\pm$6.86& 20.24$\pm$6.86& 20.24$\pm$6.86\\
- FP& 17.77$\pm$6.48& 17.77$\pm$6.48& 17.77$\pm$6.48& 17.77$\pm$6.48\\
- FN& 15.48$\pm$7.34& 15.48$\pm$7.34& 15.48$\pm$7.34& 15.48$\pm$7.34\\ \hline
$|\Args^{+,\delta}|$& 9.14$\pm$4.46& 9.89$\pm$5.02& 9.14$\pm$4.46& 10.34$\pm$5.00\\
- TP& 12.06$\pm$4.27& 13.58$\pm$4.53& 12.06$\pm$4.27& 13.85$\pm$4.50\\
- TN& 6.85$\pm$3.11& 6.92$\pm$3.21& 6.85$\pm$3.11& 7.43$\pm$3.32\\
- FP& 9.50$\pm$3.56& 10.77$\pm$3.74& 9.50$\pm$3.56& 11.42$\pm$4.09\\
- FN& 6.19$\pm$3.60& 6.29$\pm$3.61& 6.19$\pm$3.60& 7.10$\pm$3.96\\ \hline
$|\Args^{-,\delta}|$& 10.30$\pm$4.91& 9.55$\pm$5.33& 10.30$\pm$4.91& 9.10$\pm$5.13\\
- TP& 7.48$\pm$3.16& 5.96$\pm$3.14& 7.48$\pm$3.16& 5.69$\pm$2.93\\
- TN& 13.39$\pm$4.77& 13.32$\pm$4.81& 13.39$\pm$4.77& 12.81$\pm$4.68\\
- FP& 8.27$\pm$3.34& 7.00$\pm$3.27& 8.27$\pm$3.34& 6.35$\pm$2.86\\
- FN& 9.29$\pm$4.23& 9.19$\pm$4.19& 9.29$\pm$4.23& 8.38$\pm$3.77\\ \hline
$|\mathcal{R}|$& 22.03$\pm$9.43& 22.03$\pm$9.43& 21.48$\pm$9.62& 21.48$\pm$9.62\\
- TP& 23.96$\pm$9.93& 23.96$\pm$9.93& 23.12$\pm$10.32& 23.12$\pm$10.32\\
- TN& 21.41$\pm$8.69& 21.41$\pm$8.69& 21.01$\pm$8.67& 21.01$\pm$8.67\\
- FP& 20.04$\pm$8.88& 20.04$\pm$8.88& 20.00$\pm$9.75& 20.00$\pm$9.75\\
- FN& 16.81$\pm$9.35& 16.81$\pm$9.35& 16.29$\pm$9.37& 16.29$\pm$9.37\\ \hline
$|\mathcal{R}_{\backslash\delta}|$& 8.19$\pm$6.35& 8.19$\pm$6.35& 8.19$\pm$6.35& 8.19$\pm$6.35\\
- TP& 10.69$\pm$7.09& 10.69$\pm$7.09& 10.69$\pm$7.09& 10.69$\pm$7.09\\
- TN& 6.47$\pm$5.13& 6.47$\pm$5.13& 6.47$\pm$5.13& 6.47$\pm$5.13\\
- FP& 7.35$\pm$5.64& 7.35$\pm$5.64& 7.35$\pm$5.64& 7.35$\pm$5.64\\
- FN& 5.52$\pm$4.59& 5.52$\pm$4.59& 5.52$\pm$4.59& 5.52$\pm$4.59\\ \hline
$|\Atts|$& 9.30$\pm$4.35& 7.00$\pm$3.43& 8.43$\pm$4.25& 6.13$\pm$3.24\\
- TP& 11.41$\pm$4.43& 6.45$\pm$3.34& 10.43$\pm$4.50& 5.98$\pm$3.51\\
- TN& 7.75$\pm$3.48& 7.55$\pm$3.34& 6.96$\pm$3.29& 6.32$\pm$2.96\\
- FP& 9.58$\pm$4.03& 7.12$\pm$3.84& 8.81$\pm$4.04& 6.38$\pm$3.65\\
- FN& 6.52$\pm$3.75& 6.52$\pm$3.75& 5.52$\pm$3.14& 5.48$\pm$2.91\\ \hline
$|\Supps|$& 12.73$\pm$6.18& 15.03$\pm$7.09& 13.05$\pm$6.45& 15.34$\pm$7.35\\
- TP& 12.55$\pm$6.27& 17.52$\pm$7.66& 12.69$\pm$6.54& 17.14$\pm$7.66\\
- TN& 13.66$\pm$6.03& 13.86$\pm$6.11& 14.06$\pm$6.25& 14.69$\pm$6.81\\
- FP& 10.46$\pm$5.58& 12.92$\pm$5.59& 11.19$\pm$6.15& 13.62$\pm$6.66\\
- FN& 10.29$\pm$6.29& 10.29$\pm$6.29& 10.76$\pm$6.74& 10.81$\pm$7.10\\ \hline
 \end{tabular}
\end{table}

\paragraph{Number of arguments}
The spam dataset had the minimum average number of arguments ($\sim$ 10 arguments per example as shown by $|\Args|$ in Table~\ref{tab5:statspam}). However, if we look at examples for which the prediction is positive (i.e., both true positive examples TP and false positive examples FP in Table~\ref{tab5:statspam}), we find 36 arguments per example on average. 
Looking at the underlying PLR model, we found that the default argument $\delta$ before post processing supported the negative class with $\tau(\delta) = 5.800$ (not shown in the Table), which was very high compared to the base scores of other arguments. 
This means that the classifier answered ``Not spam'' by default unless it could identify sufficient evidence to answer ``Spam''. 
Even true negative examples (TN) had around three arguments for the negative class on average, including $\delta$.
Interestingly,  false negative examples (FN) had a relatively higher number of arguments than true negatives, but still less than those of true positives (TP). This implies that the false negative examples usually had some, but insufficient, evidence for the positive class, compared to the true negatives which almost have nothing.

Unlike the SMS Spam Collection dataset, the base scores of $\delta$ for the Amazon Clothes and the Deceptive Hotel Reviews datasets were 0.2597 and 0.6932 supporting the negative class, respectively (not shown in the tables).
In order to push the prediction to either positive or negative, we needed evidence. Hence, the average number of arguments for these two datasets were similar for both classes (as shown by $|\Args|$ of TP, TN, FP, FN in Tables~\ref{tab5:statsentiment}-\ref{tab5:statdeceptive}). 
Examples predicted as positive, therefore, had higher number of arguments for the positive class ($|\Args^{+,\delta}|$) than those predicted as negative.
Similarly, examples predicted as negative had higher number of arguments for the negative class ($|\Args^{-,\delta}|$) than those predicted as positive.
Overall, the statistics of \QBAF and \QBAFP reported here can help us notice the global behavior of the underlying PLR model (e.g., conditions for it to predict positive or negative). 

\paragraph{Number of relations}
$|\mathcal{R}|$ had the similar trend as the number of arguments. Texts predicted as spams had a significantly higher number of attacks and supports than those predicted as non-spams (see $|\Atts|$ and $|\Supps|$ in Table~\ref{tab5:statspam}).
For the other two datasets, they usually had more supports than attacks, especially after post-processing, to provide sufficient evidence for the predictions. 
In any case, all three datasets had $|\mathcal{R_{\backslash\delta}}|$ from 8 to 12, on average, making the explanations extracted from the \QBAFPs (i.e., the \xlrs) different from the standard explanations for logistic regressions (i.e., the FLXs) due to many relations between features.

\paragraph{Other remarks}
First, the number of arguments $|\Args|$ for the \TQBAF, \TQBAFP, \BQBAF, and \BQBAFP for the same example are always equal. This is expected from Definitions~\ref{def:axplrqbaf} and \ref{def:postprocess}.
Second, $|\mathcal{R}|$ is different for \TQBAFs and \BQBAFs but their $|\mathcal{R_{\backslash\delta}}|$ are the same. This is because \TQBAFs and \BQBAFs have the same relations between two non-default arguments except that the directions are reversed. For the relations with the default argument, \TQBAFs connect the arguments of the most general patterns to the default whereas \BQBAFs connect the most specific patterns to the default. That is why $|\mathcal{R}|$ was different between \TQBAFs and \BQBAFs.   
Lastly, post-processing does not change the number of pairs in relations in the experiments as shown by $|\mathcal{R}|$ of \TQBAFs and \TQBAFPs and $|\mathcal{R}|$ of \BQBAFs and \BQBAFPs. In theory, it could possibly change as the pairs $(a, b)$ with $\sigma(a) = 0$ are removed. However, because all the base scores in \QBAFs are from the weights of the trained LR model, each of which has around 15 decimal points, it is hardly possible to find an argument $a$ with $\sigma(a) = 0$ in practice. So, none of the pairs is removed during post-processing.

\subsection{Sufficiency}
Next, given a \QBAFP, we were interested in the number of supporting arguments needed in order to sufficiently explain the prediction. Here, \textit{sufficiently} means that given the base score of $\delta$ and all the attacking arguments, the strengths given by these supporting arguments are enough to make the strength of $\delta$ greater than 0. In other words, for each test example, we wanted to find the smallest $k$ such that $S \subseteq \Suppsp(\delta)$, $|S| = k$ and
\begin{equation}\label{eq5:sufficiency}
    \tau'(\delta) + \sum_{b \in S}\frac{\sigma(b)}{\nu(b)} - \sum_{b \in \Attsp(\delta)}\frac{\sigma(b)}{\nu(b)} > 0
\end{equation}
Furthermore, we extended our question to other arguments in \QBAFP which had at least one attacker or supporter. (We call them \textit{intermediate arguments}.) We wondered how many supporting arguments were needed to make the strength of the argument greater than 0, taking into account the base score and all the strengths from the attackers.
Knowing the answers to these questions helps us decide how many arguments we should show to the users for explaining the final prediction or the intermediate arguments.

\begin{figure}[t!]
\subfigure[Default $\delta$ (All)]{\includegraphics[width = 2.4in]{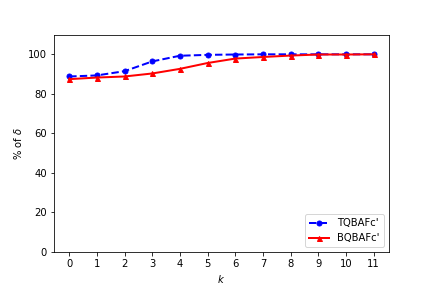}}
\subfigure[Intermediate $\alpha_i$ (All)]{\includegraphics[width = 2.4in]{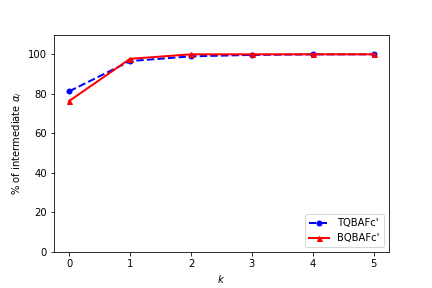}}\\
\subfigure[Default $\delta$ (Class not flipped)]{\includegraphics[width = 2.4in]{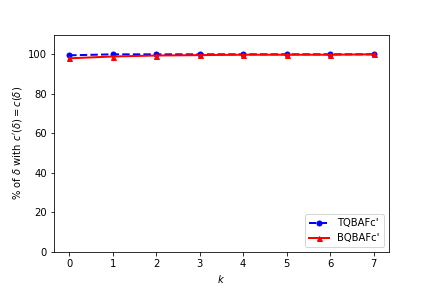}}
\subfigure[Intermediate $\alpha_i$ (Class not flipped)]{\includegraphics[width = 2.4in]{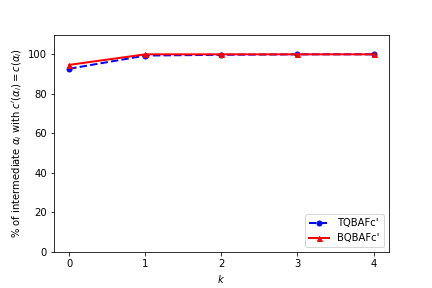}}\\
\subfigure[Default $\delta$ (Class flipped)]{\includegraphics[width = 2.4in]{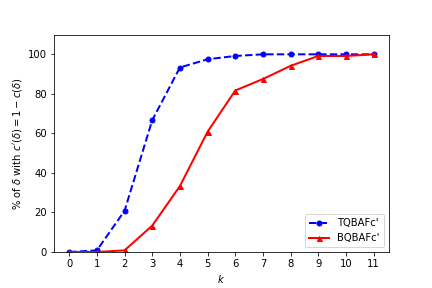}}
\subfigure[Intermediate $\alpha_i$ (Class flipped)]{\includegraphics[width = 2.4in]{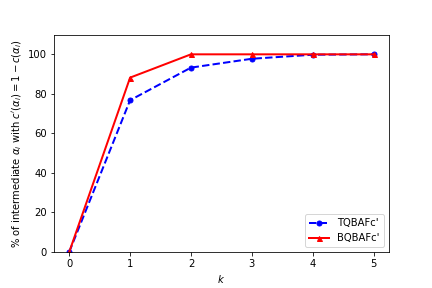}}
\caption{Plots showing the percentage of arguments (the default arguments $\delta$ or intermediate arguments $\alpha_i$) of which the strength can be greater than 0 using only $k$ supporting arguments. These arguments are extracted from test examples of the SMS Spam Collection dataset. \textit{Class flipped} means the supported class changes after post-processing.}
\label{fig5:suffspam}
\end{figure}

\begin{figure}[t!]
\subfigure[Default $\delta$ (All)]{\includegraphics[width = 2.4in]{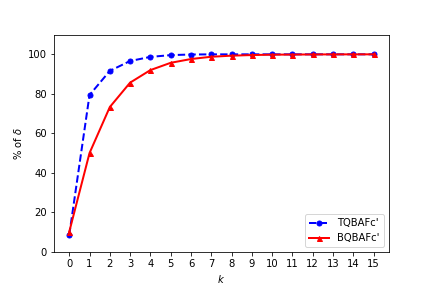}}
\subfigure[Intermediate $\alpha_i$ (All)]{\includegraphics[width = 2.4in]{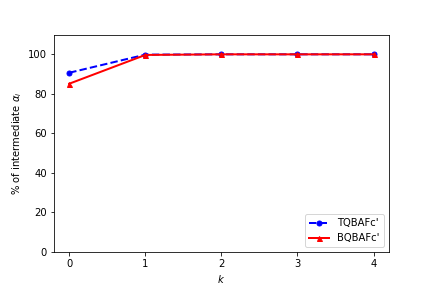}}\\
\subfigure[Default $\delta$ (Class not flipped)]{\includegraphics[width = 2.4in]{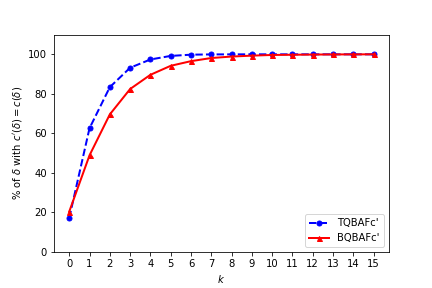}}
\subfigure[Intermediate $\alpha_i$ (Class not flipped)]{\includegraphics[width = 2.4in]{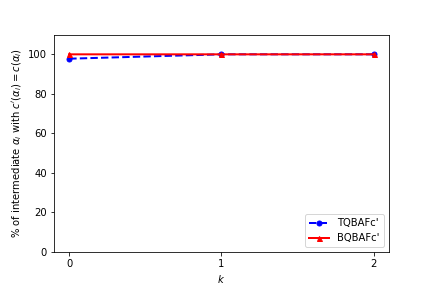}}\\
\subfigure[Default $\delta$ (Class flipped)]{\includegraphics[width = 2.4in]{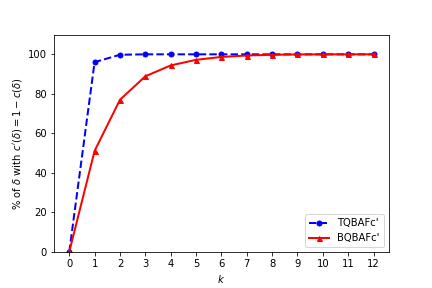}}
\subfigure[Intermediate $\alpha_i$ (Class flipped)]{\includegraphics[width = 2.4in]{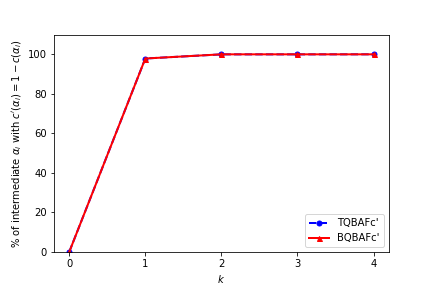}}
\caption{Plots showing the percentage of arguments (the default arguments $\delta$ or intermediate arguments $\alpha_i$) of which the strength can be greater than 0 using only $k$ supporting arguments. These arguments are extracted from test examples of the Amazon Clothes dataset. \textit{Class flipped} means the supported class changes after post-processing.}
\label{fig5:suffsentiment}
\end{figure}

\begin{figure}[t!]
\subfigure[Default $\delta$ (All)]{\includegraphics[width = 2.4in]{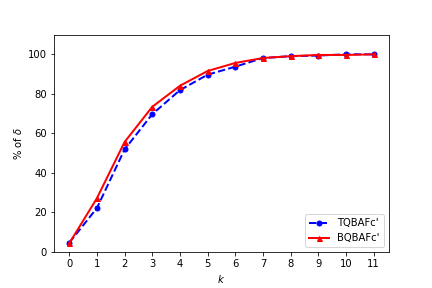}}
\subfigure[Intermediate $\alpha_i$ (All)]{\includegraphics[width = 2.4in]{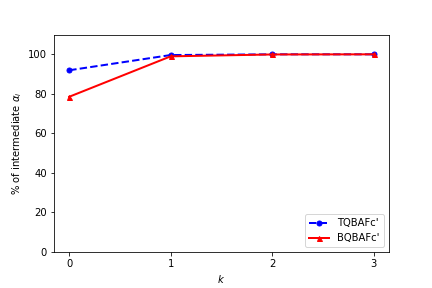}}\\
\subfigure[Default $\delta$ (Class not flipped)]{\includegraphics[width = 2.4in]{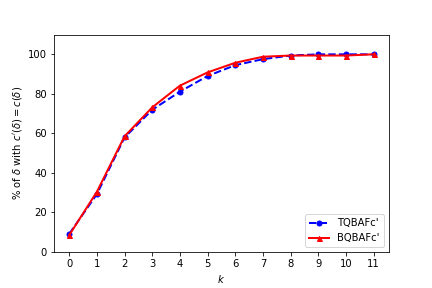}}
\subfigure[Intermediate $\alpha_i$ (Class not flipped)]{\includegraphics[width = 2.4in]{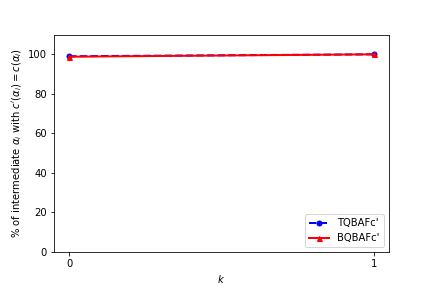}}\\
\subfigure[Default $\delta$ (Class flipped)]{\includegraphics[width = 2.4in]{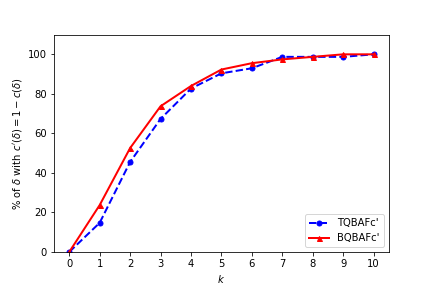}}
\subfigure[Intermediate $\alpha_i$ (Class flipped)]{\includegraphics[width = 2.4in]{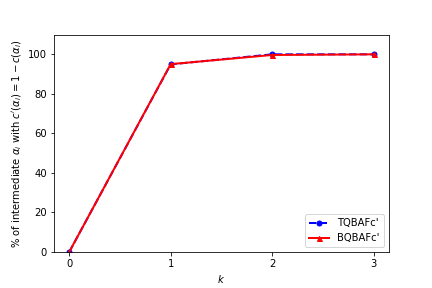}}
\caption{Plots showing the percentage of arguments (the default arguments $\delta$ or intermediate arguments $\alpha_i$) of which the strength can be greater than 0 using only $k$ supporting arguments. These arguments are extracted from test examples of the Deceptive Hotel Reviews dataset. \textit{Class flipped} means the supported class changes after post-processing.}
\label{fig5:suffdeceptive}
\end{figure}

Figures~\ref{fig5:suffspam}-\ref{fig5:suffdeceptive} show the results of this experiment. The x-axis of each plot is the number of supporting arguments used ($k$), whereas the y-axis shows the percentage of arguments (default or intermediate) of which the strength can be greater than 0 by using only $k$ supporting arguments. Considering plots on the left-hand side of the three figures, we can see that the numbers of supporting arguments needed for $\delta$ were different for each dataset.
The SMS Spam Collection dataset seemed to need the least. However, this was the case only for examples predicted as negative, i.e., the supported class was not flipped after post-processing, corresponding to Figure~\ref{fig5:suffspam}(c). The reason was that the base score of $\delta$ was relatively high. It could outnumber the strengths from the attackers even without the strengths from the supporters. 
Nevertheless, this was not true when the supported class flips from negative to positive as it required 4-6 supporting arguments to make 80\% of the test set have sufficient explanations.
Meanwhile, the Amazon Clothes and the Deceptive Hotel Reviews datasets required approximately 1-3 and 3-4 supporting arguments, respectively, for sufficient explanations of 80\% of the test set (regardless of the predicted class). 

Besides, for the SMS Spam Collection and the Amazon Clothes datasets, \BQBAFP required more supporting arguments than \TQBAFP. This was likely because \BQBAFP connected the arguments representing the most specific patterns to the default. For these two datasets, they outnumbered the most general patterns \TQBAFP connected to the default. 
So, the default argument of \BQBAFP had more supporters where the strengths were distributed. Therefore, more supporters were required to make the sufficient explanation.

Considering the plots on the right side of Figures~\ref{fig5:suffspam}-\ref{fig5:suffdeceptive}, only one supporting argument was usually sufficient to explain the supported class of an intermediate argument $\alpha_i$. 
Even without any supporters, only the base score was sufficient in most cases if the supported class is not flipped after post-processing.
Hence, if a user wants to see supporting information for an intermediate argument, when the space is limited, showing only 1-2 supporters are totally acceptable.

\section{Experiment 2: Plausibility} \label{sec:plausibility}

In this section, we aimed to evaluate the plausibility of \xlr, compared to FLX, to confirm our hypothesis that it is essential to consider relations between features (i.e., patterns) when we generate local explanations.
So, we compared the feature scores given by the explanation methods to scores reflecting how humans consider the features.
For instance, if a machine indicates pattern $p_1$ as a main reason for predicting the positive class and humans think that $p_1$ is truly a sign of the positive class, we can say that the machine explanation aligns well with human judgement (i.e., having high plausibility). 
In other words, the higher correlation between machine explanation scores and human scores implies the higher plausibility of the explanation method.
Hence, we chose Pearson's correlation as the metric in this experiment.

\subsection{Datasets} 
We used the SMS Spam Collection (spam filtering) and the Amazon Clothes (sentiment analysis) datasets since humans generally perform well on these two tasks, making the human scores reliable. 
For each dataset, we needed 500 test examples for evaluation. These examples must have the predicted probability of the output class greater than 0.9 to ensure that the bad quality of the explanation was not due to low model accuracy or text ambiguity.
Note that we did not conduct this experiment on the Deceptive Hotel Reviews dataset as lay humans are not adept at identifying deceptive reviews. The human accuracy was only around 55\% in \cite{lai2020chicago}, so we cannot trust human judgement on machine explanations in this task. We would work on the deceptive review detection task in the next experiment instead.

\subsection{Machine Explanations} 
As discussed in Section~\ref{subsec:explanations}, both FLX and \xlr use $(p_{j}, \pi(p_{j}, x), s_j)$ triplets as explanations where $s_j$ is the score of the pattern $p_j$ or the match $\pi(p_{j}, x)$ in the input $x$.
For FLX, $s_j$ equals $w_jf_j$, and any $p_j$ with the relatively large score $s_j$ can be chosen as a part of the explanation.
By contrast, shallow \xlr uses only arguments at the top level of the underlying \QBAF, i.e., arguments attacking or supporting $\delta$, as explanations. 
Meanwhile, deep \xlr can use any arguments in the \QBAF.
The $s_j$ of \xlr also depends on whether the \QBAF is \TQBAF or \BQBAF. So, we compared all of these variations in this experiment. 
Note that, because $\tau'$ and $\sigma$ of \xlr need to be interpreted with the supported class, we adjusted $s_j$ for \xlr to be self-contained. To put it simply, we multiplied $\tau'(\alpha_j)$ and $\sigma(\alpha_j)'$ of \xlr with 1 if $c'(\alpha_j) = 1$, or with -1 if $c'(\alpha_j) = -1$. This made the higher $s_j$ always imply the stronger evidence for the positive class (similar to FLX).

\subsection{Human Scores}
We recruited human participants via Amazon Mechanical Turk (MTurk)\footnote{https://www.mturk.com/} and asked them whether the pattern $p_j$ or the matched phrase $\pi(p_{j}, x)$ was the evidence for the positive or the negative class. 
Since the pattern $p_j$ only may be difficult to understand, we provided the translation to help lay users on MTurk, as shown in Figure~\ref{fig5:mturk_questions} (a).
Another way to present the pattern is to show samples of phrases (from the training set) matched by the pattern. We also collected human answers for this pattern representation showing five unique samples per pattern\footnote{If we have less than five unique matched phrases in the training set, we just show all of them.}, as displayed in Figure~\ref{fig5:mturk_questions} (b). 
Finally, a question for a single matched phrase $\pi(p_{j}, x)$ was a lot simpler as shown in Figure~\ref{fig5:mturk_questions} (c).
We provided five options for each question, ranging from definitely positive, positive, not sure, negative, and definitely negative. These correspond to the score 2, 1, 0, -1, and -2, respectively. For the SMS Spam Collection dataset, these options were instead definitely spam, spam, not sure, non-spam, and definitely non-spam. Each question was answered by five participants, and the scores were averaged before comparing with machine explanation scores.

Concerning the payment for answering questions, we paid the participants \$0.30 per 10 pattern questions, \$0.20 per 10 group-of-phrases questions, and \$0.20 per 20 matched phrase questions.

\begin{figure}[t!]
    \centering
    \subfigure[An example question for a pattern.]{\includegraphics[width=0.8\textwidth]{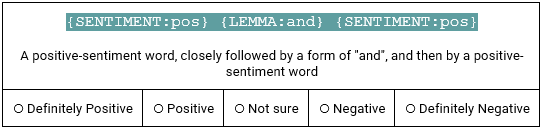}}
    \subfigure[An example question for a group of phrases sampled from the pattern.]{\includegraphics[width=0.8\textwidth]{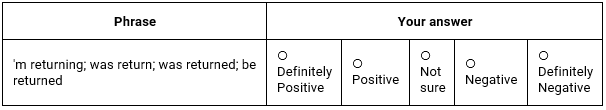}}
    \subfigure[An example question for a matched phrase.]{\includegraphics[width=0.8\textwidth]{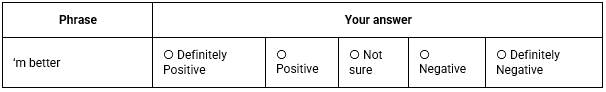}}
    \caption{Examples of questions (from the Amazon Clothes dataset) posted on Amazon Mechanical Turk to elicit human scores.}
    \label{fig5:mturk_questions}
\end{figure}

\subsection{Results}
Tables~\ref{tab5:plausibilityspam} and \ref{tab5:plausibilitysentiment} report the Pearson's correlations between the machine explanation scores and the human scores collected from Amazon Mechanical Turk for both datasets.
The last row of each table shows inter-rater agreement measures (Fleiss' kappa) \cite{fleiss1971measuring}.
We observe that the agreement measures for the SMS Spam Collection dataset were very close to zero (especially for the questions for patterns and samples), while the agreement rates for the Amazon Clothes dataset (sentiment analysis task) were significantly higher. 
This was likely because evidence from the sentiment analysis task (including patterns, samples, matched phrases) usually conveys clear meanings even without contexts, whereas evidence from the spam detection task often requires contexts for humans to make decisions.
For example, \textit{upset}, \textit{worthless}, and \textit{disappointed} were surely for negative reviews.
In contrast, \textit{mobile}, \textit{win}, and \textit{call} could appear both in spam and non-spam texts.
This caused higher disagreements in human answers though the model used these words certainly as evidence for the spam class.
As a result, the human scores for the spam task were less reliable than the scores for the sentiment analysis task.
Consequently, the overall correlations in Table~\ref{tab5:plausibilityspam} were also less than the scores in Table~\ref{tab5:plausibilitysentiment}. 

Hence, we focused on discussing the results in Table~\ref{tab5:plausibilitysentiment} with more reliable human scores. 
For each row, the correlation between the explanations and the human scores for patterns was lower than for samples and matched phrase.
Hence, we should show not only the patterns but also some matched samples of the patterns to generate better plausible explanations.
In addition, for \TQBAFP and \BQBAFP, the strengths of top-level arguments $\sigma(\alpha_j)'$ were better than the base scores $\tau'(\alpha_j)$ in terms of the alignment with human judgement. 
The correlations were also significantly higher than FLX.
This confirmed the advantage of the prominent feature of \xlr, i.e., considering interactions between patterns when generating local explanations.
However, by extending from arguments in the top level to all levels in the \QBAFP, only the correlations in \BQBAFP remained high, while the correlations in \TQBAFP dropped.
Therefore, deep \xlr, utilizing arguments of all levels in the graph, would go along better with \BQBAFP than \TQBAFP. 
It also implied that the base scores of the most specific patterns, which equaled their strengths in \TQBAFP, required some adjustments to align well with human judgement.

\begin{table}[t!]
\caption{Pearson's correlation between explanation scores and human  scores for the SMS Spam Collection dataset.} \label{tab5:plausibilityspam}
\setlength{\tabcolsep}{2.5pt}
\centering
\begin{tabular}{ L{0.3\linewidth} C{0.1\linewidth} C{0.1\linewidth} C{0.1\linewidth}} 
 \hline
 &\multicolumn{3}{c}{\textbf{Human scores}}\\
 \textbf{Explanation scores} & Pattern & Samples & Matched Phrase \\ \hline
 \textit{FLX} & 0.227 & 0.175 & -0.022 \\ \hline
 \textit{\TQBAFP} \\
 \quad $s_j = \tau'(\alpha_j)$ (top level) & \textbf{0.687} & \textbf{0.533} & -0.019\\ 
 \quad $s_j = \sigma(\alpha_j)'$ (top level) & 0.520 & 0.462 & \textbf{0.125}\\ 
 \quad $s_j = \sigma(\alpha_j)'$ (all levels) & 0.176 & 0.100 & 0.005\\ \hline 
 \textit{\BQBAFP} \\
 \quad $s_j = \tau'(\alpha_j)$ (top level) & 0.197 & -0.005 & -0.047\\ 
 \quad $s_j = \sigma(\alpha_j)'$ (top level) & 0.240 & 0.175 & 0.046\\ 
 \quad $s_j = \sigma(\alpha_j)'$ (all levels) & 0.271 & 0.308 & 0.053\\
 \hline
 Fleiss $\kappa$ & 0.001 & 0.068 & 0.118 \\
 \hline
 \end{tabular}
\end{table}

\begin{table}[t!]
\caption{Pearson's correlation between explanation scores and human  scores for the Amazon Clothes dataset.} \label{tab5:plausibilitysentiment}
\setlength{\tabcolsep}{2.5pt}
\centering
\begin{tabular}{ L{0.3\linewidth} C{0.1\linewidth} C{0.1\linewidth} C{0.1\linewidth}} 
 \hline
 &\multicolumn{3}{c}{\textbf{Human scores}}\\
 \textbf{Explanation scores} & Pattern & Samples & Matched Phrase \\ \hline
 \textit{FLX} & 0.503 & 0.525 & 0.529 \\ \hline
 \textit{\TQBAFP} \\
 \quad $s_j = \tau'(\alpha_j)$ (top level) & 0.423 & 0.491 & 0.487\\ 
 \quad $s_j = \sigma(\alpha_j)'$ (top level) & \textbf{0.632} & \textbf{0.693} & \textbf{0.688}\\ 
 \quad $s_j = \sigma(\alpha_j)'$ (all levels) & 0.490 & 0.503 & 0.501\\ \hline 
 \textit{\BQBAFP} \\
 \quad $s_j = \tau'(\alpha_j)$ (top level) & 0.442 & 0.466 & 0.486\\ 
 \quad $s_j = \sigma(\alpha_j)'$ (top level) & 0.599 & 0.621 & 0.634\\ 
 \quad $s_j = \sigma(\alpha_j)'$ (all levels) & 0.610 & 0.627 & 0.627\\
 \hline
 Fleiss $\kappa$ & 0.210 & 0.297 & 0.358 \\
 \hline
 \end{tabular}
\end{table}

\section{Experiment 3: Tutorial and Real-time Assistance} \label{sec:tutorial}
Among the three datasets, the deceptive review detection task is the most difficult tasks for humans. 
In this experiment, we follow the study \cite{lai2020chicago} to evaluate how effective \xlr can be used to teach and support humans to perform deceptive review detection.

\subsection{Setup}
We recruited participants via Amazon Mechanical Turk and redirected them to our a survey created using Qualtrics\footnote{\url{https://imperial.eu.qualtrics.com/}}.
The survey aimed to assess the capability of humans to detect deceptive hotel reviews before and after they learn from explanations. It consisted of five parts.
\begin{enumerate}
    \item Attention-check questions (4 questions) – The participant needed to answer all the questions in this part correctly to proceed.
    \item Pre-test (10 questions) -- For each question, the participant was asked whether a given hotel review was truthful or deceptive.
    \item Tutorial (10 questions) -- The format was the same as part 2, but then, we revealed to the participant the correct answer and the AI-generated prediction and explanation for them to learn from.
    \item Post-test (20 questions) -- For the first ten questions, the questions and the format were the same as part 2. We additionally showed what the participant had answered during the pre-test as a reference. The next ten questions were the same as the first ten except that we also provided AI explanations (without the predictions) for these questions, as \textit{real-time assistance} \cite{lai2020chicago}. The format of the explanations was the same as what s/he had seen during the tutorial phase. The corresponding previous answer (from the first ten questions) was also provided when the participant answered each of the last ten questions.
    \item Additional questions (5 questions) -- The participant was asked general questions before finishing the survey. These include, for example, how they detected deceptive and truthful reviews and any (free-text) feedback they might want to tell us.
\end{enumerate}
At the end of the survey, each participant was given a Reference ID as a proof that s/he had completed the task (i.e., the HIT) for claiming the reward from the MTurk system.
The improved performance of humans after being trained and assisted by the explanations showed how useful the explanations were. 
To motivate the participants to pay attention to the tasks, we divided the payment into two parts.
\begin{itemize}
    \item A guaranteed reward (\$2.00) was given after the participant completed the whole survey.
    \item A bonus reward -- The participant was given an additional bonus reward of \$0.10 for each question answered correctly (both in the pre-test and in the post-test). Therefore, the maximum bonus reward each participant could get was \$0.10 x 30 = \$3.00.
\end{itemize}

\subsection{Explanations}
We compared four explanation methods in this experiment including SVM, FLX, shallow \xlr, and deep \xlr. We selected linear SVM since there is a study \cite{lai2020chicago} showing that tutorials from simple models such as linear SVM worked better than tutorials from deep models such as BERT \cite{devlin-etal-2019-bert}. 
To train the SVM, we used TF-IDF vectorizer and employed exhaustive search to find the best hyperparameter $C\in \{1, 10, 100, 1000\}$.
As a result, the model achieved the accuracy and the macro F1 of 0.891. We generated the explanations for the SVM model by showing the most important 10 words according to the absolute value of SVM coefficients. We also highlighted these words in text with the color and the intensity reflecting the sign and the magnitude of the coefficient, respectively. An example of SVM explanations during the tutorial phase is shown in Figure~\ref{fig5:svm}. 

\begin{figure}[t]
    \centering
    \includegraphics[width=0.8\linewidth]{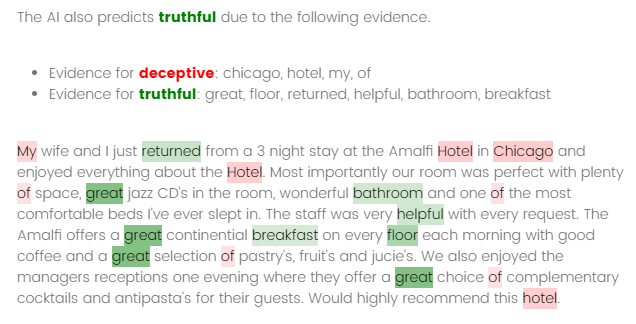}
    \caption{Example of SVM explanation during the tutorial phase}
    \label{fig5:svm}
\end{figure}

FLX, shallow \xlr, and deep \xlr were extracted from the same pattern-based LR model, of which the performance was shown in Table~\ref{tab5:accuracy}. Note that the LR model underperformed the SVM model, with the accuracy of 0.853 and 0.891, respectively. 
We decided to use \BQBAFP for both shallow and deep \xlr due to two reasons. First, the top-level arguments of \BQBAFP provided more contexts than those of \TQBAFP, and n-gram explanations are usually better than word-level explanations thanks to more contexts provided \cite{lertvittayakumjorn-toni-2019-human}.
Second, deep \xlr went along better with \BQBAFP than \TQBAFP as discussed in Section~\ref{sec:plausibility}.
Both FLX and shallow \xlr showed top 10 patterns/arguments and   
share the same presentation, as shown in Figure~\ref{fig5:shallowAXPLR}.
Deep \xlr also started from the top 10 arguments but allowed the users to expand them to see attacking and supporting arguments, as shown in Figure~\ref{fig5:deepAXPLR}. 
Moreover, we provided the input text with highlights similar to SVM explanations to help the users locate where the patterns appear in the input.
The intensity of the highlight represented the sum of the explanation scores of all patterns that the word matched. For \xlr, we summed the scores from only the top-level patterns as the scores from other levels had been aggregated into the top level.

\subsection{Question Selection}
For test questions, we randomly selected 50 questions from the test set of the Deceptive Hotel Reviews dataset. Then we partitioned them into five question sets. One participant was assigned one set of test questions and one explanation method (for tutorial and real-time assistance). Each pair of explanation method and question set was assigned to five people. Overall, we had 4 explanation methods $\times$ 5 question sets $\times$ 5 annotations = 100 surveys in total. So, we recruited exactly 100 participants on MTurk without allowing a participant to do the survey twice. 

For the ten tutorial questions for each explanation method, we selected them from the development set of the Deceptive Hotel Reviews dataset using submodular pick \cite{ribeiro2016should} to ensure that the selected examples covered important features of the task. 
Although submodular pick is a greedy algorithm, it provides a constant-factor approximation guarantee of $1-e^{-1}$ to the optimum \cite{krause2014submodular}.
This made the tutorial questions different for each explanation method except that shallow \xlr and deep \xlr share the same set of tutorial questions.

\subsection{Results}
The average scores of human participants are displayed in Table~\ref{tab5:tutorial}.
Using the pre-test scores as a baseline, we observe that the tutorial phase only did not help the participants perform better as the post-test scores without real-time assistance were not significantly greater than the baseline.
However, the real-time assistance after the tutorial indeed helped. 
By the approximate randomization test with 1,000 iterations and a significance level of 0.05 \cite{noreen1989computer,graham-etal-2014-randomized}, the post-test scores with real-time assistance from the explanations were significantly higher than the pre-test scores and the post-test scores with no assistance of the same explanation methods.
Nevertheless, we see no significant difference across explanation groups, so we can conclude only that FLX, shallow \xlr, and deep \xlr are competitive with SVM for providing explanations to teach and support humans to detect deceptive reviews.

The last column in Table~\ref{tab5:tutorial} shows the average performance of the underlying AI model on the same set of questions. SVM achieved 9 out of 10 in three question sets and 10 out of 10 in the other two, whereas the LR model (underlying FLX and \xlr) got 7, 7, 8, 9, and 10 for the five question sets (regardless of the order).
The total numbers of people that 
scored better than or equal to the AI
during the pre-test, post-test with no assistance, and post-test with assistance are 6, 8, and 26 out of 100 people, respectively.
This again shows the effectiveness of real-time assistance after the participants learned from the tutorial.

\begin{table}[t!]
\caption{Scores of the human participants (Average $\pm$ SD) in the tutorial and real-time assistance experiment using the Deceptive Hotel Reviews dataset. 
The last column shows the average score of the model that provides real-time assistance. The maximum score is 10.} \label{tab5:tutorial}
\setlength{\tabcolsep}{2.5pt}
\centering
\begin{tabular}{ l c c c c } 
 \hline
 \multirow{2}{*}{Explanation} & \multirow{2}{*}{Pre-test score} & \multicolumn{2}{c}{Post-test score}& \multirow{2}{*}{Model}\\
 &&No assistance & + assistance& \\
 \hline
 SVM & 5.68$\pm$1.60 & 5.12$\pm$1.24 & 6.56$\pm$1.87 & 9.40$\pm$0.50\\
 FLX & 5.64$\pm$1.38 & 5.68$\pm$1.44 & 6.56$\pm$1.73 & 8.20$\pm$1.19\\
 Shallow AXPLR& 5.40$\pm$1.53 & 5.60$\pm$1.35 & 6.56$\pm$1.89 & 8.20$\pm$1.19\\
 Deep AXPLR& 5.24$\pm$1.36 & 5.44$\pm$1.61 & 6.76$\pm$1.79 & 8.20$\pm$1.19\\
 \hline
 \end{tabular}
\end{table}

Finally, we asked the participants in the final part of the survey how they detected deceptive and truthful reviews.
We manually selected interesting answers from the participants who got 8 correct answers or more during the post-test with AI assistance. The answers are shown in Tables~\ref{tab5:deceptivecomments} and \ref{tab5:truthfulcomments}.
As expected, participants learning from SVM explanations rarely mentioned patterns but individual words. Some used the majority of highlighting colors as a heuristic (which was surprisingly effective, probably due to the good performance of SVM). 
Since FLX was extracted from the LR model with \grasp patterns, we noticed some patterns and generalizations noted by participants who learned from FLX such as ``when my was closely followed by 1  and hotel was followed by different words'' and ``the language used, and symbols and punctuation''.
Similarly, we also saw patterns noted by participants who learned from both types of \xlr such as ``It uses pronouns closely together'' and ``The patterns of specific words close together stood out, like luxury hotel.'', as well as implicit patterns such as ``There's also much less usage of city and hotel names''. 
On the other hand, they could also cover word-level cues, as we can see from the comments like ``I would also assume it was deceptive if the reviewer said ``I'' a lot.''.
However, there was no participant in the deep \xlr group mentioning the usefulness of sub-patterns (which could be expanded or collapsed). Also, the average scores of both types of \xlr were not significantly different. It could imply that shallow \xlr is already sufficient for tutorial and real-time assistance, without the need to go deep.
Last but not least, we found two interesting comments from the deep \xlr group. One contrasted deceptive and truthful reviews -- ``If the review said ``location'' as apposed to naming the city, I was more likely to assume it was true, or if it mentioned the elevators or doormen. If it said ``we'' instead of ``I'' I was usually more inclined say it was truthful.''. The other theorized the reason behind prominent patterns -- ``human phrasing that doesn't have hallmarks of being algorithmically generated or designed with the obvious intent to be picked up by a search engine (repeatedly mentioning the word Chicago was one example of this used).''. 

\begin{table}[t!]
\caption{Some answers from the participants on how they knew that a review was \textbf{deceptive}. These answers were manually picked from the participants who got 8 correct answers or more during the post-test with AI assistance. We also show the explanation method they were assigned and the final scores they got.} \label{tab5:deceptivecomments}
\setlength{\tabcolsep}{2.5pt}
\centering
\begin{tabular}{ B{0.1\linewidth} N{0.10\linewidth} B{0.75\linewidth}} 
 \hline
 Explanation & Score & Answer \\
 \hline
 \multirow{3}{*}{SVM} & 10 & If it seems too biased and sounds exaggerated. \\
 & 9 & Certain key words are used repeatedly and unnaturally. \\
 & 8 & If it has more red than green. \\ \hline
 \multirow{3}{*}{FLX} & 9 & Extreme and/or superlative language. \\
 & 8 & when my was closely followed by 1  and hotel was followed by different words \\
 & 8 & because of the words used, and naming the location, etc \\ \hline
 \multirow{3}{*}{\parbox{2cm}{Shallow\\ \xlr}} & 10 & A city was not capitalized or the overuse and closeness of ``my'' and ``I''.\\ 
 & 9 & There were methods to look at the text or the type, as well as sentiment and identify some un natural responses. The patterns of specific words close together stood out, like luxury hotel. \\ 
 & 8 & It uses pronouns closely together, uses proper names for hotels and cities oddly, and so on.\\ \hline
 \multirow{3}{*}{\parbox{2cm}{Deep\\ \xlr}} & 9 & The review mentioned the city by name a few times and was accompanied by odd sounding and separated facts.\\
 & 8 & If the review kept mentioning the name of the city or referring to things as being luxurious or smelly, then I would generally assume that the review was deceptive. I would also assume it was deceptive if the reviewer said ``I'' a lot.\\
 & 8 & It uses certain turns of phrase that are highly improbable or likely to come from a genuine human. Syntax issues can also be indicative of a deceptive review. \\
 \hline
 \end{tabular}
\end{table}

\begin{table}[t!]
\caption{Some answers from the participants on how they knew that a review was \textbf{truthful}. These answers were manually picked from the participants who got 8 correct answers or more during the post-test with AI assistance. We also show the explanation method they were assigned and the final scores they got.} \label{tab5:truthfulcomments}
\setlength{\tabcolsep}{2.5pt}
\centering
\begin{tabular}{ B{0.1\linewidth} N{0.10\linewidth} B{0.75\linewidth}} 
 \hline
 Explanation & Score & Answer \\
 \hline
 \multirow{3}{*}{SVM} & 10 & It sounds truthful and may sometimes talk about both the good and bad of the experience. \\
 & 9 & Words are not frequently repeated and they are used in a natural manner.\\
 & 8 & If it has more green than red \\ \hline
 \multirow{3}{*}{FLX} & 9 & Down to earth. Pros and cons are expressed in a balanced, not hyperbolic way. \\
 & 8 & when the text wasnt too long and sounded realistic \\
 & 8 & the language used, and symbols and punctuation \\ \hline
 \multirow{3}{*}{\parbox{2cm}{Shallow\\ \xlr}} & 10 & The use of brackets or parentheses.\\ 
 & 9 & The way the sentence was structured was far different then the other ones. The deceptive ones tried to appear truthful but the other ones just came off as natural.\\ 
 & 8 & It describes the layout and number of things in a more detailed fashion. There's less of a focus on repetitious usage of pronouns. There's also much less usage of city and hotel names. \\ \hline
 \multirow{3}{*}{\parbox{2cm}{Deep\\ \xlr}} & 9 & The review spoke on a personal level and did mention city names many times.\\
 & 8 & If the review said ``location'' as apposed to naming the city, I was more likely to assume it was true, or if it mentioned the elevators or doormen. If it said ``we'' instead of ``I'' I was usually more inclined say it was truthful. I also just payed attention to the overall vibe of the review.\\
 & 8 & Review features ordinary, human phrasing that doesn't have hallmarks of being algorithmically generated or designed with the obvious intent to be picked up by a search engine (repeatedly mentioning the word Chicago was one example of this used). \\
 \hline
 \end{tabular}
\end{table}

\subsection{Discussion}

We may conclude from the results of Experiment 3 that \xlr is competitive with SVM and FLX in terms of assisting humans in detecting deceptive reviews.
Also, according to the qualitative analysis, \xlr helps humans capture non-obvious patterns which are helpful to perform the task to some degree.
Still, there is a gap between human performance and model performance as we can notice in the last two columns of Table~\ref{tab5:tutorial}.
To narrow down this gap further, there are some interesting directions that could be explored.
First, how could we make the tutorial part more effective? 
We hypothesize that submodular pick might not be the best method to select tutorial questions.
In fact, \cite{lai2020chicago} has tried the \textit{spaced repetition} strategy where humans are presented with important features repeatedly (with some space in-between). However, it cannot be concluded from their experiment that spaced repetition is significantly better than submodular pick when it comes to selecting tutorial examples.
It would be interesting to study whether there is a better method to select and arrange tutorial questions for supporting human learning.

Additionally, in our experiment, \xlr transformed \QBAFP into a local input-based explanation, identifying important parts in the input together with the associated patterns.
However, there are other forms of explanations which could be extracted from \QBAFP and might be more suitable for this task. 
One is counterfactual explanation, showing which arguments should be added or removed from the current \QBAFP in order to change the model prediction.
This may help humans better learn relative importance of the patterns. 
It is likely possible to extract counterfactual explanation from our QBAF', in line with a recent work by 
\cite{albini2021interpreting} extracting counterfactual explanations from argumentation frameworks for PageRank \cite{page1999pagerank}.
Besides, if needed, we could generate synthetic example(s) and/or \QBAFPs to teach humans cases which are interesting but do not exist in the training data. 
For example, an input $x_1$ has a group of patterns strongly supporting the positive class, while an input $x_2$ has another group of patterns strongly supporting the negative class. What would happen if the two groups of patterns appeared in the same input?
The answer to this question could aid humans in prioritizing knowledge learned from individual real examples.
Combining groups of patterns is easier to do with \xlr, but not FLX since FLX does not group related patterns together.
Thus, overall, 
although \xlr did not outperform existing methods significantly in this experiment, the experiment is a first step towards 
several possible extensions of \xlr that may be worth exploring to better support human learning of new tasks.

\section{General Considerations on \xlr} \label{sec:discussions}
According to Experiment 2, \xlr renders highly plausible explanations compared to FLX, the traditional explanation method of LR.
One possible reason for \xlr not shining in Experiment 3 is that plausibility is not necessary for the tutorial and real-time assistance task.
The task only requires humans to learn and apply useful patterns though they may not know the reasons why such patterns are for the genuine class or the deceptive class. 
On the other hand,  \xlr would be more suitable for the task where plausibility is needed.
For example, if we use the classifier as a decision support tool, we want the explanation to provide insights about the input text that align well with human reasoning.
Even though the prediction is correct, if the explanation does not make sense, it is possible that the humans distrust the model and make a wrong final decision, leading to undesirable consequences.

Another context where \xlr could be useful is explanation-based human debugging of the model \cite{lertvittayakumjorn-toni-2021-explanation}.
The individual model weight $w_i$ for the pattern feature $p_i$ may not make sense to humans when $p_i$ is in fact related to other pattern features (as we can see in Experiment 2, where $\tau(\alpha_i)$ does not quite correlate with human reasoning).
This may cause misunderstanding in the humans and lead to their feedback being harmful to the overall model performance.
The \QBAFs of \xlr would provide a more accurate view of how the pattern features have been used by the model.
So, we believe that it is more likely leading to a successful model debugging than FLX.
Moreover, with the argumentative structure of \xlr, it would be interesting to see whether and how \xlr could let humans argue with the model, contributing to a richer way of human-AI collaboration for reversing an undesirable output or improving the model.

\section{Related Work}
\label{sec:relatedwork}

\paragraph{Local Explanations for Text Classification}
Text classification is a fundamental task in natural language processing, so there exist many explanation methods which are applicable to this task.
Focusing on local explanation methods (aiming to explain specific predictions), we can see several forms of explanations in literature such as extracted rationales \cite{lei-etal-2016-rationalizing,jain-etal-2020-learning}, attribution scores \cite{shrikumar2017gradient,arras-etal-2016-explaining}, rules \cite{stumpf2009interacting,ribeiro2018anchors}, influential training examples \cite{koh2017understanding,khanna2019interpreting}, and counterfactual examples \cite{yang-etal-2020-generating,ross-etal-2021-explaining}. 
Since \xlr forms an explanation using triplets of a pattern $p_{i'}$, a matched phrase $\pi(p_{i'}, x)$, and a score $s_{i'}$, it could be considered a mixture of rationales, attribution scores, and rules. 
This is another novelty aspect of \xlr as we rarely find XAI work that combines multiple forms of explanations together.
However, what makes this possible are the transparency of logistic regression (LR) and the interpretability of \grasp patterns.
So, we classify \xlr as a \textit{model-specific} explanation method, unlike LIME \cite{ribeiro2016should} and SHAP \cite{lundberg2017unified} which are \textit{model-agnostic} methods being applicable to any model architectures.
Nonetheless, the issue of dependency between features is also found in other architectures besides LR, such as convolutional neural network (learned) features in \cite{lertvittayakumjorn-etal-2020-find}. 
Therefore, it would be interesting to study how to extend \xlr to other architectures.

\paragraph{Computational Argumentation for Explainable AI}
As discussed in the introduction, computational argumentation has been used to support some XAI methods and construct argumentative explanations in the literature.
According to \cite{vcyras2021argumentative}, existing works in this area can be divided into two groups. 
The first group (i.e., \textit{intrinsic}) draws explanations from models that are natively using argumentative techniques such as AA-CBR \cite{vcyras2019explanations} and DeLP \cite{rodriguez2017educational}.
The second group (i.e., \textit{post-hoc}) extracts argumentative explanations from non-argumentative models such as neural networks \cite{dejl2021argflow} and Bayesian networks \cite{timmer2017two}.
Following \cite{vcyras2021argumentative}, some post-hoc approaches create a \textit{complete} mapping between the target model and the argumentation framework from which explanations are derived such as \cite{schulz2016justifying,albini2021interpreting}, while other post-hoc methods create an \textit{incomplete} mapping between the model and the argumentation framework (so called \textit{approximate} approaches) such as \cite{dejl2021argflow,timmer2017two}.
However, \xlr is a post-hoc approach (due to the non-argumentative PLR model) that does not fit nicely into this complete-approximate dichotomy.
On one hand, \xlr constructs a complete mapping between the PLR model and the \QBAF since every activated feature in the model (as well as the bias term) has a corresponding argument in the \QBAF.
On the other hand, the logistic regression semantics $\sigma$ of \xlr approximates the dialectical strength of every argument given that this strength does not actually exist in the PLR model.
The approximation of $\sigma$ is under an assumption that the strength of an argument is distributed equally and accordingly to every argument it attacks or supports, as represented by the fragments in Equation~\ref{eq5:logregstr}.
So, we could say that \xlr is a complete but intentionally approximate post-hoc approach so as to yield plausible explanations.

\paragraph{Argument Mining} 
Our work stays at the intersection of explainable AI, natural language processing, and computational argumentation.
Another research area that is similar to ours is \textit{argument mining}, which involves natural language processing and computational argumentation.
Argument mining is the process of automatically detecting and modeling the structure of inference and reasoning given in natural language texts \cite{lawrence2020argument}.
However, our work is not considered an argument mining work because the arguments in our \QBAF are arguing about the predicted output of a text classifier (PLR), whereas arguments in general argument mining works are arguing about a specific claim or conclusion in text.
Therefore, input texts in argument mining works must possess the argumentative spirit inside, while input texts for \xlr do not need to be argumentative but the classifier instead turns parts of them to be arguments for making classifications.

\section{Conclusion} \label{sec:conclusion}
To generate local explanations for pattern-based logistic regression models, we proposed \xlr, an explanation method enabled by quantitative bipolar argumentation frameworks we defined (\TQBAF and \BQBAF), capturing interactions among the patterns.
We proved that the extracted and post-processed frameworks underpinning \xlr are faithful to the LR model and satisfy many desirable properties.
After that, we proposed two presentations of \xlr, shallow and deep, specifying whether we present only the top-level arguments or all the arguments in the explanations.
We also conducted a number of experiments with \xlr, amounting to empirical as well as  human studies. The former discussed the statistics of the underlying argumentation frameworks for all input texts in the test sets
and analyzed sufficiency of the explanations in terms of the number of supporting arguments needed.
The latter  assessed whether \xlr is more plausible and helpful for human learning than traditional explanation methods for pattern-based LR models.
The results show that taking into account relations between arguments as \xlr does indeed helps the explanations align better with human judgement, particularly in the sentiment analysis task. 
Though \xlr performed competitively with traditional explanation methods in tutoring and supporting humans to detect deceptive hotel reviews, there were many participants learning from \xlr that could recall well-generalized patterns and important but implicit patterns deemed useful for the task.

\begin{acks}
We would like to thank Alessandra Russo and Simone Stumpf for their helpful comments.
Piyawat Lertvittayakumjorn wishes to thank the support from Anandamahidol Foundation, Thailand.
Francesca Toni was funded in part by the European Research Council (ERC) under the European Union’s Horizon 2020 research and innovation programme (grant agreement No. 101020934) and in part by J.P. Morgan and by the Royal Academy of Engineering under the Research Chairs and Senior Research Fellowships scheme.  Any views or opinions expressed herein are solely those of the authors listed, and may differ from the views and opinions expressed by J.P. Morgan or its affiliates. This material is not a product of the Research Department of J.P. Morgan Securities LLC. This material should not be construed as an individual recommendation for any particular client and is not intended as a recommendation of particular securities, financial instruments or strategies for a particular client.  This material does not constitute a solicitation or offer in any jurisdiction.
\end{acks}

\begin{appendix}
\section{Proofs}

\subsection{Proof of Lemma~\ref{lem:succ}} \label{proof:lem:succ}
\noindent\textbf{Lemma~\ref{lem:succ}.} 
The relation $\succ$ is not reflexive and not symmetric, but it is transitive. 
\begin{proof}
Let us consider each of the properties.
\begin{itemize}
    \item \textit{Not reflexive}: Proof by contradiction. Assume that $\succ$ is reflexive. Thus, $p \succ p$. According to Definition~\ref{def:specific}, it implies that $p \succeq p$ and $p \nsucceq p$, resulting in contradiction. Hence, $\succ$ is not reflexive. 
    \item \textit{Not symmetric}: Assume $p_1 \succ p_2$. By Definition~\ref{def:specific}, we obtain that $p_1 \succeq p_2$ and $p_2 \nsucceq p_1$. So, $p_2 \nsucc p_1$, implying that $\succ$ is not symmetric.
    \item \textit{Transitive}: Before proving that $\succ$ is transitive, we will prove that $\succeq$ is transitive first. Assume that $p_1 \succeq p_2$ and $p_2 \succeq p_3$. Because $p_1 \succeq p_2$, by Definition~\ref{def:specific}, every text $t$ matched by $p_1$ is also matched by $p_2$. Similarly, because $p_2 \succeq p_3$, every text $t$ matched by $p_2$, including those matched by $p_1$, is also matched by $p_3$. Therefore, $p_1 \succeq p_3$, i.e., $\succeq$ is transitive.
    
    Next, we will prove that $\succ$ is transitive using proof by contradiction. Assume that $p_1 \succ p_2$, $p_2 \succ p_3$, but $p_1 \nsucc p_3$. Because $p_1 \succ p_2$, by Definition~\ref{def:specific}, we get that $p_1 \succeq p_2$ and $p_2 \nsucceq p_1$. Similarly, because $p_2 \succ p_3$, we get that $p_2 \succeq p_3$ and $p_3 \nsucceq p_2$. Due to the transitivity of $\succeq$, we know that $p_1 \succeq p_3$. Now, let us consider $p_1 \nsucc p_3$. It is true iff $p_1 \nsucceq p_3$ or $p_3 \succeq p_1$. Still, we know that $p_1 \nsucceq p_3$ cannot be true, so it must be the case that $p_3 \succeq p_1$. 
    Due to the transitivity of $\succeq$, $p_3 \succeq p_1$ and $p_1 \succeq p_2$ imply that $p_3 \succeq p_2$. However, this contradicts with the result of $p_2 \succ p_3$. With this contradiction, $p_1 \nsucc p_3$ cannot be true. In other words, $p_1 \succ p_3$, implying that $\succ$ is transitive.
\end{itemize}
Hence, $\succ$ is not reflexive, not symmetric, but transitive.
\end{proof}

\subsection{Proof of Theorem~\ref{theo:dag}} \label{proof:theo:dag}
\noindent\textbf{Theorem~\ref{theo:dag}.} 
The graph structure of any \QBAF is a directed acyclic graph (DAG).
\begin{proof}
Let us consider the \TQBAF first. From the graph theory perspective, our graph $\langle \Args, \Atts_T, \Supps_T \rangle$ is equivalent to $G = \langle V, E \rangle$ where $V = \Args$ is a set of vertices and $E = \Atts_T \cup \Supps_T$ is a set of edges. According to Definition~\ref{def:axplrqbaf}, we can write $E$ explicitly as
\begin{align*}
    E =& \Atts_T \cup \Supps_T \\
    E =& \{(\alpha_i, \delta)| c(\alpha_i) \neq c(\delta) \wedge \nexists j[\alpha_j \in \Args \wedge p_i \succ p_j]\} \cup \\
        & \{(\alpha_i, \alpha_j)| c(\alpha_i) \neq c(\alpha_j) \wedge p_i \succ p_j \wedge \nexists k[\alpha_k \in \Args \wedge p_i \succ p_k \succ p_j]\} \cup \\
        & \{(\alpha_i, \delta)| c(\alpha_i) = c(\delta) \wedge \nexists j[\alpha_j \in \Args \wedge p_i \succ p_j]\} \cup \\
        & \{(\alpha_i, \alpha_j)| c(\alpha_i) = c(\alpha_j) \wedge p_i \succ p_j \wedge \nexists k[\alpha_k \in \Args \wedge p_i \succ p_k \succ p_j]\} \\
    E =& \{(\alpha_i, \delta)| \nexists j[\alpha_j \in \Args \wedge p_i \succ p_j]\} \cup \\
        & \{(\alpha_i, \alpha_j)| p_i \succ p_j \wedge \nexists k[\alpha_k \in \Args \wedge p_i \succ p_k \succ p_j]\}
\end{align*}
We will prove the result by contradiction. Assume that the graph $G$ is not a DAG. Hence, there must be a non-trivial path which forms a cycle in $G$. Assume that the path is $\alpha_{i_0}, \alpha_{i_1}, \ldots, \alpha_{i_k}, \alpha_{i_0}$ with $k\geq 1$. In this path, there must not be $\delta$ since the out-degree of $\delta$ is 0 (from Lemma~\ref{lem:deltaroot}). So, every edge in this path must be in the second set of the union above. Hence, we obtain that $p_{i_0} \succ p_{i_1}, p_{i_1} \succ p_{i_2}, \ldots, p_{i_k} \succ p_{i_0}$. Because $\succ$ is transitive (from Lemma~\ref{lem:succ}), $p_{i_0} \succ p_{i_0}$, but this is impossible since $\succ$ is not reflexive (also from Lemma~\ref{lem:succ}). Here is the contradiction. Thus, the graph structure of the \TQBAF is a DAG.

The proof for \BQBAF is similar to the proof for \TQBAF. We will obtain that 
\begin{align*}
    E =& \Atts_B \cup \Supps_B \\
    E =& \{(\alpha_i, \delta)| \nexists j[\alpha_j \in \Args \wedge p_j \succ p_i]\} \cup \\
        & \{(\alpha_j, \alpha_i)| p_i \succ p_j \wedge \nexists k[\alpha_k \in \Args \wedge p_i \succ p_k \succ p_j]\}
\end{align*}
Assume the directed cyclic path in $G$ is $\alpha_{i_0}, \alpha_{i_1}, \ldots, \alpha_{i_k}, \alpha_{i_0}$ with $k\geq 1$. Hence, $p_{i_0} \succ p_{i_k}$ (from the last edge in the path), $p_{i_k} \succ p_{i_{k-1}}$ (from the second last edge), $\ldots, p_{i_1} \succ p_{i_0}$ (from the first edge). Because $\succ$ is transitive, $p_{i_0} \succ p_{i_0}$, but this is impossible since $\succ$ is not reflexive. Here is the contradiction. Thus, the graph structure of any \BQBAF is also a DAG.
\end{proof}

\subsection{Proof of Theorem~\ref{theo:predict}} \label{proof:theo:predict}
\noindent\textbf{Theorem~\ref{theo:predict}.}
For a given \QBAF, the prediction of the underlying LR model can be inferred from the strength of the default argument:
\begin{enumerate}
    \item The predicted probability for the class $c(\delta)$ equals $sigmoid(\sigma(\delta))$.
    \item Hence, if $\sigma(\delta) > 0$, the LR model predicts class $c(\delta)$. Otherwise, it predicts the opposite class (i.e., $1-c(\delta)$).
\end{enumerate}
\begin{proof}
First, we will prove that the predicted probability for the class $c(\delta)$ equals $sigmoid(\sigma(\delta))$. 
In other words, we need to prove that, for $c(\delta) = 1$, $sigmoid(\sigma(\delta)) = sigmoid(\sum\limits_{i=1}^dw_if_i + b)$, i.e., $\sigma(\delta) = \sum\limits_{i=1}^dw_if_i + b$. 
Also, we need to prove that, for $c(\delta) = 0$, $sigmoid(\sigma(\delta)) = 1 - sigmoid(\sum\limits_{i=1}^dw_if_i + b) = sigmoid(-\sum\limits_{i=1}^dw_if_i - b)$, i.e., $\sigma(\delta) = -\sum\limits_{i=1}^dw_if_i - b$.\footnote{$1-sigmoid(x)=sigmoid(-x)$.}   

\textit{Case 1: $c(\delta) = 1$ -- The  class supported by  $\delta$ is Positive.}

Each argument $\alpha_i \in \Args - \{\delta\}$  supports class $c(\alpha_i) \in \{0, 1\}$. We partition $\Args - \{\delta\}$ into two sets -- one with Positive (1) as the  supported class and the other with Negative (0) as the supported class. We use $\Args^+$ and $\Args^-$ to represent the two sets, respectively.

Applying Definition~\ref{def:logregstr} to $\delta$, we obtain
\begin{align*}
    \sigma(\delta) &= \tau(\delta) + 
    \sum_{g \in \Supps(\delta)}\frac{\sigma(g)}{\nu(g)} - 
    \sum_{g \in \Atts(\delta)}\frac{\sigma(g)}{\nu(g)}
\end{align*}
Because $c(\delta) = 1$, we know from Definition~\ref{def:axplrqbaf} that the bias term of the underlying LR model is $b \geq 0$. Hence, $\tau(\delta) = |b| = b$.
Furthermore, $\Supps(\delta) \subseteq \Args^+$ and $\Atts(\delta) \subseteq \Args^-$. So, we obtain that 
\begin{align} \label{eq5:b}
    b &= \sigma(\delta) - 
    \sum_{g \in \Supps(\delta)}\frac{\sigma(g)}{\nu(g)} +
    \sum_{g \in \Atts(\delta)}\frac{\sigma(g)}{\nu(g)}
\end{align}
Next, for $\alpha_i$ in $\Args^+$, we know that $c(\alpha_i) = 1$, so the corresponding weight in the LR model is $w_i \geq 0$. Hence, $\tau(\alpha_i) = |w_i| = w_i$. Again, $\Supps(\alpha_i) \subseteq \Args^+$ and $\Atts(\alpha_i) \subseteq \Args^-$. By Definition~\ref{def:logregstr},
\begin{align} \label{eq5:wi}
    w_i = \tau(\alpha_i) &= \sigma(\alpha_i) - 
    \sum_{g \in \Supps(\alpha_i)}\frac{\sigma(g)}{\nu(g)} +
    \sum_{g \in \Atts(\alpha_i)}\frac{\sigma(g)}{\nu(g)}
\end{align}
For $\alpha_j \in \Args^-$, in contrast, we know that $c(\alpha_j) = 0$, so the corresponding weight in the LR model is $w_j < 0$. Hence, $\tau(\alpha_j) = |w_j| = -w_j$. By Definition~\ref{def:axplrqbaf}, all the supporters must support the same class whereas all the attackers must support the opposite  class. So, $\Supps(\alpha_j) \subseteq \Args^-$ and $\Atts(\alpha_j) \subseteq \Args^+$. By Definition~\ref{def:logregstr},
\begin{align} \label{eq5:wj}
    w_j = -\tau(\alpha_j) &= -\sigma(\alpha_j) + 
    \sum_{g \in \Supps(\alpha_j)}\frac{\sigma(g)}{\nu(g)} -
    \sum_{g \in \Atts(\alpha_j)}\frac{\sigma(g)}{\nu(g)}
\end{align}
By Definition~\ref{def:axplrqbaf}, $\Args - \{\delta\} = \Args^+ \cup \Args^- = \{\alpha_i|f_i=1\}$, so $\sum\limits_{i=1}^dw_if_i + b = \sum\limits_{\alpha_i \in \Args^+}w_i + \sum\limits_{\alpha_j \in \Args^-}w_j + b$. 
By summing up Equations~\ref{eq5:b}, \ref{eq5:wi} (for all $\alpha_i \in \Args^+$), and \ref{eq5:wj} (for all $\alpha_i \in \Args^-$), we obtain that
\begin{align*}
    \sum\limits_{i=1}^dw_if_i + b &= \sigma(\delta) + \sum\limits_{\alpha_i \in \Args^+}\sigma(\alpha_i) - \sum\limits_{\alpha_j \in \Args^-}\sigma(\alpha_j)\\
    &\quad - 
    \sum_{g \in \Supps(\delta)}\frac{\sigma(g)}{\nu(g)} +
    \sum_{g \in \Atts(\delta)}\frac{\sigma(g)}{\nu(g)} \\
    &\quad + \sum\limits_{\alpha_i \in \Args^+} \left(- 
    \sum_{g \in \Supps(\alpha_i)}\frac{\sigma(g)}{\nu(g)} +
    \sum_{g \in \Atts(\alpha_i)}\frac{\sigma(g)}{\nu(g)}\right)\\
    &\quad +\sum\limits_{\alpha_j \in \Args^-} \left(\sum_{g \in \Supps(\alpha_j)}\frac{\sigma(g)}{\nu(g)} -
    \sum_{g \in \Atts(\alpha_j)}\frac{\sigma(g)}{\nu(g)}\right) \\
    \sum\limits_{i=1}^dw_if_i + b &= \sigma(\delta) + \sum\limits_{\alpha_i \in \Args^+}\sigma(\alpha_i) - \sum\limits_{\alpha_j \in \Args^-}\sigma(\alpha_j)\\
    &\quad - {\color{blue} \left(\sum_{g \in \Supps(\delta)}\frac{\sigma(g)}{\nu(g)} + \sum\limits_{\alpha_i \in \Args^+}\sum_{g \in \Supps(\alpha_i)}\frac{\sigma(g)}{\nu(g)}+\sum\limits_{\alpha_j \in \Args^-}\sum_{g \in \Atts(\alpha_j)}\frac{\sigma(g)}{\nu(g)}\right)}\\ 
    &\quad + {\color{magenta} \left(\sum_{g \in \Atts(\delta)}\frac{\sigma(g)}{\nu(g)}+ \sum\limits_{\alpha_i \in \Args^+}\sum_{g \in \Atts(\alpha_i)}\frac{\sigma(g)}{\nu(g)}+\sum\limits_{\alpha_j \in \Args^-}\sum_{g \in \Supps(\alpha_j)}\frac{\sigma(g)}{\nu(g)}\right)}
\end{align*}
Next, we will show that $\sum\limits_{\alpha_i \in \Args^+}\sigma(\alpha_i)$ and the {\color{blue}blue} part above are equal. As $\alpha_i \in \Args^+$, it can appear as $g$ only in the blue part. In other words, $\alpha_i$ can either support another argument in $\Args^+$ or $\delta$ or attack another argument in $\Args^-$. If $\alpha_i$ supports or attacks $\nu(\alpha_i)$ arguments in total (where $\nu(\alpha_i)$ is the out-degree of $\alpha_i$), we will find exactly $\nu(\alpha_i)$ terms of $\frac{\sigma(\alpha_i)}{\nu(\alpha_i)}$ in the blue part, and they sum up to $\sigma(\alpha_i)$.    
Hence, for every $\sigma(\alpha_i)$ in $\sum\limits_{\alpha_i \in \Args^+}\sigma(\alpha_i)$, we can find the equivalent amount in the blue part. 
Meanwhile, $g$ in the blue part must come from $\Args^+$ only (not $\delta$ or $\Args^-$ according to Definition~\ref{def:axplrqbaf}).
Thereby, $\sum\limits_{\alpha_i \in \Args^+}\sigma(\alpha_i)$ and the blue part are equal and cancelling each other.

Similarly, $\alpha_j \in \Args^-$ can either support another argument in $\Args^-$ or attack another argument in $\Args^+$ or $\delta$. With the same logic as for the blue part, we obtain that $\sum\limits_{\alpha_j \in \Args^-}\sigma(\alpha_j)$ and the {\color{magenta}magenta} part are equal and cancelling each other.

Finally, we obtain $\sum\limits_{i=1}^dw_if_i + b = \sigma(\delta)$ as required.

\textit{Case 2: $c(\delta) = 0$ -- The  class supported by $\delta$ is Negative.}

The proof of this case is similar to the previous case, so we will highlight only the differences here. First, because $c(\delta) = 0$, the bias term of the LR model is $b < 0$. Hence, $\tau(\delta) = |b| = -b$. Equation~\ref{eq5:b} then becomes
\begin{align} \label{eq5:b-}
    b &= - \sigma(\delta) + 
    \sum_{g \in \Supps(\delta)}\frac{\sigma(g)}{\nu(g)} -
    \sum_{g \in \Atts(\delta)}\frac{\sigma(g)}{\nu(g)}
\end{align}
However, Equations~\ref{eq5:wi} and \ref{eq5:wj} remain the same.
By summing up Equations~\ref{eq5:b-}, \ref{eq5:wi} (for all $\alpha_i \in \Args^+$), and \ref{eq5:wj} (for all $\alpha_i \in \Args^-$), we obtain that
\begin{align*}
    \sum\limits_{i=1}^dw_if_i + b &= -\sigma(\delta) + \sum\limits_{\alpha_i \in \Args^+}\sigma(\alpha_i) - \sum\limits_{\alpha_j \in \Args^-}\sigma(\alpha_j)\\
    &\quad - {\color{blue} \left(\sum_{g \in \Atts(\delta)}\frac{\sigma(g)}{\nu(g)} + \sum\limits_{\alpha_i \in \Args^+}\sum_{g \in \Supps(\alpha_i)}\frac{\sigma(g)}{\nu(g)}+\sum\limits_{\alpha_j \in \Args^-}\sum_{g \in \Atts(\alpha_j)}\frac{\sigma(g)}{\nu(g)}\right)}\\ 
    &\quad + {\color{magenta} \left(\sum_{g \in \Supps(\delta)}\frac{\sigma(g)}{\nu(g)}+ \sum\limits_{\alpha_i \in \Args^+}\sum_{g \in \Atts(\alpha_i)}\frac{\sigma(g)}{\nu(g)}+\sum\limits_{\alpha_j \in \Args^-}\sum_{g \in \Supps(\alpha_j)}\frac{\sigma(g)}{\nu(g)}\right)}
\end{align*}
Because $\alpha_i \in \Args^+$ can either support another argument in $\Args^+$ or attack another argument in $\Args^-$ or $\delta$, with the same logic as in the previous case, we obtain that $\sum\limits_{\alpha_i \in \Args^+}\sigma(\alpha_i)$ and the {\color{blue}blue} part are equal and cancelling each other.
Similarly, $\sum\limits_{\alpha_j \in \Args^-}\sigma(\alpha_j)$ and the {\color{magenta}magenta} part are equal and cancelling each other. What remains is $\sum\limits_{i=1}^dw_if_i + b = -\sigma(\delta)$. So, $\sigma(\delta) = -\sum\limits_{i=1}^dw_if_i - b$ as required.

Finally, the second point of the theorem is pretty obvious. Using the result from the first point, the predicted probability of class $c(\delta)$ equals $sigmoid(\sigma(\delta))$ which is greater than 0.5 if $\sigma(\delta)>0$. So, it predicts $c(\delta)$. Otherwise, $sigmoid(\sigma(\delta)) < 0.5$, and it predicts the other class which is $1-c(\delta)$.
\end{proof}

\subsection{Proof of Theorem~\ref{theo:sigma'}} \label{proof:theo:sigma'}
\noindent\textbf{Theorem~\ref{theo:sigma'}.}
Given a \QBAF $\langle \Args, \Atts, \Supps, \tau, c \rangle$ and the corresponding \QBAFP $\langle \Args', \Attsp, \Suppsp, \tau', c' \rangle$, using the logistic regression semantics, we use $\sigma(a)$ and $\sigma(a)'$ to represent the strengths of $a \in \Args = \Args'$ in \QBAF and \QBAFP, respectively. The following statements are true for $a \in \Args = \Args'$.
\begin{itemize}
    \item If $\sigma(a) \geq 0$, then $\sigma(a)'= \sigma(a)$.
    \item If $\sigma(a) < 0$, $\sigma(a)'= -\sigma(a)$.
\end{itemize}

\begin{proof}
According to Theorem~\ref{theo:dag}, $G = \langle V, E \rangle$ for the \QBAF is a DAG where $V = \Args$ and $E = \Atts \cup \Supps$. Let $G' = \langle V', E' \rangle$ be the graph structure of \QBAFP where  $V' = \Args'$ and $E' = \Attsp \cup \Suppsp$. By Definition~\ref{def:postprocess}, we know that $V' = \Args' = \Args = V$ and $E' = \Attsp \cup \Suppsp \subseteq \Atts \cup \Supps = E$. Hence, $G'$ is a subgraph of $G$ and also a DAG\footnote{A subgraph of a DAG must be a DAG, since it cannot contain a cycle that does not exist in the supergraph.}. 

Then we can obtain the topological ordering $t$ of arguments in $V'$ which is the ordering of strength computation for \QBAFP. Assume that the order $t$ is $a_1, a_2, \ldots, a_k$. Obviously, $a_1$ must be an argument in $V'$ that has no attack or support. 
Furthermore, we can divide the vertices in $t$ into two groups (corresponding to the two bullet points of this theorem), one with $\sigma(a_i) \geq 0$ and the other with $\sigma(a_i) < 0$. 
We name arguments in the former group and the latter group as $g_1, \ldots, g_r$ and $l_1, \ldots, l_s$, respectively, where $g_i$ must be before $g_{i+1}$ in $t$, $l_i$ must be before $l_{i+1}$ in $t$, and $r+s=k$. 
So, $t$ could be written as, for example, $g_1, g_2, l_1, g_3, l_2, l_3, \ldots, g_r, l_{s-1}, l_s$.
In any case, $g_1$ must be $a_1$ since $a_1$ has neither attacker nor supporter and, therefore, $\sigma(a_1) = \tau(a_1) \geq 0$.
In general, the first part of $t$ must be $g_1, \ldots, g_{r^*}, l_1, \ldots$ where $1 \leq r^* \leq r$, and arguments after $l_1$ could be from either $g$ or $l$.
We will use mathematical induction on this topological ordering $t$ to prove the two bullet points of this theorem.
\begin{itemize}
    \item For $g_1$ (the base case): Because it has neither attacker nor supporter, $\sigma(g_1) = \tau(g_1)$ and $\sigma(g_1)' = \tau'(g_1)$. Since $\sigma(g_1) \geq 0$, by Definition~\ref{def:postprocess}, $\tau'(g_1) = \tau(g_1)$. Hence, $\sigma(g_1)' = \sigma(g_1)$ as required.
    \item For $g_i$ with $i \leq r^*$ (using strong induction): We have shown that $\sigma(g_1)' = \sigma(g_1)$. Next, assuming that $\sigma(g_h)' = \sigma(g_h)$ for $1 \leq h \leq i < r^*$, we need to show that $\sigma(g_{i+1})' = \sigma(g_{i+1})$ where $i+1 \leq r^*$. 
    
    Due to the topological ordering, all the original attackers and supporters of $g_{i+1}$ must be in $\{g_h | 1 \leq h \leq i\}$. As a result, for $a \in \Atts(g_{i+1}) \cup \Supps(g_{i+1})$, $\tau'(a) = \tau(a)$, $c'(a) = c(a)$, and  $\sigma(a)' = \sigma(a)$. Also, $\tau'(g_{i+1}) = \tau(g_{i+1})$ and $c'(g_{i+1}) = c(g_{i+1})$ because $\sigma(g_{i+1}) \geq 0$ by the definition of $g$.
    Since the classes of both $g_{i+1}$ and its original attackers and supporters do not change, $\Attsp(g_{i+1})=\Atts(g_{i+1})$ and $\Suppsp(g_{i+1})=\Supps(g_{i+1})$\footnote{For simplicity, we include the attackers and supporters $a$ where $\sigma(a) = 0$ in $\Attsp(g_{i+1})$ and $\Suppsp(g_{i+1})$, respectively, as they play no role when computing $\sigma(g_{i+1})'$ anyway. The same logic also applies to the next cases in this proof.}.
    
    Applying Definition~\ref{def:logregstr} to $g_{i+1}$ in \QBAFP, we obtain
    \begin{align*}
        \sigma(g_{i+1})' &= \tau'(g_{i+1}) + \sum_{b \in \Suppsp(g_{i+1})}\frac{\sigma(b)'}{\nu(b)} - \sum_{b \in \Attsp(g_{i+1})}\frac{\sigma(b)'}{\nu(b)} \\
        &= \tau(g_{i+1}) + \sum_{b \in \Supps(g_{i+1})}\frac{\sigma(b)}{\nu(b)} - \sum_{b \in \Atts(g_{i+1})}\frac{\sigma(b)}{\nu(b)} \\
        &= \sigma(g_{i+1})
    \end{align*}
    Hence, $\sigma(g_{i+1})' = \sigma(g_{i+1})$ where $i+1 \leq r^*$ as required. 
    \item For $l_1$: Because $\sigma(l_1) < 0$ by the definition of $l$, we need to show that $\sigma(l_1)' = -\sigma(l_1)$.
    
    According to the ordering $t$, all the original attackers and supporters of $l_1$ must be in $\{g_h | 1 \leq h \leq r^*\}$. As a result, for $a \in \Atts(l_1) \cup \Supps(l_1)$, $\tau'(a) = \tau(a)$, $c'(a) = c(a)$, and $\sigma(a)' = \sigma(a)$ (as proven above). 
    In contrast, since $\sigma(l_1) < 0$, by Definition~\ref{def:postprocess}, $\tau'(l_1) = -\tau(l_1)$ and $c'(l_1) = 1-c(l_1)$. In other words, the base score and the supported class of $l_1$ are flipped after post-processing. 
    Because the classes of the attackers and supporters are not change whereas the class of $l_1$ is flipped, $\Attsp(l_1) = \Supps(l_1)$ and $\Suppsp(l_1) = \Atts(l_1)$.
    
    Applying Definition~\ref{def:logregstr} to $l_1$ in \QBAFP, we obtain
    \begin{align*}
        \sigma(l_1)' &= \tau'(l_1) + \sum_{b \in \Suppsp(l_1)}\frac{\sigma(b)'}{\nu(b)} - \sum_{b \in \Attsp(l_1)}\frac{\sigma(b)'}{\nu(b)} \\
        &= -\tau(l_1) + \sum_{b \in \Atts(l_1)}\frac{\sigma(b)}{\nu(b)} - \sum_{b \in \Supps(l_1)}\frac{\sigma(b)}{\nu(b)} \\
        &= -\left(\tau(l_1) + \sum_{b \in \Supps(l_1)}\frac{\sigma(b)}{\nu(b)} - \sum_{b \in \Atts(l_1)}\frac{\sigma(b)}{\nu(b)}\right) \\
        &= -\sigma(l_1)
    \end{align*}
    Hence, $\sigma(l_1)' = -\sigma(l_1)$ as required. 
    
    \item For any argument $a_i$ in $t$ after $l_1$ (using strong induction): So far, we have shown that this theorem is true for $g_1, \ldots, g_{r^*}$ and $l_1$. Next, assuming that the theorem is true for any argument $a_1, \ldots, a_i$, we need to show that the theorem is also true for $a_{i+1}$ regardless of the group it belongs to.
    
    If $a_{i+1}$ belongs to the $g$ group (i.e., $\sigma(a_{i+1}) \geq 0$), we obtain that $\tau'(a_{i+1}) = \tau(a_{i+1})$ and $c'(a_{i+1}) = c(a_{i+1})$. The supported class of $a_{i+1}$ is not flipped after post-processing. The original attackers of $a_{i+1}$ could belong to either the $g$ group or the $l$ group. We use $\Atts_g(a_{i+1})$ and $\Atts_l(a_{i+1})$ to represent those sets of original attackers, respectively. 
    After post-processing, the attackers in $\Atts_g(a_{i+1})$ will still be attackers (due to the unchanged supported classes on both sides of the relations) whereas the ones in $\Atts_l(a_{i+1})$ will become supporters (due to the supported class flipped only on one side). 
    Similarly, the original supporters of $a_{i+1}$ could be split into $\Supps_g(a_{i+1})$ and $\Supps_l(a_{i+1})$. After post-processing, those in $\Supps_g(a_{i+1})$ will still be supporters while those in $\Supps_l(a_{i+1})$ will become attackers. 
    To sum up, $\Suppsp(a_{i+1})$ is the union of two disjoint sets -- $\Supps_g(a_{i+1})$ and $\Atts_l(a_{i+1})$. Meanwhile, $\Attsp(a_{i+1})$ is the union of two disjoint sets -- $\Atts_g(a_{i+1})$ and $\Supps_l(a_{i+1})$.
    
    Applying Definition~\ref{def:logregstr} to $a_{i+1}$ in \QBAFP, we obtain
    \begin{align*}
        \sigma(a_{i+1})' =&\quad \tau'(a_{i+1}) + \sum_{b \in \Suppsp(a_{i+1})}\frac{\sigma(b)'}{\nu(b)} - \sum_{b \in \Attsp(a_{i+1})}\frac{\sigma(b)'}{\nu(b)} \\
        =&\quad  \tau(a_{i+1}) + \sum_{b \in \Supps_g(a_{i+1})}\frac{\sigma(b)'}{\nu(b)} + \sum_{b \in \Atts_l(a_{i+1})}\frac{\sigma(b)'}{\nu(b)} \\
        &\quad - \sum_{b \in \Atts_g(a_{i+1})}\frac{\sigma(b)'}{\nu(b)} - \sum_{b \in \Supps_l(a_{i+1})}\frac{\sigma(b)'}{\nu(b)}\\
        =&\quad  \tau(a_{i+1}) + \sum_{b \in \Supps_g(a_{i+1})}\frac{\sigma(b)}{\nu(b)} + \sum_{b \in \Atts_l(a_{i+1})}\frac{-\sigma(b)}{\nu(b)} \\
        &\quad - \sum_{b \in \Atts_g(a_{i+1})}\frac{\sigma(b)}{\nu(b)} - \sum_{b \in \Supps_l(a_{i+1})}\frac{-\sigma(b)}{\nu(b)}\\
        =&\quad  \tau(a_{i+1}) + \left(\sum_{b \in \Supps_g(a_{i+1})}\frac{\sigma(b)}{\nu(b)} 
        + \sum_{b \in \Supps_l(a_{i+1})}\frac{\sigma(b)}{\nu(b)}\right)\\
        &\quad - \left(\sum_{b \in \Atts_g(a_{i+1})}\frac{\sigma(b)}{\nu(b)} + \sum_{b \in \Atts_l(a_{i+1})}\frac{\sigma(b)}{\nu(b)}\right) \\
        =&\quad  \tau(a_{i+1}) + \sum_{b \in \Supps(a_{i+1})}\frac{\sigma(b)}{\nu(b)} 
        - \sum_{b \in \Atts(a_{i+1})}\frac{\sigma(b)}{\nu(b)}\\
        =&\quad \sigma(a_{i+1})
    \end{align*}
    Hence, for $a_{i+1}$ where $\sigma(a_{i+1}) \geq 0$, $\sigma(a_{i+1})' = \sigma(a_{i+1})$ as the theorem stated.
    
    Analogously, if $a_{i+1}$ belongs to the $l$ group (i.e., $\sigma(a_{i+1}) < 0$), we obtain that $\tau'(a_{i+1}) = -\tau(a_{i+1})$ and $c'(a_{i+1}) = 1 - c(a_{i+1})$. After post-processing, the supported class of $a_{i+1}$ is flipped. 
    Also, the attackers in $\Atts_g(a_{i+1})$ will become supporters (due to the supported class flipped only on one side of the relations at $a_{i+1}$) whereas the ones in $\Atts_l(a_{i+1})$ will still be attackers (due to the supported class flipped on both sides).
    Similarly, the supporters in $\Supps_g(a_{i+1})$ will become attackers, while those in $\Supps_l(a_{i+1})$ will still be supporters.
    To sum up, $\Suppsp(a_{i+1})$ is the union of two disjoint sets -- $\Supps_l(a_{i+1})$ and $\Atts_g(a_{i+1})$. Meanwhile, $\Attsp(a_{i+1})$ is the union of two disjoint sets -- $\Atts_l(a_{i+1})$ and $\Supps_g(a_{i+1})$.
    
    Applying Definition~\ref{def:logregstr} to $a_{i+1}$ in \QBAFP, we obtain
    \begin{align*}
        \sigma(a_{i+1})' =&\quad \tau'(a_{i+1}) + \sum_{b \in \Suppsp(a_{i+1})}\frac{\sigma(b)'}{\nu(b)} - \sum_{b \in \Attsp(a_{i+1})}\frac{\sigma(b)'}{\nu(b)} \\
        =&\quad  -\tau(a_{i+1}) + \sum_{b \in \Supps_l(a_{i+1})}\frac{\sigma(b)'}{\nu(b)} + \sum_{b \in \Atts_g(a_{i+1})}\frac{\sigma(b)'}{\nu(b)} \\
        &\quad - \sum_{b \in \Atts_l(a_{i+1})}\frac{\sigma(b)'}{\nu(b)} - \sum_{b \in \Supps_g(a_{i+1})}\frac{\sigma(b)'}{\nu(b)}\\
        =&\quad  -\tau(a_{i+1}) + \sum_{b \in \Supps_l(a_{i+1})}\frac{-\sigma(b)}{\nu(b)} + \sum_{b \in \Atts_g(a_{i+1})}\frac{\sigma(b)}{\nu(b)} \\
        &\quad - \sum_{b \in \Atts_l(a_{i+1})}\frac{-\sigma(b)}{\nu(b)} - \sum_{b \in \Supps_g(a_{i+1})}\frac{\sigma(b)}{\nu(b)}\\
        =&\quad  -\tau(a_{i+1}) - \left(\sum_{b \in \Supps_g(a_{i+1})}\frac{\sigma(b)}{\nu(b)} 
        + \sum_{b \in \Supps_l(a_{i+1})}\frac{\sigma(b)}{\nu(b)}\right)\\
        &\quad + \left(\sum_{b \in \Atts_g(a_{i+1})}\frac{\sigma(b)}{\nu(b)} + \sum_{b \in \Atts_l(a_{i+1})}\frac{\sigma(b)}{\nu(b)}\right) \\
        =&\quad  -\tau(a_{i+1}) - \sum_{b \in \Supps(a_{i+1})}\frac{\sigma(b)}{\nu(b)} 
        + \sum_{b \in \Atts(a_{i+1})}\frac{\sigma(b)}{\nu(b)}\\
        =&\quad -\sigma(a_{i+1})
    \end{align*}
    Hence, for $a_{i+1}$ where $\sigma(a_{i+1}) < 0$, $\sigma(a_{i+1})' = -\sigma(a_{i+1})$ as the theorem stated.
    
    From both cases, the induction step is completed.
\end{itemize}
Our proof has covered all the arguments in the topological ordering $t$. Thus, the theorem is true for $a \in \Args = \Args'$.
\end{proof}

\subsection{Proof of Dialectical Properties} \label{proof:properties}

Here, we conduct proofs of the results in Table~\ref{tab5:properties} concerning the group properties for \QBAF and \QBAFP listed in Table~\ref{table:new}.
When conducting the proofs, we may use the \QBAF and the corresponding \QBAFP in Figures~\ref{fig5:exqb} and \ref{fig5:exqbpp}, respectively, as a counterexample.

\begin{figure}[!t]
\centering
\small
\begin{tikzpicture}[node distance=2cm]
\node [fill=green!20,label={[xshift=1.3cm, yshift=-0.7cm](0.1, 0.9)}] (d) at (0, 0) [argsm] {$\delta$};
\node [fill=red!20,label={[xshift=0.0cm, yshift=0.07cm](0.4, 0.3)}] (a1) at (-5.4, -3) [argsm] {$\alpha_1$};
\node [fill=green!20,label={[xshift=0.5cm, yshift=0.07cm](0.4, 0.5)}] (a2) at (-3.6, -3) [argsm] {$\alpha_2$};
\node [fill=red!20,label={[xshift=0.0cm, yshift=-1.5cm](0.4, 0.4)}] (a3) at (-1.8, -3) [argsm] {$\alpha_3$};
\node [fill=green!20,label={[xshift=0.0cm, yshift=0.07cm](0.8, 0.4)}] (a4) at (0, -3) [argsm] {$\alpha_4$};
\node [fill=green!20,label={[xshift=0.0cm, yshift=0.07cm](0.8, 0.5)}] (a5) at (1.8, -3) [argsm] {$\alpha_5$};
\node [fill=red!20,label={[xshift=0.0cm, yshift=0.07cm](0.2, 0.5)}] (a6) at (3.6, -3) [argsm] {$\alpha_6$};
\node [fill=green!20,label={[xshift=0.0cm, yshift=0.07cm](0.2, 0.6)}] (a7) at (5.4, -3) [argsm] {$\alpha_7$};
\node [fill=red!20,label={[xshift=1.3cm, yshift=-0.7cm](0.3, -0.2)}] (a8) at (-4.5, -5.5) [argsm] {$\alpha_8$};
\node [fill=green!20,label={[xshift=1.3cm, yshift=-0.7cm](0.5, 0.5)}] (a9) at (-4.5, -8) [argsm] {$\alpha_9$};
\node [fill=red!20,label={[xshift=0.0cm, yshift=-1.5cm](0.4, 0.4)}] (a10) at (0, -5.5) [argsm] {$\alpha_{10}$};
\node [fill=red!20,label={[xshift=0.0cm, yshift=-1.5cm](0.6, 0.6)}] (a11) at (2.7, -5.5) [argsm] {$\alpha_{11}$};
\node [fill=green!20,label={[xshift=0.0cm, yshift=-1.5cm](0.4, 0.4)}] (a12) at (5.4, -5.5) [argsm] {$\alpha_{12}$};
\tatt{a1}{d};
\tsupp{a2}{d};
\tatt{a3}{d};
\draw [rel] (a4) -- node[anchor=south west] {\large \textbf{+}} (d);
\tsupp{a5}{d};
\tatt{a6}{d};
\tsupp{a7}{d};
\tsupp{a8}{a1};
\tatt{a8}{a2};
\draw [rel] (a9) -- node[anchor=south west] {\large \textbf{--}} (a8);
\draw [rel] (a10) -- node[anchor=south west] {\large \textbf{--}} (a4);
\tatt{a11}{a5}
\tsupp{a11}{a6}
\draw [rel] (a12) -- node[anchor=south west] {\large \textbf{+}} (a7);
\end{tikzpicture}
\caption{Example of \QBAF used as a counterexample in Appendix~\ref{proof:properties}. With each argument, there is a value pair $(x, y)$ where $x$ and $y$ represent the base score and the strength (based on the logistic regression semantics $\sigma$) of the argument, respectively. The color represents the argument's supported class (i.e., green for positive (1) and red for negative (0)).}\label{fig5:exqb}
\end{figure}
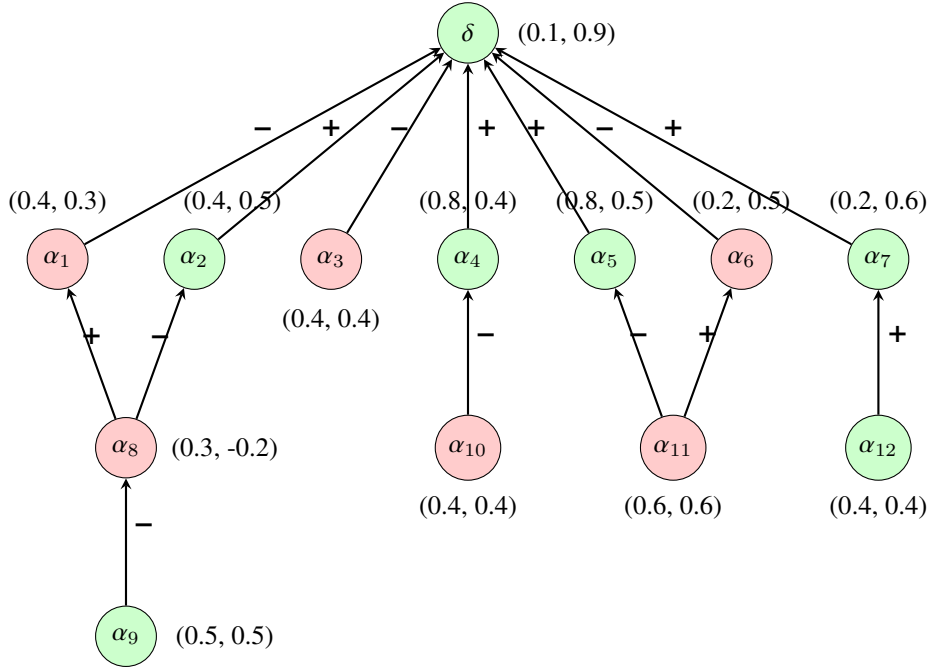

\begin{preprop}
Both $\langle \QBAF, \sigma \rangle$ and $\langle \QBAFP, \sigma \rangle$ satisfy GP1.
\end{preprop}
\begin{proof}
According to Definition~\ref{def:logregstr}, with $\Atts(\alpha)=\emptyset$ and $\Supps(\alpha)=\emptyset$, we have
$\sigma(\alpha) = \tau(\alpha) + \sum_{b \in \emptyset}\frac{\sigma(b)}{\nu(b)} - \sum_{b \in \emptyset}\frac{\sigma(b)}{\nu(b)} = \tau(\alpha)$
as required.

Hence, both $\langle \QBAF, \sigma \rangle$ and $\langle \QBAFP, \sigma \rangle$ satisfy GP1.
\end{proof}

\begin{preprop}
$\langle \QBAF, \sigma \rangle$ does not satisfy GP2, but $\langle \QBAFP, \sigma \rangle$ does.
\end{preprop}
\begin{proof}
According to Definition~\ref{def:logregstr}, with $\Supps(\alpha)=\emptyset$, we have
\[\sigma(\alpha) = \tau(\alpha) + 
\sum_{b \in \emptyset}\frac{\sigma(b)}{\nu(b)} - 
\sum_{b \in \Atts(\alpha)}\frac{\sigma(b)}{\nu(b)} = 
\tau(\alpha) - 
\sum_{b \in \Atts(\alpha)}\frac{\sigma(b)}{\nu(b)}.\]
In \QBAF, $\sigma(b)$ could be either positive or negative, so it is possible that $\sum_{b \in \Atts(\alpha)}\frac{\sigma(b)}{\nu(b)} < 0$, making $\sigma(\alpha) > \tau(\alpha)$. 
$\alpha_2$ in Figure~\ref{fig5:exqb} is one counterexample of this GP.
By contrast, in \QBAFP, we have removed any attacker and supporter of which the strength is zero. According to this and Corollary~\ref{cor:sigma'}, $\sigma(b) > 0$ for $b \in \Atts(\alpha)$. Therefore, $\sum_{b \in \Atts(\alpha)}\frac{\sigma(b)}{\nu(b)} > 0$, resulting in $\sigma(\alpha) < \tau(\alpha)$ as required.

Hence, $\langle \QBAF, \sigma \rangle$ does not satisfy GP2, but $\langle \QBAFP, \sigma \rangle$ does.
\end{proof}

\begin{preprop}
$\langle \QBAF, \sigma \rangle$ does not satisfy GP3, but $\langle \QBAFP, \sigma \rangle$ does.
\end{preprop}
\begin{proof}
According to Definition~\ref{def:logregstr}, with $\Atts(\alpha)=\emptyset$, we have
\[\sigma(\alpha) = \tau(\alpha) + 
\sum_{b \in \Supps(\alpha)}\frac{\sigma(b)}{\nu(b)} - 
\sum_{b \in \emptyset}\frac{\sigma(b)}{\nu(b)} = 
\tau(\alpha) + 
\sum_{b \in \Supps(\alpha)}\frac{\sigma(b)}{\nu(b)}.\]
As in the proof of GP2, for \QBAF, $\sum_{b \in \Supps(\alpha)}\frac{\sigma(b)}{\nu(b)}$ could be negative, rendering $\sigma(\alpha) < \tau(\alpha)$.
$\alpha_1$ in Figure~\ref{fig5:exqb} is one counterexample of this GP.
Meanwhile, in \QBAFP, $\sum_{b \in \Supps(\alpha)}\frac{\sigma(b)}{\nu(b)} > 0$. Therefore, $\sigma(\alpha) > \tau(\alpha)$ as required.

Hence, $\langle \QBAF, \sigma \rangle$ does not satisfy GP3, but $\langle \QBAFP, \sigma \rangle$ does.
\end{proof}

\begin{preprop}
$\langle \QBAF, \sigma \rangle$ does not satisfy GP4, but $\langle \QBAFP, \sigma \rangle$ does.
\end{preprop}
\begin{proof}
From $\sigma(\alpha) < \tau(\alpha)$ with Definition~\ref{def:logregstr}, we obtain that
\[\sigma(\alpha) = \tau(\alpha) + 
\sum_{b \in \Supps(\alpha)}\frac{\sigma(b)}{\nu(b)} - 
\sum_{b \in \Atts(\alpha)}\frac{\sigma(b)}{\nu(b)} < \tau(\alpha).\]
Therefore, 
\[\sum_{b \in \Supps(\alpha)}\frac{\sigma(b)}{\nu(b)} - 
\sum_{b \in \Atts(\alpha)}\frac{\sigma(b)}{\nu(b)} < 0.\]
For \QBAF, it is possible that $\Atts(\alpha)=\emptyset$ because $\sum_{b \in \Supps(\alpha)}\frac{\sigma(b)}{\nu(b)}$ could be negative, satisfying the inequality. So, the property is not satisfied. 
$\alpha_1$ in Figure~\ref{fig5:exqb} is one counterexample of this GP.
In contrast, for \QBAFP, let us proof by contradiction. Assume that $\Atts(\alpha)=\emptyset$. We will obtain $\sum_{b \in \Supps(\alpha)}\frac{\sigma(b)}{\nu(b)} < 0$ which is impossible under \QBAFP where $\sigma(b) > 0$. Therefore, it must be the case that $\Atts(\alpha)\neq\emptyset$. 

Hence, $\langle \QBAF, \sigma \rangle$ does not satisfy GP4, but $\langle \QBAFP, \sigma \rangle$ does.
\end{proof}

\begin{preprop}
$\langle \QBAF, \sigma \rangle$ does not satisfy GP5, but $\langle \QBAFP, \sigma \rangle$ does.
\end{preprop}
\begin{proof}
From $\sigma(\alpha) > \tau(\alpha)$ with Definition~\ref{def:logregstr}, we obtain that
\[\sigma(\alpha) = \tau(\alpha) + 
\sum_{b \in \Supps(\alpha)}\frac{\sigma(b)}{\nu(b)} - 
\sum_{b \in \Atts(\alpha)}\frac{\sigma(b)}{\nu(b)} > \tau(\alpha).\]
Therefore, 
\[\sum_{b \in \Supps(\alpha)}\frac{\sigma(b)}{\nu(b)} - 
\sum_{b \in \Atts(\alpha)}\frac{\sigma(b)}{\nu(b)} > 0.\]
As in the proof of GP4, for \QBAF, it is possible that $\Supps(\alpha)=\emptyset$ because $\sum_{b \in \Atts(\alpha)}\frac{\sigma(b)}{\nu(b)}$ could be negative, satisfying the inequality. So, the property is not satisfied.
$\alpha_2$ in Figure~\ref{fig5:exqb} is one counterexample of this GP.
In contrast, for \QBAFP, let us proof by contradiction. Assume that $\Supps(\alpha)=\emptyset$. We will obtain $- 
\sum_{b \in \Atts(\alpha)}\frac{\sigma(b)}{\nu(b)} > 0$. Thus, $\sum_{b \in \Atts(\alpha)}\frac{\sigma(b)}{\nu(b)} < 0$  which is impossible under \QBAFP where $\sigma(b) > 0$. Therefore, it must be the case that $\Supps(\alpha)\neq\emptyset$. 

Hence, $\langle \QBAF, \sigma \rangle$ does not satisfy GP5, but $\langle \QBAFP, \sigma \rangle$ does.
\end{proof}

\begin{preprop}
Both $\langle \QBAF, \sigma \rangle$ and $\langle \QBAFP, \sigma \rangle$ satisfy GP6.
\end{preprop}
\begin{proof}
According to Definition~\ref{def:logregstr}, we have
\[\sigma(\alpha) = \tau(\alpha) + 
\sum_{b \in \Supps(\alpha)}\frac{\sigma(b)}{\nu(b)} - 
\sum_{b \in \Atts(\alpha)}\frac{\sigma(b)}{\nu(b)}.\]
Because $\Atts(\alpha)=\Atts(\beta)$, $\Supps(\alpha)=\Supps(\beta)$, and $\tau(\alpha)=\tau(\beta)$, we can replace $\tau(\alpha)$, $\Supps(\alpha)$, and $\Atts(\alpha)$ by $\tau(\beta)$, $\Supps(\beta)$, and $\Atts(\beta)$, respectively. Therefore, 
\[\sigma(\alpha) = \tau(\beta) + 
\sum_{b \in \Supps(\beta)}\frac{\sigma(b)}{\nu(b)} - 
\sum_{b \in \Atts(\beta)}\frac{\sigma(b)}{\nu(b)} = \sigma(\beta).\]

Hence, both $\langle \QBAF, \sigma \rangle$ and $\langle \QBAFP, \sigma \rangle$ satisfy GP6.
\end{proof}

\begin{preprop}
$\langle \QBAF, \sigma \rangle$ does not satisfy GP7, but $\langle \QBAFP, \sigma \rangle$ does.
\end{preprop}
\begin{proof}
Since $\Atts(\alpha)\subset\Atts(\beta)$, we can partition $\Atts(\beta)$ into two disjoint sets which are $\Atts(\alpha)$ and a non-empty set of arguments $X = \Atts(\beta) - \Atts(\alpha)$. According to Definition~\ref{def:logregstr}, we have
\begin{align*}
    \sigma(\beta) &= \tau(\beta) + 
        \sum_{b \in \Supps(\beta)}\frac{\sigma(b)}{\nu(b)} - 
        \sum_{b \in \Atts(\beta)}\frac{\sigma(b)}{\nu(b)} \\
        &= \tau(\beta) + 
        \sum_{b \in \Supps(\beta)}\frac{\sigma(b)}{\nu(b)} 
        - \sum_{b \in \Atts(\alpha)}\frac{\sigma(b)}{\nu(b)}
        - \sum_{b \in X}\frac{\sigma(b)}{\nu(b)}
\end{align*}
Because $\Supps(\alpha)=\Supps(\beta)$ and $\tau(\alpha)=\tau(\beta)$, we obtain that
\begin{align*}
    \sigma(\beta)
        &= \tau(\alpha) + 
        \sum_{b \in \Supps(\alpha)}\frac{\sigma(b)}{\nu(b)} 
        - \sum_{b \in \Atts(\alpha)}\frac{\sigma(b)}{\nu(b)}
        - \sum_{b \in X}\frac{\sigma(b)}{\nu(b)} \\
        &= \sigma(\alpha) - \sum_{b \in X}\frac{\sigma(b)}{\nu(b)}
\end{align*}
As in the previous proofs, for \QBAF, it is possible that $\sum_{b \in X}\frac{\sigma(b)}{\nu(b)} < 0$, rendering $\sigma(\alpha)<\sigma(\beta)$ and unsatisfying GP7.
We can find a counterexample in Figure~\ref{fig5:exqb} with $\alpha = \alpha_3$ and $\beta = \alpha_2$.
In contrast, for \QBAFP, $\sigma(b) > 0$. Therefore, $\sum_{b \in X}\frac{\sigma(b)}{\nu(b)}$ is always greater than 0. As a result, $\sigma(\alpha)>\sigma(\beta)$ as required.

Hence, $\langle \QBAF, \sigma \rangle$ does not satisfy GP7, but $\langle \QBAFP, \sigma \rangle$ does.
\end{proof}

\begin{preprop}
$\langle \QBAF, \sigma \rangle$ does not satisfy GP8, but $\langle \QBAFP, \sigma \rangle$ does.
\end{preprop}
\begin{proof}
Since $\Supps(\alpha)\subset\Supps(\beta)$, we can partition $\Supps(\beta)$ into two disjoint sets which are $\Supps(\alpha)$ and a non-empty set of arguments $X = \Supps(\beta) - \Supps(\alpha)$. According to Definition~\ref{def:logregstr}, we have
\begin{align*}
    \sigma(\beta) &= \tau(\beta) + 
        \sum_{b \in \Supps(\beta)}\frac{\sigma(b)}{\nu(b)} - 
        \sum_{b \in \Atts(\beta)}\frac{\sigma(b)}{\nu(b)} \\
        &= \tau(\beta) + 
        \sum_{b \in \Supps(\alpha)}\frac{\sigma(b)}{\nu(b)} 
        + \sum_{b \in X}\frac{\sigma(b)}{\nu(b)}
        - \sum_{b \in \Atts(\beta)}\frac{\sigma(b)}{\nu(b)}
\end{align*}
Because $\Atts(\alpha)=\Atts(\beta)$ and $\tau(\alpha)=\tau(\beta)$, we obtain that
\begin{align*}
    \sigma(\beta)
        &= \tau(\alpha) + 
        \sum_{b \in \Supps(\alpha)}\frac{\sigma(b)}{\nu(b)} 
        - \sum_{b \in \Atts(\alpha)}\frac{\sigma(b)}{\nu(b)}
        + \sum_{b \in X}\frac{\sigma(b)}{\nu(b)} \\
        &= \sigma(\alpha) + \sum_{b \in X}\frac{\sigma(b)}{\nu(b)}
\end{align*}
As in the previous proofs, for \QBAF, it is possible that $\sum_{b \in X}\frac{\sigma(b)}{\nu(b)} < 0$, rendering $\sigma(\alpha)>\sigma(\beta)$ and unsatisfying GP8.
We can find a counterexample in Figure~\ref{fig5:exqb} with $\alpha = \alpha_3$ and $\beta = \alpha_1$.
In contrast, for \QBAFP, $\sigma(b) > 0$. Therefore, $\sum_{b \in X}\frac{\sigma(b)}{\nu(b)}$ is always greater than 0. As a result, $\sigma(\alpha)<\sigma(\beta)$ as required.

Hence, $\langle \QBAF, \sigma \rangle$ does not satisfy GP8, but $\langle \QBAFP, \sigma \rangle$ does.
\end{proof}

\begin{preprop}
Both $\langle \QBAF, \sigma \rangle$ and $\langle \QBAFP, \sigma \rangle$ satisfy GP9.
\end{preprop}
\begin{proof}
Let us proof by contradiction. Assume that $\sigma(\alpha)\geq\sigma(\beta)$. Applying Definition~\ref{def:logregstr}, we obtain
\[\tau(\alpha) + 
        \sum_{b \in \Supps(\alpha)}\frac{\sigma(b)}{\nu(b)} - 
        \sum_{b \in \Atts(\alpha)}\frac{\sigma(b)}{\nu(b)}
        \geq \tau(\beta) + 
        \sum_{b \in \Supps(\beta)}\frac{\sigma(b)}{\nu(b)} - 
        \sum_{b \in \Atts(\beta)}\frac{\sigma(b)}{\nu(b)}\]
Because $\Atts(\alpha)=\Atts(\beta)$ and $\Supps(\alpha)=\Supps(\beta)$, 
\begin{align*}
    \tau(\alpha) + 
        \sum_{b \in \Supps(\alpha)}\frac{\sigma(b)}{\nu(b)} - 
        \sum_{b \in \Atts(\alpha)}\frac{\sigma(b)}{\nu(b)}
        &\geq \tau(\beta) + 
        \sum_{b \in \Supps(\alpha)}\frac{\sigma(b)}{\nu(b)} - 
        \sum_{b \in \Atts(\alpha)}\frac{\sigma(b)}{\nu(b)}\\
    \tau(\alpha) &\geq \tau(\beta)
\end{align*}
However, this contradicts the given statement that $\tau(\alpha)<\tau(\beta)$, so $\sigma(\alpha)\geq\sigma(\beta)$ cannot be true. Thus, $\sigma(\alpha)<\sigma(\beta)$ as required.

Hence, both $\langle \QBAF, \sigma \rangle$ and $\langle \QBAFP, \sigma \rangle$ satisfy GP9.
\end{proof}

\begin{figure}[!t]
\centering
\small
\begin{tikzpicture}[node distance=2cm]
\node [fill=green!20,label={[xshift=1.3cm, yshift=-0.7cm](0.1, 0.9)}] (d) at (0, 0) [argsm] {$\delta$};
\node [fill=red!20,label={[xshift=0.0cm, yshift=0.07cm](0.4, 0.3)}] (a1) at (-5.4, -3) [argsm] {$\alpha_1$};
\node [fill=green!20,label={[xshift=0.5cm, yshift=0.07cm](0.4, 0.5)}] (a2) at (-3.6, -3) [argsm] {$\alpha_2$};
\node [fill=red!20,label={[xshift=0.0cm, yshift=-1.5cm](0.4, 0.4)}] (a3) at (-1.8, -3) [argsm] {$\alpha_3$};
\node [fill=green!20,label={[xshift=0.0cm, yshift=0.07cm](0.8, 0.4)}] (a4) at (0, -3) [argsm] {$\alpha_4$};
\node [fill=green!20,label={[xshift=0.0cm, yshift=0.07cm](0.8, 0.5)}] (a5) at (1.8, -3) [argsm] {$\alpha_5$};
\node [fill=red!20,label={[xshift=0.0cm, yshift=0.07cm](0.2, 0.5)}] (a6) at (3.6, -3) [argsm] {$\alpha_6$};
\node [fill=green!20,label={[xshift=0.0cm, yshift=0.07cm](0.2, 0.6)}] (a7) at (5.4, -3) [argsm] {$\alpha_7$};
\node [fill=green!20,label={[xshift=1.3cm, yshift=-0.7cm](-0.3, 0.2)}] (a8) at (-4.5, -5.5) [argsm] {$\alpha_8$};
\node [fill=green!20,label={[xshift=1.3cm, yshift=-0.7cm](0.5, 0.5)}] (a9) at (-4.5, -8) [argsm] {$\alpha_9$};
\node [fill=red!20,label={[xshift=0.0cm, yshift=-1.5cm](0.4, 0.4)}] (a10) at (0, -5.5) [argsm] {$\alpha_{10}$};
\node [fill=red!20,label={[xshift=0.0cm, yshift=-1.5cm](0.6, 0.6)}] (a11) at (2.7, -5.5) [argsm] {$\alpha_{11}$};
\node [fill=green!20,label={[xshift=0.0cm, yshift=-1.5cm](0.4, 0.4)}] (a12) at (5.4, -5.5) [argsm] {$\alpha_{12}$};
\tatt{a1}{d};
\tsupp{a2}{d};
\tatt{a3}{d};
\draw [rel] (a4) -- node[anchor=south west] {\large \textbf{+}} (d);
\tsupp{a5}{d};
\tatt{a6}{d};
\tsupp{a7}{d};
\tatt{a8}{a1};
\tsupp{a8}{a2};
\draw [rel] (a9) -- node[anchor=south west] {\large \textbf{+}} (a8);
\draw [rel] (a10) -- node[anchor=south west] {\large \textbf{--}} (a4);
\tatt{a11}{a5}
\tsupp{a11}{a6}
\draw [rel] (a12) -- node[anchor=south west] {\large \textbf{+}} (a7);
\end{tikzpicture}
\caption{The corresponding \QBAFP of the \QBAF in Figure~\ref{fig5:exqb}, used as a counterexample in Appendix~\ref{proof:properties}. With each argument, there is a value pair $(x, y)$ where $x$ and $y$ represent the base score and the strength (based on the logistic regression semantics $\sigma$) of the argument, respectively. The color represents the argument's supported class (i.e., green for positive (1) and red for negative (0)).}\label{fig5:exqbpp}
\end{figure}
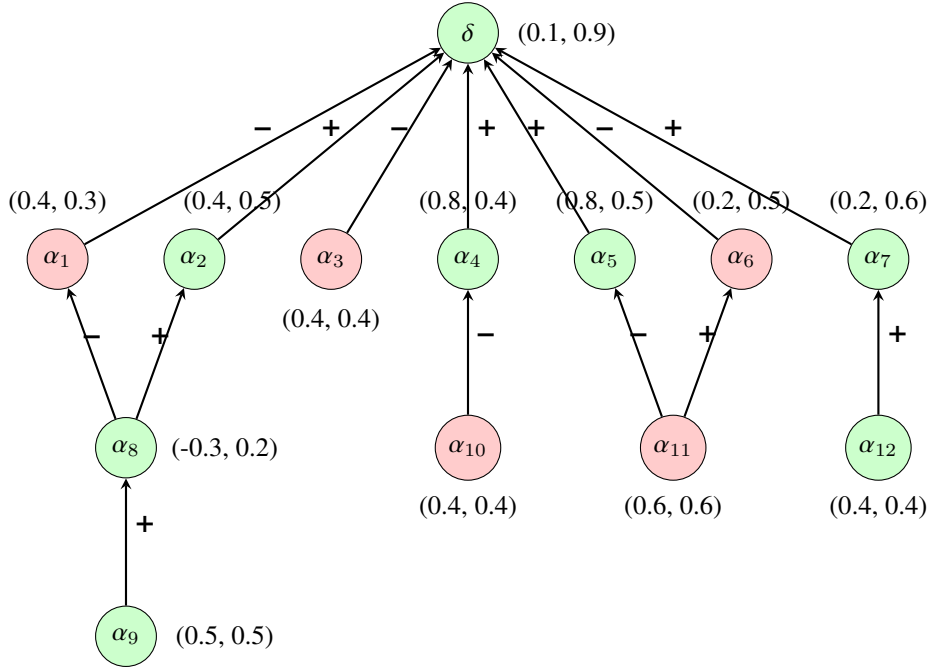

Note again that the definition of $<$ between two sets used in GP10 and GP11 is defined as follows. Given $P$ and $Q$ are subsets of $\Args$, $P \leq Q$ if and only if there exists an injective mapping $f$ from $P$ to $Q$ such that $\forall \alpha \in P, \sigma(\alpha) \leq \sigma(f(\alpha))$. Furthermore, $P < Q$ if and only if $P \leq Q$ but $Q \nleq P$.

\begin{preprop}
Both $\langle \QBAF, \sigma \rangle$ and $\langle \QBAFP, \sigma \rangle$ do not satisfy GP10.
\end{preprop}
\begin{proof}
Counterexamples of this GP for \QBAF and \QBAFP are in Figures~\ref{fig5:exqb} and \ref{fig5:exqbpp}, respectively, where $\alpha = \alpha_4$ and $\beta = \alpha_5$. 
We can see that $\Atts(\alpha)\leq\Atts(\beta)$ as we have a mapping from $\Atts(\alpha)$ to $\Atts(\beta)$, $f=\{(\alpha_{10}, \alpha_{11})\}$, where $\forall a \in \Atts(\alpha), \sigma(a) \leq \sigma(f(a))$. In this case, $\sigma(\alpha_{10}) \leq \sigma(\alpha_{11})$. However, $\Atts(\beta)\nleq\Atts(\alpha)$. So, $\Atts(\alpha)<\Atts(\beta)$. 
In addition, $\Supps(\alpha)=\Supps(\beta)=\emptyset$ and $\tau(\alpha)=\tau(\beta)=0.4$.
Nevertheless, $\sigma(\alpha)=0.4$ and $\sigma(\beta)=0.5$, so $\sigma(\alpha)>\sigma(\beta)$ is not true. 

Hence, both $\langle \QBAF, \sigma \rangle$ and $\langle \QBAFP, \sigma \rangle$ do not satisfy GP10. (This is because $\sigma$ considers not only the strengths of the attackers and the supporters but also their out-degrees.)
\end{proof}

\begin{preprop}
Both $\langle \QBAF, \sigma \rangle$ and $\langle \QBAFP, \sigma \rangle$ do not satisfy GP11.
\end{preprop}
\begin{proof}
Counterexamples of this GP for \QBAF and \QBAFP are in Figures~\ref{fig5:exqb} and \ref{fig5:exqbpp}, respectively, where $\alpha = \alpha_7$ and $\beta = \alpha_6$. 
We can see that $\Supps(\alpha)\leq\Supps(\beta)$ as we have a mapping from $\Supps(\alpha)$ to $\Supps(\beta)$, $f=\{(\alpha_{12}, \alpha_{11})\}$, where $\forall a \in \Supps(\alpha), \sigma(a) \leq \sigma(f(a))$. In this case, $\sigma(\alpha_{12}) \leq \sigma(\alpha_{11})$. However, $\Supps(\beta)\nleq\Supps(\alpha)$. So, $\Supps(\alpha)<\Supps(\beta)$. 
In addition, $\Atts(\alpha)=\Atts(\beta)=\emptyset$ and $\tau(\alpha)=\tau(\beta)=0.2$.
Nevertheless, $\sigma(\alpha)=0.6$ and $\sigma(\beta)=0.5$, so $\sigma(\alpha)<\sigma(\beta)$ is not true. 

Hence, both $\langle \QBAF, \sigma \rangle$ and $\langle \QBAFP, \sigma \rangle$ do not satisfy GP11. (This is because $\sigma$ considers not only the strengths of the attackers and the supporters but also their out-degrees.)
\end{proof}

\section{Machine Learning Terminology} \label{app:terminology}
This section explains meanings of technical terms concerning machine learning and classification that are used in this paper.

\paragraph{Dataset Splits} Working on a classification task, we usually divide a dataset we have into three splits. 
The first split is a \textit{training set}, which is used to train our classification model. 
The second split is a \textit{test set}, which is used to evaluate the final trained model.
The third split is a \textit{development set} (also called a \textit{validation set}), which is used to evaluate model(s) under development so as to choose the best model architecture or hyperparameters.
To ensure that the model can generalize beyond what it sees during training and development, all the three data splits must be mutually exclusive. 

Because logistic regression (LR) does not have any hyperparameters or multiple architectures to choose, we did not use the development set while training the LR models in our experiments (Sections \ref{sec:setup}-\ref{sec:tutorial}). 
However, in Section~\ref{sec:tutorial}, we used the development set of the deceptive hotel review dataset for two purposes -- \textit{(i)} to tune the regularization hyperparameter of the support vector machine (SVM) model and \textit{(ii)} to generate explanations for tutoring human participants to detect deceptive reviews. 

\paragraph{Evaluation Metrics}
For binary classification, let $y$ and $\hat{y}$ be the true class and the predicted class of an example $x$, respectively. Hence, both $y$ and $\hat{y}$ could be either 0 or 1, resulting in four possible situations.
\begin{itemize}
\item \textbf{True Positive (TP)}: The model correctly predicts that the class of $x$ is 1 (i.e., $y = \hat{y} = 1$).
\item \textbf{True Negative (TN)}: The model correctly predicts that the class of $x$ is 0 (i.e., $y = \hat{y} = 0$).
\item \textbf{False Positive (FP)}: The model predicts that the class of $x$ is 1 (i.e., $\hat{y} = 1$), but the true class of $x$ is in fact 0 (i.e., $y = 0$).
\item \textbf{False Negative (TN)}: The model predicts that the class of $x$ is 0 (i.e., $\hat{y} = 0$), but the true class of $x$ is in fact 1 (i.e., $y = 1$).
\end{itemize}

By default, we consider class 1 to be the positive class and consider class 0 to be the negative class. However, we can also consider the four situations above with respect to a specific class $c$. Particularly, $TP_c$ (true positives) and $FP_c$ (false positives) are the number of examples predicted as class $c$ by the classifier that are correct and incorrect predictions, respectively.
$FN_c$ (false negatives) is the number of examples with class $c$ as their true label but the model does not predict correctly.
$TN_c$ (true negatives) is the number of examples where the true label is not class $c$ and the model also does not predict class $c$. 

To evaluate the performance of a classifier, we apply it to predict examples in a labeled test dataset $\mathcal{D}'$ and report the percentage of correct predictions, so called the \textit{accuracy} or \textit{classification rate}.

\begin{equation}
    Accuracy = \frac{|\{(x_i, y_i) \in \mathcal{D}' | \hat{y}_i=y_i\}|}{|\mathcal{D}'|} = \frac{TP+TN}{TP+TN+FP+FN} 
\end{equation}

However, if the test dataset is class imbalanced, i.e., having examples of one class more than the other, accuracy may not be the best evaluation metric because a model can get a high accuracy only by always answering the majority class.
Alternatively, we can report the model performance for each specific class $c \in \{0, 1\}$ using the class \textit{precision}, \textit{recall}, and \textit{F1}, defined as follows \cite[chapter~4]{jurafsky2020speech}. 
\begin{align}
    \mbox{Precision}_c &= P_c = \frac{TP_c}{TP_c + FP_c} = \frac{|\{(x_i, y_i) \in \mathcal{D}' | \hat{y}_i=y_i=c\}|}{|\{(x_i, y_i) \in \mathcal{D}' | \hat{y}_i=c\}|} \label{eq2:precision} \\
    \mbox{Recall}_c &= R_c = \frac{TP_c}{TP_c + FN_c} = \frac{|\{(x_i, y_i) \in \mathcal{D}' | \hat{y}_i=y_i=c\}|}{|\{(x_i, y_i) \in \mathcal{D}' | y_i=c\}|} \label{eq2:recall} \\
    \mbox{F1}_{c} &= \frac{2P_cR_c}{P_c + R_c} \label{eq2:f1}
\end{align}

There are two ways to aggregate the class-specific metrics to be the metrics for the overall performance. First, \textit{micro-averaging} combines the $TP$, $FP$, and $FN$ of all classes (in $\mathcal{C}$) before computing precision and recall.
\begin{align}
    \mbox{Micro Precision} &= \frac{\sum\limits_{c \in \mathcal{C}}TP_c}{\sum\limits_{c \in \mathcal{C}}TP_c + \sum\limits_{c \in \mathcal{C}}FP_c} \label{eq2:microprecision}\\
    \mbox{Micro Recall} &= \frac{\sum\limits_{c \in \mathcal{C}}TP_c}{\sum\limits_{c \in \mathcal{C}}TP_c + \sum\limits_{c \in \mathcal{C}}FN_c} \label{eq2:microrecall}\\
    \mbox{Micro F1} &= \frac{2 \times \mbox{Micro Precision} \times \mbox{Micro Recall}}{\mbox{Micro Precision} + \mbox{Micro Recall}} \label{eq2:microf1}
\end{align}

Second, \textit{macro-averaging} averages out precision and recall scores of all the classes so all the classes are weighted equally regardless of their size. 
\begin{align}
    \mbox{Macro Precision} &= \frac{1}{|\mathcal{C}|}\sum\limits_{c \in \mathcal{C}}P_c \label{eq2:macroprecision}\\
    \mbox{Macro Recall} &= \frac{1}{|\mathcal{C}|}\sum\limits_{c \in \mathcal{C}}R_c \label{eq2:macrorecall}\\
    \mbox{Macro F1} &= \frac{2 \times \mbox{Macro Precision} \times \mbox{Macro Recall}}{\mbox{Macro Precision} + \mbox{Macro Recall}} \label{eq2:macrof1}
\end{align}

Normally, when we work with datasets that are class imbalanced, we want the model to work well for all the classes, not just for the majority class. Therefore, we often use macro F1 as the main evaluation metric in addition to the classification accuracy.
\end{appendix}



\bibliographystyle{ios1}           
\bibliography{anthology,bibliography}        

%

\end{document}